\documentclass[twoside]{article}

\usepackage[accepted]{aistats2026arxiv}
\usepackage{amsmath,amsfonts}
\usepackage{algorithmic}
\usepackage{algorithm}
\usepackage{array}
\usepackage[caption=false,font=normalsize,labelfont=sf,textfont=sf]{subfig}
\usepackage{textcomp}
\usepackage{stfloats}
\usepackage{url}
\usepackage{verbatim}
\usepackage{wrapfig}
\usepackage{graphicx,booktabs,hyperref}
\usepackage[sort]{natbib}
\usepackage{xcolor}
\definecolor{Bleu}{HTML}{20bd89}
\definecolor{Red}{HTML}{f14e0c}
\hypersetup{colorlinks,citecolor=Bleu,linkcolor=Red}
\usepackage[nohints]{minitoc}

\usepackage{array}
\usepackage{makecell}
\usepackage{float,xcolor,todonotes,tcolorbox}
\usepackage{mathtools}
\usepackage{amsmath, amssymb, natbib, graphicx, url}
\usepackage{dsfont,enumitem}
\usepackage{footnote}
\usepackage{caption}
\usepackage{mwe}
\usepackage{cancel}
\usepackage{multirow}
\usepackage{tcolorbox}

\makesavenoteenv{tabular}
\makesavenoteenv{table}
\newcommand{\E}{\mathbb{E}}
\newcommand{\bx}{\mathbf{x}}

\newcommand\pol{\ensuremath{\boldsymbol{\pi}}}
\newcommand \defn {\mathrel{\triangleq}}

\DeclareMathOperator*{\argmax}{arg\,max}
\DeclareMathOperator*{\argmin}{arg\,min}

\newcommand{\SP}{\mathfrak{s}}
\newcommand{\charknown}{\ensuremath{T_{\mathcal{F},r}(\mu)}}
\newcommand{\charactimer}{\ensuremath{T_{\Fest,r}^{-1}}}

\newcommand{\tildeA}{\Tilde{\bA}}
\newcommand{\nullspace}{\mathrm{Null}(\tildeA')}
\newcommand{\minsing}{\sigma_{\mathrm{min}}(\tildeA')}
\newcommand{\maxsing}{\sigma_{\mathrm{max}}(\tildeA')}
\newcommand{\minsinginv}{\sigma_{\mathrm{min}}({\tildeA'}^{-1})}
\newcommand{\maxsinginv}{\sigma_{\mathrm{max}}({\tildeA'}^{-1})}
\newcommand{\numconst}{\ensuremath{d}}
\newcommand{\numarm}{{K}}
\newcommand{\identity}{I}
\newcommand{\indicator}{\mathds{1}}
\newcommand{\va}{\boldsymbol v_a}

\newcommand{\polreco}{\hat{\pol}}
\newcommand{\cost}{\mathbf{c}}
\newcommand{\costnoise}{\boldsymbol{\eta}}

\usepackage{amsthm}
\newtheorem{theorem}{Theorem}
\newtheorem{corollary}{Corollary}
\newtheorem{lemma}{Lemma}
\newtheorem{proposition}{Proposition}
\newtheorem{definition}{Definition}
\newtheorem{remark}{Remark}

\newtheorem{assumption}{Assumption}

\makeatletter
\newtheorem*{rep@theorem}{\rep@title}
\newcommand{\newreptheorem}[2]{%
	\newenvironment{rep#1}[1]{%
		\def\rep@title{\textbf{#2} \ref{##1}}%
		\begin{rep@theorem}}%
		{\end{rep@theorem}}}
\makeatother
\newreptheorem{theorem}{Theorem}
\newreptheorem{lemma}{Lemma}
\newreptheorem{proposition}{Proposition}
\newreptheorem{assumption}{Assumption}
\newreptheorem{corollary}{Corollary}

\allowdisplaybreaks%
\newif\ifdoublecol
\doublecolfalse

\usepackage{mathrsfs}

\newcommand{\dbcomment}[1]{\todo[inline,color=blue!20]{{\textbf{DB:~}#1}}}

\usepackage{tikz}
\usetikzlibrary{automata}
\usetikzlibrary{bayesnet, arrows, calc, shapes, backgrounds,arrows,decorations.pathmorphing,fit,positioning}
\tikzset{
   container/.style = {rectangle, rounded corners, draw=yellow, dashed,
fit=#1, inner sep=6mm, node contents={}},
circle-label/.style = {circle, draw}
        }
\tikzset{box/.style={draw, diamond, thick, text centered, minimum height=0.5cm, minimum width=1cm, text width=0.9cm}}
\tikzset{line/.style={draw, thick, -latex'}}
\allowdisplaybreaks
\usepackage{pgfplots}

\usepackage{layouts}

\def\bA{\mathbf{A}}

\def\blambda{{\boldsymbol\lambda}}

\def\ba{\mathbf{a}}

\def\bx{\mathbf{x}}

\def\btheta{{\boldsymbol\theta}}

\def\blambda{{\boldsymbol\lambda}}
\def\bmu{{\boldsymbol\mu}}

\def\bpi{{\boldsymbol\pi}}

\def\bomega{{\boldsymbol\omega}}

\def\cC{\mathcal{C}}
\def\cD{\mathcal{D}}

\def\cF{\mathcal{F}}

\def\cL{\mathcal{L}}

\def\cX{\mathcal{X}}

\def\E{\mathbb{E}}

\def\det{{\sf det}}

\newcommand{\reals}{\mathbb{R}}
\newcommand{\Fest}{\Hat{\cF}}
\newcommand{\altset}{\Lambda_{\Fest}(\bmu)}
\newcommand{\altsetreal}{\Lambda_{\cF}(\bmu)}
\newcommand{\altsetr}{\Lambda_{\Fest}(\bmu)}
\newcommand{\simp}{\Delta_{K}}
\newcommand{\opt}{\pol^{\star}}
\newcommand{\stopping}{\tau_{\delta}}

\newcommand{\kl}{d\left(\bmu_a,\blambda_a\right)}
\newcommand{\sumk}{\sum_{a=1}^K}
\newcommand{\neigh}{\nu_{\Fest}(\opt_{\Fest})}
\newcommand{\neighr}{\nu_{\Fest}(\pol)}
\newcommand{\neighreal}{\nu_{\cF}(\opt_{\cF})}
\newcommand{\lag}{\boldsymbol{l}}
\newcommand{\allocation}{\boldsymbol{\omega}}
\newcommand{\klvec}{d\left(\bmu,\blambda\right)}
\DeclareMathOperator*\lowlim{\underline{lim}}
\DeclareMathOperator*\uplim{\overline{lim}}
\newcommand{\rgood}{\Pi_{\mathcal{F}}^r}
\newcommand{\rgoodest}{\Pi_{\Fest}^r}
\newcommand{\trgoodest}{\Pi_{\Fest_t}^r}
\newcommand{\rpolest}{\Hat{\pol}^{\star}}

\begin{document}



\onecolumn

\aistatstitle{Learning to Explore with Lagrangians for Bandits under Unknown Constraints}
\runningtitle{Pure Exploration under Unknown Linear Constraints}

\aistatsauthor{ Udvas Das \And Debabrota Basu}

\aistatsaddress{ Univ. Lille, Inria, CNRS, Centrale Lille, UMR 9189 – CRIStAL F-59000 Lille, France} 

\begin{abstract}
Pure exploration in bandits formalises multiple real-world problems, such as tuning hyper-parameters or conducting user studies to test a set of items, where different safety, resource, and fairness constraints on the decision space naturally appear. We study these problems as pure exploration in multi-armed bandits with unknown linear constraints, where the aim is to identify an \textit{$r$-optimal and feasible policy} as fast as possible with a given level of confidence. First, we propose a Lagrangian relaxation of the sample complexity lower bound for pure exploration under constraints. %We show how this lower bound evolves with the sequential estimation of constraints. 
Second, we leverage properties of convex optimisation in the Lagrangian lower bound to propose two computationally efficient extensions of Track-and-Stop and Gamified Explorer, namely LATS and LAGEX. Then, we propose a constraint-adaptive stopping rule, and while tracking the lower bound, use optimistic estimate of the feasible set at each step. We show that LAGEX achieves asymptotically optimal sample complexity upper bound, while LATS shows asymptotic optimality up to \textit{novel} constraint-dependent constants. Finally, we conduct numerical experiments with different reward distributions and constraints that validate efficient performance of LATS and LAGEX.% with respect to the baselines.
\end{abstract}

\tableofcontents\newpage

\section{Introduction}
% \textbf{Bandit setup and pure exploration}:
% \IEEEPARstart{D}{ecision}-making under uncertainty is a ubiquitous challenge encountered across various domains, including clinical trials~\citep{Villar15}, recommendation systems~\citep{Zhao24}, and more. 
Multi-Armed Bandit (MAB) serves as an archetypal framework for sequential decision-making under uncertainty that allows us to study the corresponding information-utility trade-offs~\citep{lattimore_szepesvari_2020}. 
In MAB, at each step, an agent interacts with an environment consisting of $K$ decisions (also knows as \textit{arms}) corresponding to $K$ noisy feedback distributions (or \textit{reward} distributions). At each step, the agent takes a decision, and obtains a reward from its unknown reward distribution. The goal of the agent is to compute a \textit{policy}, i.e. a distribution over the decisions, maximising a utility metric (e.g. accumulated rewards~\citep{auer2002finite}, probability of identifying the best arm~\citep{kaufmann2016complexity} etc.).  %over time.

% At its core, the MAB problem revolves around a gambler's(agent) dilemma: faced with a row of slot machines (arms), each with an unknown reward distribution, how does one strategically allocate pulls to maximize cumulative reward over time? The agent, at each time step, selects an arm to pull, observes the associated noisy reward (Limited Feedback), updates the estimates of the parameters of the reward model, which influences his decisions in the future to maximize long-term reward. % MAB captures the fundamental trade-off between \textit{exploration}, i.e. gathering information, and exploitation, i.e. leveraging current information to maximize rewards. 
In this paper, we focus on the \textit{pure exploration} problem of MABs, where the agent interacts by realising a sequence of policies (or experiments) with the goal of \textit{answering a query} about the environment \textit{as correctly as possible}.
A well-studied pure exploration problem is Best-Arm Identification (BAI), where the agent aims to identify the arm with highest expected reward~\citep{EvenDar2002PACBF, bubeck2009pure,jamieson2014best,kaufmann2016complexity}. 
\cite{kaufmann2016complexity} derives an information-theoretic lower bound quantifying the minimum number of agent-environment interactions needed to identify the best arm with a given level of confidence. The lower bound depends on optimising the weighted sum of KL-divergences between the reward distributions of arms and their most confusing counterparts (Equation~\eqref{eq:lb_known}). In this paper, we leverage the lower bound for algorithm design that plugs-in the empirical estimates of the reward distributions in lower bound and solves the optimisation problem on-the-go~\citep{kaufmann2016complexity,degenne2019nonasymptotic,carlsson2024pure}. 

BAI has been applied in hyper-parameter tuning~\citep{li2017hyperband}, communication networks~\citep{lindstaahl2022measurement}, influenza mitigation~\citep{libin2019bayesian}, finding the optimal dose of a drug~\citep{aziz2021multi} etc. However, real-world scenarios often impose constraints on the arm pulls~\citep{carlsson2024pure}. 
\textit{For example,} \cite{baudry2024multi} considers a recommendation problem with the aim to guarantee a fixed (known) minimum expected revenue per recommended content while identifying the best content from bandit feedback. 
Additionally, if we have multiple objectives in a decision making problem, a popular approach to optimize them is finding the optimal policy for one objective while constraining the others~\citep{fonseca1998multiobjective}. %But often the best thresholds applicable in these constraints are unknown and have to be learnt on the go. 

\vspace{2mm}\textbf{Pure Exploration under Constraints.} The aforementioned problems motivated the study of pure exploration under a set of known and unknown constraints~\citep{katz2018feasible,wang2021best,li2023optimal,wu2023best,carlsson2024pure}. Specifically, we aim to find the optimal policy that maximises the expected reward obtained from the set of arms and satisfies the true constraints, with confidence $1-\delta$. 
% \textbf{What's the setting.} 
%The agent, at each step $t$, selects an action according to a chosen policy observes associated reward and the cost incurred with respect to the constraints. The agent uses the feedbacks to update the estimates of the expected rewards and constraints. Using these estimates, the agent chooses the next policy and action till the optimal policy satisfying constraints is identified with confidence $1-\delta$. 
This is known as the \textit{fixed-confidence setting} of pure exploration~\citep{wang2021best,carlsson2024pure}. On the other hand, there is also the fixed-budget setting, which is of independent interest~\citep{katz2018feasible,li2023optimal,faizal2022constrained}.
Existing literature has studied either the general linear constraints when they are known~\citep{carlsson2024pure,Camilleri22}, or very specific type of unknown constraints, e.g. safety~\citep{wang2021best}, knapsack~\citep{li2023optimal}, fairness~\citep{wu2023best}, preferences~\citep{lindner2022interactively} etc.
Here, we study the \textit{pure exploration problem in the fixed-confidence setting subject to unknown linear constraints on the policy}, which generalises all these settings (Section~\ref{sec:formulation}). A detailed discussion on related works is deferred to Appendix~\ref{sec:Ex_related_work}.
% \textbf{Why interesting.} % \citep{carlsson2024pure} study impact of known constraints on pure exploration and its deviation from classic BAI problems~\citep{carlsson2024pure}. Whereas

Recently, \cite{carlsson2024pure} derives a tight lower bound and designs asymptotically optimal algorithms for this problem when the constraints are known, and show that a bandit instance might become harder or easier depending on the geometry of the constraints. They pose that \textit{studying similar phenomenon for unknown constraints is an open problem} as constraints are also estimated. %Then, they propose two algorithms that leverages their lower bound with known constraints, either by tracking it in each iteration or treating the lower bound optimisation as a two-player zero sum game. 

%Specifically, we have to simultaneously control concentrations of the mean rewards and the constraints to their `true' values to reach the feasible policy set. 
%Additionally, at any finite time, the estimated constraints might exhibit small but non-zero errors.
%This prevents from `exactly' obtaining the feasible policy set, and consequently, the recommended optimal and feasible policy though the error might be below numerical limits. 
The challenge is that the lower bound dictating the hardness of the constrained pure exploration problem is sensitive to the active constraints, and for unknown constraints, we only have access to estimated constraints with non-zero error. This affects the exact detection of the optimal feasible policy. %Hardness of any BAI problem is often characterised by the indistinguishability between the true and an \textit{alternative} instance (another bandit instance exhibiting a different optimal answer).   
 
%The following example further clarifies this nuance.
% For example, if one perturbs any one of the active constraints, even slightly, the hardness of the identification problem blows up inversely. 

\textit{Example.} Let us consider a bandit environment with 4 arms with Gaussian rewards -- means $\bmu = [1.0, 0.8, 0.6]$ and variance $1$. The constraints are $A\bpi \leq \mathbf{0}$ where $A = \begin{bmatrix}
    0.0, -0.3, -0.2\\
    -0.2, 0.1, 0.1
\end{bmatrix}$. 
Here, the optimal policy is $\bpi^{\star} = [1.0, 0, 0]$, and the first constraint is active at the optimal policy. Figure \ref{fig:epsilon_perturbation} shows a perturbation of the active constraint by $0.005$ shifts the optimal policy, and the sample complexity~\citep{carlsson2024pure} blows with $\mathcal{O}(4\times 10^4)$.

\begin{figure}[t!]%{r}{8cm}
\centering%\vspace*{-1em}
%\begin{subfigure}[b]{\textwidth}
\includegraphics[width=0.4\textwidth]{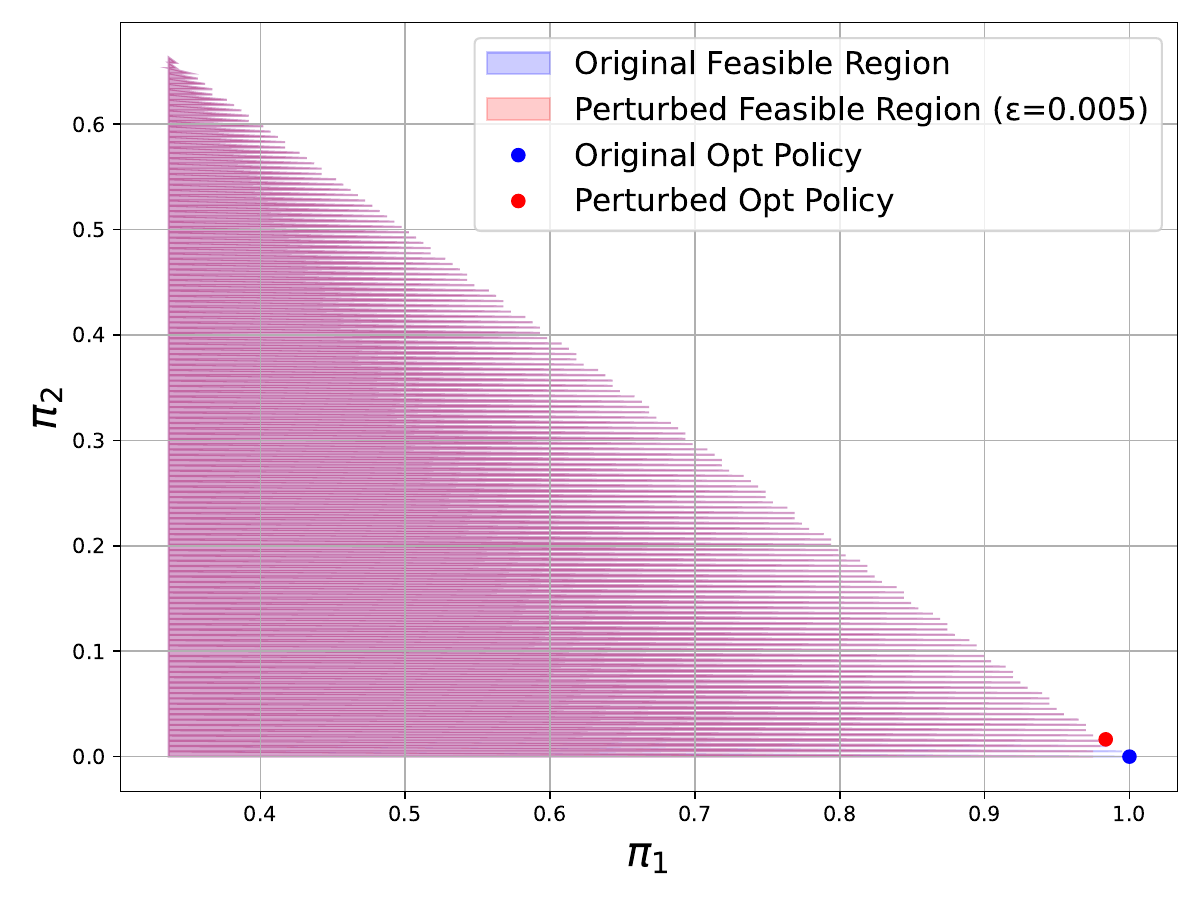}\vspace*{-.5em}
%\end{subfigure}
\caption{Effect of rank-one update of active constraint}\label{fig:epsilon_perturbation}%\vspace*{-1em}
\end{figure}

Hence, for rigour, we relax the problem of finding the optimal feasible policy to finding an $r$-optimal feasible policy. For a given $r>0$, an $r$-optimal policy has mean reward not more than $r$ away from that of the optimal policy~\citep{mason2020finding,jourdan2022choosing,jourdan2023varepsilon}.
% In practice, one can consider $r$ to be the numerical limit of computation or a reasonably small quantity. 
This leads us to two questions:%\vspace*{-.5em}
\begin{tcolorbox}[left=2pt,right=2pt,top=1pt,bottom=1pt]
    1. \textit{How does the hardness of finding $r$-optimal feasible policy change under unknown constraints?}\\ 
    2. \textit{How can we design a generic algorithmic scheme to track both the constraints and optimal policy with sample- and computational-efficiency?}
\end{tcolorbox}

% \end{enumerate}

\iffalse
\textbf{What's missing.}
Existing literature shows people has only considered constraints in the context of safety, knapsack or any other known constraint setup, where the normal cone built around the true optimal policy by the constraints is known. But in our setting this normal cone is subject to estimation. In real-life scenarios, for e.g, in a clinical-trials, while we must try to maximize the overall likelihood of success, but under limited resources, as number of medical personnel, laboratory equipment or funds etc. Here, we may know the upper bound on the total resource used, but may not know the contribution of each resource, i.e unknown constraints. In these cases, our problem setting comes as a rescue. 
\fi
% \textbf{Example.}
% \textbf{Our questions.} The questions we are aiming to answer in this work are,

\textbf{Our Contributions} positively address the questions.

1. \textit{Lagrangian relaxation of the lower bound.} Minimum number of samples required to conduct $r$-optimal pure exploration with fixed confidence is expressed by a lower bound-- an optimisation problem under known constraints (Eq.~\eqref{eq:lb_known}). To efficiently handle unknown constraints, we propose a novel Lagrangian relaxation of this optimisation (Section~\ref{sec:lag_lb}). At every step, we construct optimistic feasible policy set, and plug it in the relaxation. Lagrangian multipliers balance the identifiability of an $r$-optimal policy and the feasibility under estimated constraints. We leverage results from convex analysis to show that the relaxed lower bound with optimistic feasible set preserves all the desired properties of the lower bound under known constraints, and thus, allows designing lower bound tracking algorithm. %Additionally, it satisfies strong duality to yield a unique optimal policy for interactions, and also bounds on the Lagrangian multipliers at every step. 
% Further, we characterise the Lagrangian lower bound for Gaussian rewards, which connects with the same for known constraints.

2. \textit{Generic algorithm design.} First, we propose a new stopping rule accommodating the estimated constraints.% in the existing rule for known constraints. 
This ensures concentration of the estimates of mean rewards and constraints to their true values before final recommendation. 
We further show this stopping rule recommends a policy that is both feasible and $r$-optimal with confidence at least $1-\delta$. Then, we extend the Track-and-Stop~\citep{pmlr-v49-garivier16a} and gamified explorer~\citep{degenne2019nonasymptotic} approaches with the Lagrangian lower bound to design LATS (LAgrangian Track and Stop) and LAGEX (LAgrangian Gamified EXplorer), respectively in (Section~\ref{sec:algo}). 

3. \textit{Upper bounds on sample complexities.} We derive upper bounds on the sample complexities of LATS and LAGEX (Section~\ref{sec:algo}). This requires proving a novel concentration inequality for the constraint estimates. As a consequence, LATS achieves an upper bound, which is $(1+\mathfrak{s})$ times the asymptotic upper bound of TS under known constraints, while LAGEX exhibits optimality. $\mathfrak{s}$ is the shadow price (ratio between maximum and minimum index value of the slack vector) that quantifies its stability under perturbation. %\todo[inline]{Should we say the rest?} In contrast, LAGEX leads to an upper bound that has only an additive $\mathfrak{s}$ factor with the known constraint lower bound. This suggests that LAGEX should be more sample-efficient than LATS. 
Finally, we conduct experiments across synthetic and real data. We observe that LAGEX requires the least samples among competing algorithms and exactly tracks the change in hardness due to constraints across environments (Section~\ref{sec:experiments}).

\section{Exploration under Unknown Constraints}\label{sec:formulation}%\vspace*{-.5em}
\setlength{\textfloatsep}{10pt}

% \textbf{Notations.} $\bx$, $X$, and $\cX$ denote a vector, a matrix, and a set, respectively. Detailed notations are in Table~\ref{tab:notations}.%$\bx_i$ denotes $i$-th component of $\bx$. For a positive semi-definite matrix $A$ and vector $\bz$, $\|\bz\|_{A}^2 = \langle \bz, A \bz\rangle$. In $\reals_{+}^d$, we include $\mathbf{0}^d$. $[K]$ refers to $\{1,\ldots, K\}$ for $K \in \mathbb{N}$. $\mathrm{Supp}(P)$ denotes the support of a distribution $P$. $\simp$ is the simplex over $[K]$. $d(P,Q)$ is the Kullback-Leibler divergence between two absolutely continuous distributions $P$, $Q$.
\textbf{Problem Formulation.}
%\udcomment{Rewrote this section. Seems Dense. Can be simplified.}
We work with a MAB instance consisting of $K \in \mathbb{N}$ arms. Each arm $a \in [K]$ has a reward distribution $P_a$ with unknown mean $\bmu_a \in \mathbb{R}$. At each step $t\in\mathbb{N}$, the agent chooses an action $A_t \in [K]$, and observes a stochastic reward $R_t \sim P_{A_t}$. A feasible policy $\pol \in \simp$ satisfies $A\pol \leq \mathbf{0}$ with respect to $\numconst$ linear constraints $A\in\mathbb{R}^{\numconst \times \numarm}$.\footnote{\textbf{Notations:} $\bx$, $X$, and $\cX$ denote a vector, a matrix, and a set, respectively. Detailed notations are in Table~\ref{tab:notations}. We augment the simplex constraints in $A$, and normalise each row of $A$, i.e. $\|A_i\|_2 =1$ for all $i\in[d]$.}

\begin{algorithm}[t!]%\scriptsize{7pt}
\caption{Pure Exploration in Bandits with Unknown Linear Constraints}\label{alg:protocol}
\begin{algorithmic}[1]
    \STATE \textbf{Input:} Tolerance $r >0$, Confidence level $\delta \in (0,1)$
    % \STATE \textbf{Initialization :} $\Hat{A}_0 = \textbf{0}_{{\numconst}\times K}$, $\Hat{\bmu}_0 = \textbf{0}_K$, $\Sigma_0 = v\mathbf{I}_{K}$, $l_0$ 
    % \STATE Play each arm once to set $\bmu_1$ and $\Hat{A}_1$.
    \FOR{$t=1,\, \ldots$}
    %{\color{blue}$\rightsquigarrow$ Proposition~\ref{prop:Projection Lemma for Lagrangian Formulation}}    
    \STATE \textbf{Decision/sampling:} Play an arm $a_t \in [1,K]$
    \STATE \textbf{Reward-cost Feedback:} Observe reward $r_t \sim P_{a_t}$ and cost $\cost_t \sim A a_{t} + \eta_t$
    \IF{more than $1-\delta$ confident about the estimated answer being correct}
    \STATE Stop and stopping time $\tau_{\delta} \gets t$
    \ENDIF
    % \STATE such that $\Vert\lag_t^{\star}\Vert_1 \leq \frac{1}{\gamma}  \mathcal{D}(\omega_{t}^{\star},\Hat{\bmu}_t,\Fest_t) $, where, $\gamma = \min_{i\in[1,{\numconst}]} \{ -\Tilde{A}_t^i \omega_{t}^{\star}\}$
    \ENDFOR
    \STATE \textbf{Recommendation:} $\polreco = \argmax_{\pol\in\Pi_{\Fest_{\tau_{\delta}}}^r} \Hat{\bmu}_{\tau_{\delta}}^{\top} \pol$
\end{algorithmic}
\end{algorithm}

\ifdoublecol
If $A$ is known, the agent has access to the non-empty and compact set of feasible policies $\cF \defn \{ \pol \in \simp \mid A\pol \leq \boldsymbol{0}\}$. The agent aims to identify an $r$-optimal optimal feasible policy, i.e. any feasible policy which belongs to $\rgood \defn \{ \pol\in\mathcal{F} \mid \bmu^{\top}\pol +r \geq \bmu^{\top}\pol_{\cF}^{\star} \}$, given 

\begin{align}\label{eq:opt_pol}
	\opt_{\cF} \defn \argmax_{\pol\in\cF} \bmu^{\top} \pol\,. 
\end{align}%\vspace*{-1.7em}

%We assume that $\opt_{\cF}$ is unique~\citep{carlsson2024pure}.
\else
If $A$ is known, the agent has access to the non-empty and compact set of feasible policies $\cF \defn \{ \pol \in \simp \mid A\pol \leq \boldsymbol{0}\}$. The agent aims to identify an $r$-optimal optimal feasible policy, i.e. any feasible policy which belongs to $\rgood \defn \{ \pol\in\mathcal{F} \mid \bmu^{\top}\pol +r \geq \bmu^{\top}\pol_{\cF}^{\star} \}$, given 
\begin{align}\label{eq:opt_pol}
	\opt_{\cF} \defn \argmax_{\pol\in\cF} \bmu^{\top} \pol\,. 
\end{align}\vspace*{-1em}
\fi

\begin{definition}[$(1-\delta)$-correct and $(1-\delta)$-feasible $r$-optimal pure exploration]
	\ifdoublecol
	For $\delta \in [0,1)$, an $r$-optimal pure exploration algorithm is called $(1-\delta)$-correct and $(1-\delta)$-feasible if the policy $\hat{\pol}$ recommended by it satisfies $\Pr[\hat{\pol} \notin \rgood] \leq \delta$ and $\Pr[{A}\hat{\pol} \geq \boldsymbol{0}] \leq \delta$.
	\else
	For $\delta \in [0,1)$, an $r$-optimal pure exploration algorithm is called $(1-\delta)$-correct and $(1-\delta)$-feasible if the policy $\hat{\pol}$ recommended by it satisfies $\Pr[\hat{\pol} \notin \rgood] \leq \delta$ and $\Pr[{A}\hat{\pol} \geq \boldsymbol{0}] \leq \delta$.
	\fi
\end{definition}
\ifdoublecol	
In our setting, we do not have access to the true set of constraints. Hence, using the observations, we construct $\Hat{A}$ as an estimate of $A$. Then, the agent builds an estimated feasible set $\Fest \defn \{ \pol \in \simp\mid\hat{A} \pol \leq \textbf{0} \}$ to identify the optimal feasible policy as $\opt_{\Fest} \defn \arg \max_{\pol\in\Fest} \bmu^{\top} \pol$. In addition, the estimated $r$-optimal policy set is $\rgoodest \defn \{ \pol\in\Fest \mid \bmu^{\top}\pol +r \geq \bmu^{\top}\pol_{\Fest}^{\star} \}$.
We know that obtaining accurate estimates of these quantities would require us to collect feedback of satisfying constraints over time. This poses an additional cost of using observations to estimate $\bmu$.
\else
In our setting, we do not have access to the true set of constraints. Hence, using the observations, we construct $\Hat{A}$ as an estimate of $A$. Then, the agent builds an estimated feasible set $\Fest \defn \{ \pol \in \simp\mid\hat{A} \pol \leq \textbf{0} \}$ to identify the optimal feasible policy as $\opt_{\Fest} \defn \arg \max_{\pol\in\Fest} \bmu^{\top} \pol$. In addition, the estimated $r$-optimal policy set is $\rgoodest \defn \{ \pol\in\Fest \mid \bmu^{\top}\pol +r \geq \bmu^{\top}\pol_{\Fest}^{\star} \}$.
We know that obtaining accurate estimates of these quantities would require us to collect feedback of satisfying constraints over time. This poses an additional cost of using observations to estimate $\bmu$.
\fi

\textbf{Goal.} In order to recommend a $(1-\delta)$-correct and $(1-\delta)$-feasible policy that is $r$-optimal with respect ot the true optimal policy $\opt_{\cF}$, we aim to minimise the expected number of interactions $\mathbb{E}[\stopping] \in \mathbb{N}$. %Formally, we aim to design an algorithm that recommends a $(1-\delta)$-correct and $(1-\delta)$-feasible policy while keeping $\mathbb{E}[\stopping]$ as small as possible.

\subsection{Extension of Prior Bandit Problems} 
Now, we clarify our motivation by showing that different existing problems are special cases of our setting.

\vspace{2mm}\textbf{a. Thresholding Bandits.} Thresholding bandits~\citep{aziz2021multi} are motivated from the safe dose finding problem, where one wants to identify the most effective dose of a drug below a known safety level. This has also motivated the safe arm identification problem~\citep{wang2021best}. Our setting generalises it further to detect the optimal combination of doses of available drugs yielding highest efficacy while staying below the safety threshold. Formally, we identify $\opt = \argmax_{\pol\in\simp} \bmu^{\top}\pol$, such that $\identity\pol \leq \identity \btheta$ for thresholds $\btheta_a >0$.

% \vspace{2mm}\textbf{b. Optimal Policy under Knapsack.} Bandits under knapsack constraints are studied in both BAI~\citep{li2023optimal,soton337280,pmlr-v139-li21s} and regret minimisation literature~\citep{10.1145/3164539, NIPS2016_f3144cef,10.1145/3557045,pmlr-v49-agrawal16, pmlr-v84-sankararaman18a}. Detecting an optimal arm might have additional resource constraints than the number of required samples. This led to study of BAI with knapsacks under fixed-budget setting~\citep{li2023optimal}. But one might only want to deploy a final policy that maximises utility and satisfies knapsack constraints. For example, we want to manage caches where the recommended memory allocation should satisfy a certain resource budget but can violate them during exploration (End-of-Time constraint~\citep{carlsson2024pure}). Formally, $\opt_{\tau_{\delta}} = \argmax_{\pol \in C_{A}} \Hat{\bmu}_{\stopping}^{\top} \pol$,  and $C_{A} \defn \{ A\pol_{\stopping} \leq c , c\geq 0 \}$. 
% % Naturally, this is a special case of our problem setting. 

\textbf{b. BAI with Fairness Constraints across Sub-populations (BAICS).} \cite{wu2023best} aims to select an arm that must be fair across $M$ sub-populations. Here, the arm belongs to a set $C\triangleq \{k\in[\numarm]|\bmu_{k,m}\geq 0 , m\in[M]\}$ where the observation for arm $k$ and population $m$ comes from $\mathcal{N}(\bmu_{k,m},1)$. It ensures that the chosen arm does not perform \textit{too bad} for any sub-population.  Like standard BAI, finding only the optimal arm might not be enough because it might not perform equally good for all of the sub-populations. This is similar to having $M$ groups of patients and $\numarm$ drugs to administer, where we are looking for a mixture of drugs such that $\opt = \argmax_{\pol\in\simp} \bmu^{\top}\pol$, such that $\indicator_{\bmu_m\geq \boldsymbol{0}}^{\top} \pol =1 , \forall m\in[M]$. %Hence, our setting models BAICS as a special case. 

We refer to Section \ref{subsec:motivation} for further discussion on generalisation to the existing bandit settings.

\section{Lagrangian Relaxation of the Lower Bound}\label{sec:lag_lb} %\vspace*{-.5em}
Now, we derive Lagrangian relaxation of lower bound and its properties % to design a correct and feasible $r$-optimal pure exploration algorithm under unknown linear constraints.
under a structural assumption.
\begin{assumption}[Structures of means, policy, and constraints]\label{ass:structural}
    (a) The mean vector $\bmu$ is in a bounded subset $\cD$ of $\reals^{\numarm}$. (b) There exists a unique optimal feasibly policy (Eq.~\eqref{eq:opt_pol}). (c) The true constraint $A$ yields a non-zero slack vector $\Gamma$: $\max_{\pol \in \simp} (-A \pol) = \Gamma.$
\end{assumption}
We impose the unique optimal and feasible policy assumption following~\cite{carlsson2024pure}. % to ensure that the solution of Equation~\eqref{eq:opt_pol} is a unique extreme point of the polytope $\cF$. 
The assumption on slack is analogous to existence of a safe-arm~\citep{pacchiano2020stochastic}, or Slater's condition for the constraint optimisation problem~\citep{liu2021efficient}. 
%Standing on these assumptions, we prove that $\opt_{\Fest}$ is unique, i.e $\opt_{\Fest}$ is an extreme point in the polytope $\Fest$.

\iffalse
a. why Lagrangian from constrained LBs
\dbcomment{justification with also existing lag methods in reg w constraints and online learning w constraints?}
b. properties of the Lagrangian bound: estimated feasible set, continuity etc
\fi
%\vspace*{-.5em}
\subsection{Information Acquisition: Estimate Constraints}\label{sec:Estimates}%\vspace*{-.5em}
The agent acquires information at every step $t\in\mathbb{N}$ by sampling an action $\ba_t \sim \allocation_t$. As the arms are independent, we represent the $a$-th arm as the $a$-th basis in $\reals^K$, denoted by $\ba \in \reals^K$. 
$\allocation_t \in \simp$ is called \textit{the allocation policy}. %It is used for interaction at step $t$. 
As shown in Algorithm~\ref{alg:protocol}, pulling the arm $a_t$ yields a noisy reward $r_t \in \reals$ and cost vector $\cost_t \in \reals^{\numconst}$. The cost vector $\cost_t \defn A\ba_t + \costnoise_t$, where $\costnoise_t$ is an independent and identically distributed noise vector whose each component uis generated from a $1$-sub-Gaussian distribution with mean zero.

Thus, using the observations obtained till $t$, we estimate the mean vector as $\hat{\bmu}_t \defn \Sigma_t^{-1}\left(\sum_{s=1}^{t-1} r_s \ba_s\right)$. Here, $\Sigma_t \defn \identity + \sum_{s=1}^{\top} \ba_s \ba_s^{\top}$, is the Gram matrix or the design matrix. Similarly, the estimate of the $i$-th row of the constraint matrix is $\Hat{A}_t^i \defn \Sigma_t^{-1}\left(\sum_{s=1}^{t-1}A^{i,\ba_s} \ba_s\right)$.
But na\"ively using $\Hat{A}_t$ to define the feasible policy set does not ensure that for any $t$, the estimated feasible set $\Fest$ is a superset of $\cF$.
Hence, we define a confidence ellipsoid around $\hat{A}_t$ that includes $A$ with probability at least $1-\delta$, and construct a optimistic estimate of $A$. Formally, the confidence ellipsoid is
% \textbf{Definition.} For a fixed $\delta\in[0,1]$, at time $t\in\mathbb{N}$, the confidence ellipsoid of the constraint matrix is defined as --
\begin{align}\label{def:conf_ellipse}
\hspace*{-1em}    \mathcal{C}_t  \defn \big\{ A' \in \reals^{\numconst \times \numarm} \mid \|A'^i - \hat{A}_t^i\|_{\Sigma_t} \leq f(t,\delta) \forall i\in [\numconst]\big\}, 
\end{align}
where $f(\delta,t) \defn 1+\sqrt{\frac{1}{2}\log\frac{K}{\delta}+\frac{1}{4}\log\det\Sigma_t}$ is a monotonically non-decreasing function of $t$.
\begin{tcolorbox}[left=2pt,right=2pt,top=1pt,bottom=1pt]
    \begin{lemma}[Optimistic feasible sets]\label{lem:estimate_superset}
        At any time $t \in \mathbb{N}$, we construct the optimistic feasible policy set such that with probability $1-\delta$, $\Fest_t \defn \{\pol \in \simp: \min_{A' \in \mathcal{C}_t} A'\pol \leq \boldsymbol{0}\}$, 
        satisfies $\cF \subseteq \Fest_t$, where  $\Tilde{A}_t \triangleq \argmin_{A' \in \mathcal{C}_t} A'\pol$.%\vspace*{-.5em}
    \end{lemma}
\end{tcolorbox}

% We denote the $A'$ achieving the above minimum as $ \Tilde{A}_t$. 
Figure~\ref{fig:Constraint_Propagation} plots this result using the numerical values obtained from our algorithms. Note that as we acquire more samples, estimated feasible policy set $\Fest_t \rightarrow \cF$.
% \dbcomment{why taking optimistic $A$ and how does it relate with all pessimism-opt, opt-opt in  lit?}% \begin{remark}

\begin{remark}
    To ensure that the true optimal and feasible policy $\opt_{\cF} \in \Fest_t$ is inside the estimated feasible policy set, and $\Fest_t \neq \emptyset$ for any $t$, we use the optimistic estimate of the feasible set $\Fest_t$ ensuring $\cF \subseteq \Fest_t$ with high probability. Our construction of the optimistic estimates of feasible set resonate with the optimistic-pessimistic algorithms for regret-minimisation under constraints~\citep{pacchiano2020stochastic,pacchiano2024contextual,liu2021efficient,chen2022strategies}. Note that, similar confidence bounds for estimating constraints are used to test feasibility of linear programs~\citep{testingfeasibilitylinearprograms}.
\end{remark}

%\vspace*{-.5em}
\subsection{Relaxation with Estimated Constraints}%\vspace*{-.5em}
Search for the optimal policy is essentially a linear programming problem when we know the mean vector $\bmu$ and a constraint matrix $A$. The challenge in bandit is to identify them from sequential feedback, i.e. to differentiate $\bmu$ from the other confusing instances in the same family of distributions.
These are called the \textit{alternative instances}. 
The strategy is to gather enough statistical evidence to rule out all such confusing instances, specifically the one that has minimum KL-divergence from $\bmu$ as observed under the allocation policy $\allocation$~\citep{pmlr-v49-garivier16a}. 
This intuition has led to the lower bound of~\cite{carlsson2024pure}-- the expected stopping time of any $(1-\delta)$-correct and always-feasible algorithm satisfies
\begin{equation}\label{eq:lb_known}
\E[\stopping] \geq T_{\cF,r}(\bmu)\ln\frac{1}{2.4\delta}\,,
\end{equation}
if $A$ is known. $T_{\cF,r}(\bmu)$ is called the \textit{characteristics time}. 
Its reciprocal is a max-min optimisation problem over the set of alternative instances $\altsetreal \defn \big\{ \blambda\in\mathcal{D}\mid \max_{\pol \in \rgood} \blambda^{\top}\pol - r > \blambda^{\top} \opt  \big\}$, where $\opt \in \argmax_{\pol \in \rgood}\bmu^{\top} \pol$, i.e.

\ifdoublecol
\begin{align}\label{eqn:char_time_known}
    T^{-1}_{\cF,r}(\bmu) &\defn \sup_{\allocation\in\simp}\max_{\pi\in\rgood} \inf_{\blambda\in\altsetreal} \sumk \allocation_a \kl \notag\\ &\defn  \sup_{\allocation\in\simp}\max_{\pi\in\rgood} \inf_{\blambda\in\altsetreal}\allocation^{\top} \klvec\, .%_{\mathcal{D}(\allocation,\bmu,\cF)} \, .
\end{align}
\else
\vspace*{-1.5em}\begin{align}\label{eqn:char_time_known}
    T^{-1}_{\cF,r}(\bmu) \defn \sup_{\allocation\in\simp}\max_{\pi\in\rgood} \inf_{\blambda\in\altsetreal} \sumk \allocation_a \kl \defn  \sup_{\allocation\in\simp}\max_{\pi\in\rgood} \inf_{\blambda\in\altsetreal}\allocation^{\top} \klvec\, .%_{\mathcal{D}(\allocation,\bmu,\cF)} \, .
\end{align}
\fi
$\Lambda_{\cF}(\bmu)$, referred as the Alt-set, is the set of all bandit instances whose mean vectors are in a bounded subset $\cD \in \reals^K$ but the optimal policy is different than that of $\bmu \in \cD$.
Now, we inspect the change in this lower bound at any step $t>0$, when we only have access to a optimisti estimate $\Fest_t$ and the confidence ellipsoid $\cC_t$ but do not know $\cF$. For brevity, we exclude $t$ from the subscripts for where it is clear from the context. 
The Alt-set given $\Fest$ is defined as %\todo{after this we should call all $\Fest$ as $\Fest_t$.}
\ifdoublecol
\begin{align}\label{eqn:alt_set_estimated}
	\altsetr \defn \big\{ \blambda\in\mathcal{D}\mid  \max_{\pol \in \Fest} \blambda^{\top}\pol - r > \blambda^{\top}\Hat{\pol}^{\star} \big\},
\end{align}
\else
\begin{align}\label{eqn:alt_set_estimated}
	\altsetr \defn \big\{ \blambda\in\mathcal{D}\mid  \max_{\pol \in \Fest} \blambda^{\top}\pol - r > \blambda^{\top}\Hat{\pol}^{\star} \big\},
\end{align}

\fi
where $\Hat{\pol}^{\star} \in \argmax_{\pol \in \rgoodest} \bmu^{\top}\pol$. Since $\cF \subseteq \Fest$, we observe that $\altset \subseteq \altsetreal$~(Figure \ref{fig:normal_cone}). 
% Now, since the polytope $\Fest$ is bigger than the true polytope $\cF$ and $\opt_{\Fest}$ is $\epsilon$-good approximation of $\opt_{\cF}$, the alt-set we are working with is \textbf{\textit{bigger than the actual true alt-set}}, let's denote that by $\Lambda_{\cF}(\bmu)$ and can be defined as,
% To ensure the claim above let us now construct the estimate of $\cF$ modifying on the definition of $\Fest$ in \ref{eq.2}.
% \begin{align}\label{eq.5}
%     \Fest \defn \{ \pol \in \simp:\Tilde{A} \pol \leq \textbf{0} \}
% \end{align}
% where, $\Tilde{A}$ is the pessimistic estimate centred around $A$.
Now, we define Lagrangian relaxation of the lower bound, i.e.
\ifdoublecol
\hspace{-2em}\begin{align}\hspace*{-6em}
    &~~~~\sup_{\allocation\in\simp}\max_{\pol\in\rgood} \inf_{\blambda\in\altsetreal} \allocation^{\top} \klvec\leq 
    \inf_{\lag\in\reals_{+}^{\numconst}} \min_{A' \in \cC} \sup_{\allocation\in\simp}\max_{\pol\in\rgood} \inf_{\blambda\in\altsetr} \allocation^{\top}\klvec - \lag^{\top}A' \allocation \label{eqn:lag_relax}%}_{= \mathcal{D}(\allocation,\bmu,\Fest)}\vspace*{-1em}
\end{align}
\else
\begin{align}\hspace*{-6em}
	T^{-1}_{\cF,r}(\bmu) \defn \sup_{\allocation\in\simp}\max_{\pol\in\rgood} \inf_{\blambda\in\altsetreal} \allocation^{\top} \klvec\leq 
    \inf_{\lag\in\reals_{+}^{\numconst}} \min_{A' \in \cC} \sup_{\allocation\in\simp}\max_{\pol\in\rgood} \inf_{\blambda\in\altsetr} \allocation^{\top}\klvec - \lag^{\top}A' \allocation \label{eqn:lag_relax}%}_{= \mathcal{D}(\allocation,\bmu,\Fest)}\vspace*{-1em}
\end{align}
\fi

% Extending the proof in standard BAI literature \cite{garivier2018thresholding}, expected stopping time of any $\delta$-PAC algorithm $\psi_{\tau}$ for BAI with unknown set of linear constraints satisfies,
% \begin{equation}\label{eq.6}
%     \mathbb{E}[\tau_{\delta}] \geq T_{\Fest}(\bmu)\text{kl}(\delta\Vert1-\delta)
% \end{equation}
% where, the inverse of characteristic time $T_{\Fest}(\bmu)$ has the following expression,
% \begin{align}\label{eq.7}
%     &T_{\Fest}^{-1}(\bmu) = \sup_{\allocation\in\Fest} \underbrace{\bigg\{\inf_{\lag\in\reals_{+}^{\numconst}}\inf_{\blambda\in\altset} \sumk \allocation_a\kl - \lag^{\top} \Tilde{A}\allocation \bigg\}}_{= \mathcal{D}(\allocation,\bmu,\Fest)}\\ \nonumber
%     \geq &{T_{\cF}^{-1}(\bmu)}^L = \sup_{\allocation\in\cF} \underbrace{\bigg\{\inf_{\lag\in\reals_{+}^{\numconst}}\inf_{\blambda\in\altsetreal} \sumk \allocation_a\kl - \lag^{\top} \Tilde{A}\allocation \bigg\}}_{= {\mathcal{D}(\allocation,\bmu,\cF)}^L}\\ \nonumber
%     \geq& {T_{\cF}^{-1}(\bmu)}^{\Gamma} = \sup_{\allocation\in\cF} \underbrace{\bigg\{\inf_{\lag\in\reals_{+}^{\numconst}}\inf_{\blambda\in\altsetreal} \sumk \allocation_a\kl + \lag^{\top} \Gamma \bigg\}}_{= {\mathcal{D}(\allocation,\bmu,\cF)}^{\Gamma}}\\\nonumber
%     \geq& T_{\cF}^{-1}(\bmu) = \underbrace{\sup_{\allocation\in\cF} \inf_{\blambda\in\altsetreal} \sumk \allocation_a \kl}_{\mathcal{D}(\allocation,\bmu,\cF)} \nonumber
% \end{align}

We denote this Lagrangian relaxation of the characteristic time with $\Fest$ as $T_{\Fest,r}^{-1}(\bmu)$. 
For non-negative Lagrange multipliers $\lag \in \reals_{+}^{\numconst}$, the first inequality is true due to the existence of a slack for the true constraints $A$.
The inequality holds due to the optimistic choice of the estimated constraint.
Equation \eqref{eqn:lag_relax} shows that the reciprocal of the Lagrangian relaxation, $T_{\Fest,r}(\bmu)$, serves as a upper bound on the characteristic time $T_{\cF}(\bmu)$ for known constraints~\citep{carlsson2024pure}.%, and thus, leads to a valid lower bound to optimise for the expected stopping time $\E[\stopping]$. %For brevity, , and omit $t$ from $\Fest_t$ if it is true for any $t\in \mathbb{N}$.

The Lagrangian relaxation leads to a natural question:%\vspace*{-.5em}
\begin{tcolorbox}[left=2pt,right=2pt,top=1pt,bottom=1pt]
    Does the dual of the optimization problem for $T_{\Fest,r}^{-1}(\bmu)$ yield the same solution as the primal?
\end{tcolorbox}

We formalise the strong duality result and self-bounding property of the Lagrangian multiplier of the relaxation in Equation~\eqref{eqn:lag_relax} in Theorem~\ref{thm:Existence of strong duality and bound on Lagrangian multiplier} below.

\begin{tcolorbox}[left=2pt,right=2pt,top=1pt,bottom=1pt]
\ifdoublecol
    \begin{theorem}[Strong Duality and Range of Lagrange Multipliers]\label{thm:Existence of strong duality and bound on Lagrangian multiplier}
    The optimisation problem in Equation~\eqref{eqn:lag_relax} satisfies
    \begin{align}\label{eq:lag_lb}
        &\inf_{\lag\in\reals_{+}^{\numconst}} \min_{A' \in \cC} \sup_{\allocation\in\simp}\max_{\pol\in\rgoodest}\inf_{\blambda\in\altsetr} \allocation^{\top}\klvec - \lag^{\top} {A'}\allocation =\notag \\
        & \sup_{\allocation\in\simp} \min_{\lag\in\cL}\max_{\pol\in\rgoodest}\inf_{\blambda\in\altsetr} \allocation^{\top}\klvec - \lag^{\top} \Tilde{A}\allocation\,.
    \end{align}   

    Here, $\cL \defn \{\lag \in \reals_+^{\numconst}\mid 0\leq \Vert {\lag}\Vert_1 \leq \frac{1}{\gamma} \mathcal{D}(\allocation,\bmu,\mathcal{F})\}$,
    where $\gamma \defn \min_{i\in[1,{\numconst}]} \{ -\Tilde{A}^i \hat{\opt}\}$, i.e. the minimum slack w.r.t. the estimated optimal feasible policy. 
\end{theorem}

\else
    \begin{theorem}[Strong Duality and Range of Lagrange Multipliers]\label{thm:Existence of strong duality and bound on Lagrangian multiplier}
    The optimisation problem in Equation~\eqref{eqn:lag_relax} satisfies
    \begin{align}\label{eq:lag_lb}
        \inf_{\lag\in\reals_{+}^{\numconst}} \min_{A' \in \cC} \sup_{\allocation\in\simp}\max_{\pol\in\rgoodest}\inf_{\blambda\in\altsetr} \allocation^{\top}\klvec - \lag^{\top} {A'}\allocation = \sup_{\allocation\in\simp} \min_{\lag\in\cL}\max_{\pol\in\rgoodest}\inf_{\blambda\in\altsetr} \allocation^{\top}\klvec - \lag^{\top} \Tilde{A}\allocation\,.
    \end{align}   

    Here, $\cL \defn \{\lag \in \reals_+^{\numconst}\mid 0\leq \Vert {\lag}\Vert_1 \leq \frac{1}{\gamma} \mathcal{D}(\allocation,\bmu,\mathcal{F})\}$,
    where $\gamma \defn \min_{i\in[1,{\numconst}]} \{ -\Tilde{A}^i \hat{\opt}\}$, i.e. the minimum slack w.r.t. the estimated optimal feasible policy. 
\end{theorem}
\fi
\end{tcolorbox}

Detailed proof is in Appendix \ref{sec:Strong duality proof}. Hereafter, we use the RHS of Eq.~\eqref{eq:lag_lb} as $T_{\Fest,r}^{-1}(\bmu)$. 
Theorem~\ref{thm:Existence of strong duality and bound on Lagrangian multiplier} provides a hypercube to search for the Lagrangian multipliers, which is a linear programming problem.  %We denote it by $D(\allocation,\bmu,\cF)$.

\ifdoublecol
\begin{figure*}[t!]
  \centering
 \begin{minipage}{1.3\columnwidth}
  \centering%\vspace*{-.8em}
    \includegraphics[width=0.8\linewidth,trim={1cm 6cm 1cm 2.5cm},clip]{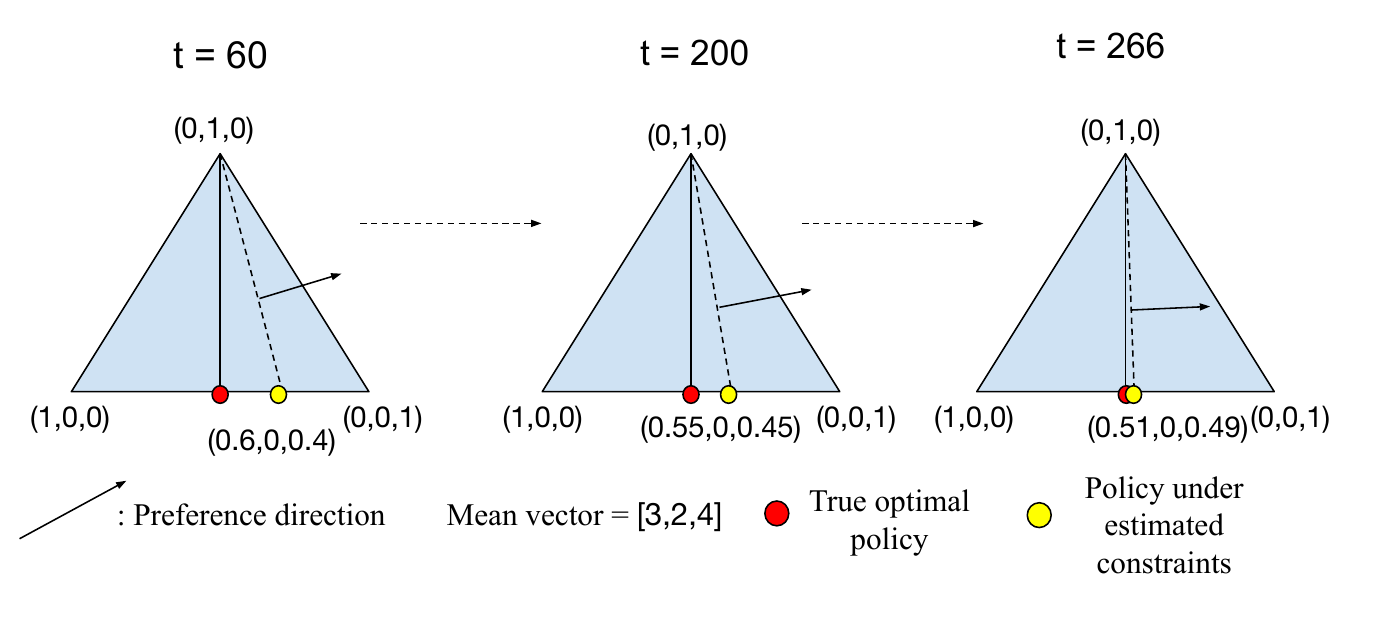}\vspace*{-1em}
    \caption{Convergence of the optimistic feasible set and optimal policy.}\label{fig:Constraint_Propagation}%\vspace*{-1.5em}
 \end{minipage}\hspace{2mm}
\begin{minipage}{0.64\columnwidth}
    \centering
    \includegraphics[width=0.8\textwidth,trim={2.5cm 0.3cm 4.5cm 4cm},clip]{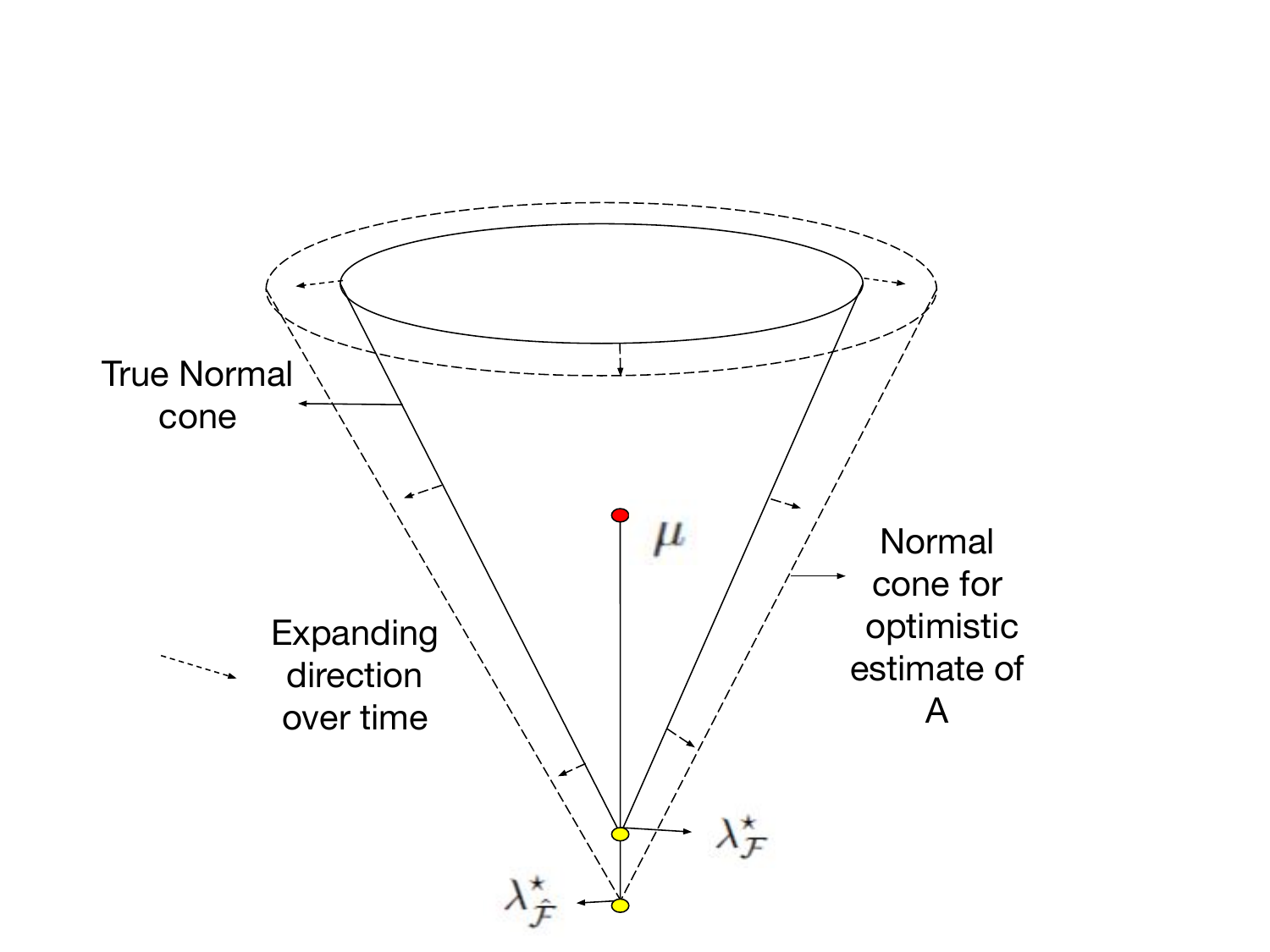}%\vspace*{-.5em}
    \caption{Normal cones over time.}\label{fig:normal_cone}
\end{minipage}\vspace*{-1.5em}
\end{figure*}
\else
\begin{figure}[t!]
  \centering
 \begin{minipage}{0.56\textwidth}
  \centering%\vspace*{-.8em}
    \includegraphics[width=\textwidth,trim={1cm 6cm 1cm 2.5cm},clip]{Plots/Constraint_prop.pdf}%\vspace*{-1em}
    \caption{Convergence of the optimistic feasible set and optimal policy.}\label{fig:Constraint_Propagation}%\vspace*{-1.5em}
 \end{minipage}\hfill
\begin{minipage}{0.43\textwidth}
    \centering
    \includegraphics[width=0.8\textwidth,trim={2.0cm 0.3cm 4.5cm 4cm},clip]{Plots/Normal_cone.pdf}%\vspace*{-.5em}
    \caption{Normal cones over time.}\label{fig:normal_cone}
\end{minipage}%\vspace*{-1.5em}
\end{figure}
\fi

\begin{remark}[Connections with Lagrangian-based Methods in Bandits.]
Regret minimisation literature  leverages Lagrangian-based optimistic-pessimistic methods~\citep{NEURIPS2020_0f34314d,pmlr-v195-foster23c} to obtain both the sub-linear regret and constraint violation guarantees~\citep{liu2021efficient,bernasconi2024noregret}. Our proposed algorithm LAGEX (Algorithm~\ref{alg:lagex}) is a prime example where the ``self-boundedness'' of the dual variables results in tighter constraint violation guarantees (Figure \ref{fig:constraint_violation_hard} and \ref{fig:Constraint_violation_easy}). %It would be interesting to see how LAGEX performs in regret minimisation setting.
\end{remark}
\textbf{I. The Inner Optimisation Problem.} Now, we peel the layers of the optimisation problem in Eq.~\eqref{eq:lag_lb} and focus on obtaining
\begin{align*}
	\mathcal{D}(\allocation,\bmu,\Fest)\triangleq \min_{\lag\in\cL}\max_{\pol\in\rgoodest}\inf_{\blambda\in\altsetr} \allocation^{\top}\klvec - \lag^{\top} \Tilde{A}\allocation\,.%\vspace{-1em}
\end{align*}

% \begin{figure}[t!]%{l}{7cm}
% \centering%\vspace*{-1em}
% %\begin{subfigure}[b]{\textwidth}
% \includegraphics[width=0.2\textwidth]{Plots/Normal_cone.pdf}
% %\end{subfigure}
% \caption{Structure of estimated normal cone}\label{fig:normal_cone}%\vspace*{-3em}
% \end{figure}
For known constraints and $r=0$, \cite{carlsson2024pure} has leveraged results from convex analysis~\citep{Boyd_Vandenberghe_2004} to show that the most confusing instance for $\bmu$ lie in the boundary of the normal cone $\altsetreal^{C}$ spanned by the active constraints $A_{\opt_{\cF}}$ for $\opt_{\cF}$. $A_{\opt_{\cF}}$ is a sub-matrix of $A$ consisting at least $K$ linearly independent rows. \textit{This is called the projection lemma. }
Specifically, $\mathcal{D}(\allocation,\bmu,\cF \mid r=0) = \max_{\pol\in\Pi_{\cF}^0} \min_{\pol^{'}\in\nu_{\cF}(\pol)} \min_{\blambda:\blambda^{\top}(\pol-\pol^{'})=0}  \allocation^{\top}\klvec.$

%\ifdoublecol
%$\mathcal{D}(\allocation,\bmu,\cF) \defn \max_{\pol\in\rgoodest}\inf_{\blambda\in\altsetr} \allocation^{\top}\klvec = \max_{\pol\in\rgoodest} \min_{\pol^{'}\in\nu_{\cF}(\pol)} \min_{\blambda:\blambda^{\top}(\pol-\pol^{'})=r}  \allocation^{\top}\klvec$.
%\else
%$\mathcal{D}(\allocation,\bmu,\cF) \defn \inf_{\blambda\in\altset} \allocation^{\top}\klvec = \min_{\pol^{'}\in\nu_{\cF}(\opt_{\cF})} \min_{\blambda:\blambda^{\top}(\opt_{\cF}-\pol^{'})=0}  \allocation^{\top}\klvec$.
%\fi
In our setting, we are sequentially estimating both the mean vectors and the constraints, and thus, the normal  cone.  Now, we derive the projection lemma for our optimisti feasible set and $r\neq 0$.
\begin{tcolorbox}[left=2pt,right=2pt,top=1pt,bottom=1pt]
 \ifdoublecol
\begin{proposition}[Projection Lemma for Unknown Constraints]\label{prop:Projection Lemma for Lagrangian Formulation}
	For any $\allocation \in \Fest$ and $\bmu\in\cD$, the following projection lemma holds for the Lagrangian relaxation in Equation~\eqref{eqn:lag_relax},
\small{\begin{align}\label{eqn:Projection Lemma for Lag}
			&\mathcal{D}(\allocation,\bmu,\Fest)=\notag \\  &  \min_{\lag\in\cL}\max_{\pol \in \rgoodest} \min_{\pol^{'}\in\nu_{\Fest}(\pol)} \min_{\blambda:\blambda^{\top}(\pol-\pol^{'})=r} \allocation^{\top}\klvec - \lag^{\top} \Tilde{A}\allocation\, . 
	\end{align}}
\end{proposition}%\vspace*{-1em}
\else
\begin{proposition}[Projection Lemma for Unknown Constraints]\label{prop:Projection Lemma for Lagrangian Formulation}
	For any $\allocation \in \Fest$ and $\bmu\in\cD$, the Lagrangian relaxation in Equation~\eqref{eqn:lag_relax} satisfies 
    {\begin{align*}\label{eqn:Projection Lemma for Lag}
		 \mathcal{D}(\allocation,\bmu,\Fest) = \min_{\lag\in\cL}\max_{\pol \in \rgoodest} \min_{\pol^{'}\in\nu_{\Fest}(\pol)} \min_{\blambda:\blambda^{\top}(\pol-\pol^{'})=r} \allocation^{\top}\klvec - \lag^{\top} \Tilde{A}\allocation\, . 
	\end{align*}}
\end{proposition}%\vspace*{-1em}
\fi   
\end{tcolorbox}	

\iffalse
Let $\Tilde{A}_{\opt}$ be the sub-matrix spanned by the active constraints for $\opt$. Since we are working with the optimisti estimate of $A$, the vector space spanned by the linearly independent rows of $A_{\opt}$ is a subset of the vector space spanned by those of $\Tilde{A}_{\rpolest}$. Since the estimated Alt-set $\altset$ is always a subset of the true Alt-set $\altsetreal$, the normal cone around $\rpolest$ is always a superset of the true normal cone of $\opt$. %We illustrate them in Figure~\ref{fig:normal_cone}. 
%We now extend the projection lemma for known constraints to the Lagrangian formulation with unknown constraints.
\fi

Proposition \ref{prop:Projection Lemma for Lagrangian Formulation} reduces the inner minimisation problem to a less intensive discrete optimisation, where we only have to search over the neighbouring vertices of the optimal policy in $\Fest$ for a solution. 
Now, a natural question arises around this formulation: 
\begin{tcolorbox}[left=2pt,right=2pt,top=1pt,bottom=1pt]
    Can we track $\mathcal{D}(\allocation,\bmu,\Fest)$ in the projection lemma over time as we sequentially estimate the constraints?
\end{tcolorbox}

To answer this, we first prove convergence of estimated feasible set and alternating instance in order to prove convergence of $\mathcal{D}(\allocation,\bmu,\Fest_t)$ in the theorem below. Detailed proof is in Appendix \ref{sec:prop of Fhat}.

\begin{tcolorbox}[left=2pt,right=2pt,top=1pt,bottom=1pt]
    \begin{theorem}\label{thm:Asymptotic Convergence of the Projection Lemma}
    For a sequence $\{ \Fest_t \}_{t\in\mathbb{N}}$ and $\{\Hat{\blambda}_t \}_{t\in\mathbb{N}}$, we show that
    (a) $\lim_{t\rightarrow\infty}\Fest_t \rightarrow \cF$,
    (b) $\blambda^{\star}$ is unique, and 
    (c) $\lim_{t\rightarrow\infty}\Hat{\blambda}_t \rightarrow \blambda^{\star}$. 
    Then, for any $\allocation \in \cF$ and $\bmu$, 
    $\lim_{t\rightarrow\infty}\mathcal{D}(\allocation,\bmu,\Fest_t) \rightarrow \mathcal{D}(\allocation,\bmu,\cF),
    $
    where $\blambda^* \in \argmin_{\blambda \in \altsetreal}\allocation^{\top} \klvec$. %Proof is in Appendix \ref{sec:prop of Fhat}.\vspace*{-.5em}
\end{theorem}
\end{tcolorbox}

\iffalse
\begin{figure}[t!]%{0.4\textwidth}\vspace{-.8em}
 \centering
    \includegraphics[width=0.4\columnwidth,trim={2.5cm 0.3cm 4.5cm 4cm},clip]{Plots/Normal cone.pdf}\vspace*{-.5em}
    \caption{Evolution of the normal cone.}\label{fig:normal_cone}
\end{figure}
\fi

\textbf{II. The Outer Optimisation Problem.}
As we guarantee the convergence of $\mathcal{D}(\allocation,\bmu,\Fest_t)$ as $\Fest_t \rightarrow \cF$, we are left with the outer optimisation problem in Equation~\eqref{eq:lag_lb}. Since it is a linear problem in $\allocation$, we can use a linear programming method leading to a vertex of $\Fest_t$. 
But to be sure of an existence of a solution at each $t \in \mathbb{N}$, we need well-behavedness properties of the optimal allocation $\allocation^*(\bmu)$. 
First, we observe that our estimates of the mean vector converge to $\bmu$ as $t \rightarrow \infty$. 
Hence, we also get $\lim_{t\rightarrow \infty}\mathcal{D}(\allocation,\Hat{\bmu}_t,\Fest_t) \rightarrow \mathcal{D}(\allocation,\bmu,\cF)$. 
Now, we ensure well-behavedness and existence of an optimal allocation for all $t>0$. 

\begin{tcolorbox}[left=2pt,right=2pt,top=1pt,bottom=1pt]
  \begin{theorem}[Existence of unique optimal allocation]\label{thm:Optimal Policy}
    For all $\bmu\in\cD$, $\allocation^{\star}(\bmu)$ satisfies: 
    1. Both the sets $\Fest$ and $\allocation^{\star}(\bmu)$ are closed and convex. 
    2. For all $\bmu\in\cD$ and $\allocation\in\Fest$,  $\lim_{t\rightarrow\infty}\mathcal{D}(\allocation,\Hat{\bmu}_t,\Fest_t)$ is continuous. 
    3. Reciprocal of the characteristic time $\lim_{t\rightarrow\infty} T^{-1}_{\Fest_t,r}(\bmu)$ is continuous for all $\bmu \in \cD$. 
    4. For all  $\bmu \in \cD$, $\bmu \rightarrow \allocation^{\star}(\bmu)$ is upper hemi-continuous. 
    Thus, the optimization problem $\max_{\pol\in\Hat{{\cF}}} \bmu^{\top}\pol$ has a unique solution.
\end{theorem}  
\end{tcolorbox}

\textbf{Characterising Lower Bound for Gaussians.} Since we can derive explicit form of the optimisation problem for Gaussian reward distributions, we characterise it further to relate our lower bound with the lower bound for known constraints (Appendix \ref{sec:gaussian lb}).

\begin{tcolorbox}[left=2pt,right=2pt,top=1pt,bottom=1pt]
  \ifdoublecol
\begin{theorem}\label{Characteristic time for Gaussian Distribution}
   Let $\{ P_a \}_{a\in[K]}$ be Gaussian distributions with equal variance $\sigma^2 > 0$ , $\charactimer$ is
   \begin{align*}
   \max_{\allocation\in\simp}\min_{\lag\in\cL}\max_{\pol\in\rgoodest}\min_{\pol' \in \neighr}&\Bigg\{\frac{(r-\bmu^{\top}(\pol-\pol'))^2}{2\sigma^{2}\Vert\pol-\pol'\Vert_{\mathrm{Diag}(1/\allocation_a)}^2} \\
        &\qquad\qquad- \lag^{\top} \Tilde{A}\allocation \Bigg\}\,,
   \end{align*}
   where $\mathrm{Diag}(1/\allocation_a)$ is a $K$-dimensional diagonal matrix with $a$-th diagonal entry $1/\allocation_a$ and $\neighr$ is the set of neighbouring policies of $\pol$ in $\Fest$.
\end{theorem}
\else
\begin{theorem}\label{Characteristic time for Gaussian Distribution}
   Let $\{ P_a \}_{a\in[K]}$ be Gaussian distributions with equal variance $\sigma^2 > 0$ , $\charactimer$ is
   \begin{align*}
   \max_{\allocation\in\simp}\min_{\lag\in\cL}\max_{\pol\in\rgoodest}\min_{\pol' \in \neighr}&\Bigg\{\frac{(r-\bmu^{\top}(\pol-\pol'))^2}{2\sigma^{2}\Vert\pol-\pol'\Vert_{\mathrm{Diag}(1/\allocation_a)}^2} - \lag^{\top} \Tilde{A}\allocation \Bigg\}\,,
   \end{align*}
   where $\mathrm{Diag}(1/\allocation_a)$ is a $K$-dimensional diagonal matrix with $a$-th diagonal entry $1/\allocation_a$ and $\neighr$ is the set of neighbouring policies of $\pol$ in $\Fest$.
\end{theorem}
\fi  
\end{tcolorbox}

Now, to connect with the lower bound under known constraints, we fix $r=0$, i.e we search for the `true' optimal policy $\opt_{\mathcal{F}}$ rather than an $r$-optimal feasible policy. 

\begin{corollary}\label{cor:Bounds on characteristic time}
    Let $d_{\pol}^2 \defn  \frac{\Vert\opt_{\cF}-\pol\Vert_{\bmu\bmu^{\top}}^2}{\Vert\opt_{\cF}-\pol\Vert_2^2}$ be the norm of the projection of $\bmu$ on the policy gap $(\opt_{\cF}-\pol)$.\\
    \emph{Part i.} Then, we get 
    $$\frac{2\sigma^2 K}{C_{\mathrm{known}}}(1+\SP_{\Tilde{A}})\leq T_{\Fest,0}(\bmu)\leq \frac{2\sigma^2 K}{C_{\mathrm{known}}},$$ where $C_{\mathrm{known}} = \min_{\pol'' \in \neighreal}d_{\pol''}^2$. \\
    \emph{Part ii.}  $$T_{\Fest,0}(\bmu) \geq \frac{H}{\kappa^2}(1+\SP_{\Tilde{A}}),$$ where $\SP_{\Tilde{A}} \defn \frac{\max_{i\in[1,d]} \tildeA\opt_{\Fest}}{\min_{i\in[1,d]} \tildeA\opt_{\Fest}}$ is the estimated shadow price and $H$ is inversely proportional to the sum of squares of gaps and $\kappa_{\mathrm{known}}$ is condition number of a sub-matrix of $A$ with $\numarm$ linearly independent active constraints for $\opt$.%\\ Proof is in Appendix \ref{sec:cor 1}.
\end{corollary}
% Corollary \ref{cor:Bounds on characteristic time} directly gives us lower bound on the  
% expected stopping time of the $\delta$-PAC learner,
% \begin{align*}
%     \mathbb{E}[\tau_{\delta}] \geq \frac{2\sigma^2}{C_{\mathrm{known}}+2C_{\mathrm{unknown}}} \mathrm{kl}(\delta\|1-\delta)
% \end{align*}
% \begin{corollary}
% \begin{align*}
%     T_{\Fest}^{-1}(\bmu) = \sup_{\allocation \in \Fest} \inf_{\lag\in \mathbb{R}_{+}^{\numconst}} \inf_{\pol' \in \neigh} \left\{ \frac{1}{2\sigma^2}\frac{(\Delta^{\top} {\tildeA'}^{-1}e_r)^2}{\Vert{\tildeA'}^{-1}e_r  \Vert_{\text{Diag}(1/\allocation_a)}^2} - \lag^{\top} \tildeA \allocation\right\} 
% \end{align*}
% \begin{align*}
%     \mathbb{E}[\tau_{\delta}] \geq \frac{H}{\kappa_{\mathrm{known}}^2 +2\kappa_{\mathrm{unknown}}^2} \mathrm{kl}(\delta\|1-\delta)
% \end{align*}
% \end{corollary}
\begin{remark}[Connection to Existing Lower Bounds.] \textbf{(a) Pure exploration under known constraints.} The upper and lower bounds on characteristic time coincides with the existing lower bound under known constraints, i.e. when $\epsilon = 0$, i.e. $\Fest=\cF$. \textbf{(b) BAI without constraints.} In BAI, we consider only deterministic policies (or pure strategies) of playing a single arm. Then, we get $d_{\pol_a} = \frac{\bmu^{\top} (\opt_{\cF}-\pol)_a}{\vert(\opt_{\cF}-\pol)_a\vert} = \bmu^{\star}-\bmu_a$, i.e. the sub-optimality gap for arm $a$. Here, $\bmu^{\star}$ is the mean of the best arm. In our setting, if there are no constraints involved then $\Tilde{A}=A$ and so $\epsilon=0$. Then the lower bound expression in Part ii. of Corollary \ref{cor:Bounds on characteristic time} is inversely proportional to the sum-squared sub-optimal gaps which resonates with the complexity measure in standard BAI with only simplex constraints (\cite{kaufmann2016complexity}). Though the lower bound stated in Corollary \ref{cor:Bounds on characteristic time} successfully encompasses the effect of unknown linear constraint by introducing novel constraint dependent factors that scales the lower bound.
\end{remark}%is inversely proportional with squared sub-optimality gaps, but also had an extra term in the numerator which was novel in BAI literature. \cite{wang2021best} argued this term as the added cost of learning the safety threshold value for the i-th arm. In our setting, when $\Fest \to \cF$, $\mathbb{E}[\tau_{\delta}] \geq \frac{2}{3}\frac{\sigma^2}{\kappa_{\mathrm{known}}^2}\frac{1}{\Delta^2} \mathrm{kl}(\delta\|1-\delta)$ becomes comparable with the lower bound in \cite{wang2021best}. We can say the new term in the denominator explains the cost of handling unknown constraints.     

\section{{Algorithm Design}}\label{sec:algo}%\vspace*{-.5em}
Now, we propose two algorithms to conduct pure exploration with Lagrangian relaxation of lower bound, and derive upper bounds on their sample complexities.
\begin{assumption}[Distributional assumptions on rewards and constraints]\label{ass:distrubutional}
We require two distributional assumptions on rewards and constraints.
    (i) Reward distributions $\{ P_a\}_{a=1}^K$ are sub-Gaussian one parameter exponential family with mean vector $\bmu \in \cD$. 
    (ii) Each constraint follows a sub-Gaussian $\numarm$-parameter exponential family parameterised by $A^i$ for $i\in [\numconst]$.%\vspace*{-.5em}
\end{assumption}

These distributional assumptions are standard in bandits under constraints~\citep{carlsson2024pure,Degenne2019PureEW,pacchiano2020stochastic,pacchiano2024contextual}.

\noindent\textbf{Algorithm Design.} Any algorithm in pure exploration setting consists of three main components. The first one is a sequential hypothesis testing that decides whether we need to keep sampling or not, popularly known as the ``stopping criterion''. To come up with such a criterion we need to make sure that we have gathered sufficient information about all the parameters in estimation, specifically $\bmu$ and $A$ in our case. So it seems natural to enforce criterion on both mean and constraint concentration.   

%\vspace{0.1cm}
\begin{lemma}\label{lemm:new_stopping_time}
	If the recommended policy is $(1-\delta)$-correct then it is $(1-\delta)$-feasible.
\end{lemma}%\vspace{0.1cm}

Thus, while implementing Algorithm \ref{alg:lagts} and \ref{alg:lagex}, we just need to \textbf{check the first condition} to stop sampling. %Refer Appendix \ref{delta correct} for Proof. 

\textbf{Component 1: Stopping Rule.}
During exploration, once we gather enough statistical information about the parameters in the system, the test statistic crosses the stopping threshold with the chosen confidence $\delta$, and we stop to recommend the optimal policy.  

%\vspace{0.1cm}
\begin{tcolorbox}[left=2pt,right=2pt,top=1pt,bottom=1pt]
    \begin{theorem}\label{Chernoff's Stopping rule with unknown constraints.}
The Chernoff \textbf{stopping rule to ensure $(1-\delta)$-correctness and $(1-\delta)$-feasibility} is%\vspace*{-.5em}
\begin{align*}
&\max_{\pol\in\Pi_{\Fest_t}^r}\inf_{\blambda \in \Lambda_{\Fest_t}(\Hat{\bmu}_t,\pol)} \sum\nolimits_{a=1}^K N_{a,t} d(\Hat{\bmu}_{a,t},\blambda_a) > \beta( t,\delta)\,,%\vspace*{-1em}
\end{align*}
where $\beta(t,\delta) \defn  3S_0 \log(1+\log N_{a,t}) + S_0\mathcal{T}\left( \frac{(\numarm\wedge\numconst) + \log\frac{1}{\delta}}{S_0}\right)$ and $0\leq S_0 \leq \numarm$. 
\end{theorem}
\end{tcolorbox}

%\ifdoublecol
\setlength{\textfloatsep}{10pt}
\begin{algorithm}[t!]%\scriptsize{7pt}
\caption{LATS - \textbf{LA}grangian \textbf{T}rack and \textbf{S}top}\label{alg:lagts}
\begin{algorithmic}[1]
    \STATE \textbf{Input :} Tolerance $r>0$, Confidence level $\delta>0$%, initial variance $v>0$
    \STATE \textbf{Initialization :} $\Hat{A}_0 = \textbf{0}_{{\numconst}\times K}$, $\Hat{\bmu}_0 = \textbf{0}_K$, $\Sigma_0 = v\mathbf{I}_{K}$, $l_0$ 
    \STATE Play each arm once to set $\bmu_1$ and $\Hat{A}_1$.
    \WHILE
    {$\beta(t-1,\delta)>\mathcal{D}(\allocation_{t-1}^{\star},\Hat{\bmu}_{t-1},\Fest_{t-1},\lag_{t-1}^{\star})$} %{\color{blue}$\rightsquigarrow$ Proposition~\ref{prop:Projection Lemma for Lagrangian Formulation}}    
    \STATE $\pol_t = \argmax_{\pol \in \Pi_{\Fest_t}^r}\Hat{\bmu}_{t-1}^{\top}\pol$
    \STATE {\color{blue}\textbf{Optimal allocation:}} $\allocation_t^{\star} \in \argmax_{\allocation\in\Fest_t}\mathcal{D}(\allocation_{t-1},\Hat{\bmu}_{t-1},\pol_{t},\Fest_{t-1},\lag_{t-1}^{\star})$ 
    \STATE \textbf{\color{blue}{Optimize Lagrangian Multiplier :}} $\lag_t^{\star} \in \argmin_{\lag\in\cL_t}\mathcal{D}(\allocation_{t}^{\star},\Hat{\bmu}_{t-1},\pol_{t-1},\Fest_{t-1},\lag)$ 
    \STATE \textbf{C-Tracking:} Play $a_t \in \argmin_{a\in[1,K]} N_{a,t-1} - \sum_{s=1}^t \allocation_{a,s}^{\star}$
    \STATE \textbf{Feedback :} Observe reward $r_t$ and cost $\cost_t$, and update $\Sigma_{t}$, $\Hat{\bmu}_{t}$ and $\Hat{A}_{t}$
    % \STATE such that $\Vert\lag_t^{\star}\Vert_1 \leq \frac{1}{\gamma}  \mathcal{D}(\omega_{t}^{\star},\Hat{\bmu}_t,\Fest_t) $, where, $\gamma = \min_{i\in[1,{\numconst}]} \{ -\Tilde{A}_t^i \omega_{t}^{\star}\}$
    %\STATE $t = t + 1$
    \ENDWHILE
    \STATE \textbf{Recommended policy:} $\hat{\bpi} = \argmax\limits_{\pol\in\Pi_{\Fest_t}^r} \Hat{\bmu}_t^{\top} \pol$
\end{algorithmic}
\end{algorithm}

\textbf{Component 2: Recommendation rule.} Once the stopping rule is fired, the agent recommends a policy based on the current estimate of $\Hat{\bmu}_t$ according the rule $\hat{\bpi} = \argmax_{\pol\in\Pi_{\Fest_{\stopping}}^r}\Hat{\bmu}_t^{\top} \pol.$
 
\textbf{Component 3: Sampling Strategy.} We present two novel sampling algorithms: LATS and LAGEX. % \dbcomment{what Tns does}

\paragraph{LATS.} The algorithm LATS (Algorithm~\ref{alg:lagts}) uses the Track and Stop strategy adapted to the unknown constraint setting. 
%We use {\color{blue}blue} markers in the pseudocode (Algorithm \ref{alg:lagts}) to denote the novel approaches taken to handle the challenge of estimating the feasible space per step. 
The algorithm warms up the parameter estimates by playing each arm once. Until the stopping rule is fired, it computes the optimal allocation (Line 4) under the estimated feasible space by solving the Lagrangian relaxed optimization problem in Proposition \ref{prop:Projection Lemma for Lagrangian Formulation}. By plugging the optimal allocation we perform a bounded optimisation on the Lagrangian multiplier using Theorem \ref{thm:Existence of strong duality and bound on Lagrangian multiplier}. The algorithm further uses C-tracking \citep{pmlr-v49-garivier16a} to track actions taken per step. Finally, we observe the instantaneous reward and cost feedback to update the parameter estimates (Line 9).

\begin{tcolorbox}[left=2pt,right=2pt,top=1pt,bottom=1pt]
  \begin{theorem}\label{thm:LATS upper bound}%\vspace*{-.5em}
Let $\SP$ be the shadow price $\SP \defn \frac{\Gamma_{\mathrm{max}}}{\Gamma_{\mathrm{min}}} = \frac{\max_{i\in[1,d]} (-A\opt_{\cF})}{\min_{i\in[1,d]} (-A\opt_{\cF})}$ of the slack $\Gamma$. Under Assumption~\ref{ass:structural} and~\ref{ass:distrubutional}, the expected stopping time of LATS satisfies $$\lim_{\substack{\delta\rightarrow 0}} \frac{\mathbb{E}[\tau_{\delta}]}{\log(1/\delta)} \leq \alpha T_{\cF,r}(\bmu) (1+\SP)$$ for any $\alpha>1$.
% \end{align*}
\end{theorem}  
\end{tcolorbox}

\textbf{Implications.} Theorem \ref{thm:LATS upper bound} suggests that LATS (Algorithm \ref{alg:lagts}) is asymptotically optimal up to problem dependent constants. This upper bound captures the effect of unknown linear constraints in shadow price, novel in the constrained bandit literature. If any one of the constraints in the bandit environment becomes more and more sensitive, minimum slack decreases, consequently $\SP$ increases and the identification problem becomes harder to solve. Thus, $\SP$ behaves as a \textit{sensitivity or stability parameter} that explains the hardness of the pure exploration problem in hand through the structure of the true constraints.  
\begin{algorithm}[t!]%\scriptsize{7pt.}
	\caption{LAGEX- \textbf{LA}grangian \textbf{G}amified \textbf{EX}plorer}\label{alg:lagex}
	\begin{algorithmic}[1]
		\STATE \textbf{Input :} Tolerance $r>0$, Confidence level $\delta>0$%, initial variance $v>0$
		% \STATE \textbf{Initialization :} $\Hat{A}_0 = \textbf{0}_{{\numconst}\times K}$, $\Hat{\bmu}_0 = \textbf{0}_K$, $\Sigma_0 = \blambda\mathbf{1}_K$ 
		\STATE Play each arm once to set $\bmu_1$ and $\Hat{A}_1$.
		\WHILE{$\beta(t-1,\delta)>\mathcal{D}(\allocation_{t-1}^{\star},\Hat{\bmu}_{t-1},\Fest_{t-1},\lag_{t-1}^{\star})$} %{\color{blue}($\rightsquigarrow$ Proposition~\ref{prop:Projection Lemma for Lagrangian Formulation})}
		\STATE $\pol_t = \argmax_{\pol \in \Pi_{\Fest_t}^r}\Hat{\bmu}_{t-1}^{\top}\pol$
		\STATE \textbf{Optimal allocation} $\allocation_t^{\star}$ {\color{blue}$\rightsquigarrow$ Using AdaGrad via Theorem \ref{Characteristic time for Gaussian Distribution}}
		\STATE \textbf{\color{blue}{Optimize Lagrangian Multiplier:}} \par $\lag_t^{\star} \in \argmin_{\lag\in\cL_t}\mathcal{D}(\allocation_{t-1},\Hat{\bmu}_{t-1},\pol_{t},\Fest_{t-1},\lag)$ 
		% \STATE such that $\Vert\lag_t^{\star}\Vert_1 \leq \frac{1}{\gamma}  \mathcal{D}(\omega_{t},\Hat{\bmu}_t,\Fest_t) $, where, $\gamma = \min_{i\in[1,{\numconst}]} \{ -\Tilde{A}_t^i \omega_{t}\}$
		\STATE \textbf{C-Tracking:} Play $a_t \in \argmin_{a\in[1,K]} N_{a,t-1} - \sum_{s=1}^t \allocation_{a,s}^{\star}$
		\STATE \textbf{Feedback :} Observe reward $r_t$ and cost $\cost_{t}$, and update $\Sigma_{t}$, $\Hat{\bmu}_{t}$ and $\Hat{A}_{t}$
		\STATE \textbf{Compute confusing instance :} $\blambda_t$ {\color{blue}$\rightsquigarrow$ Via Proposition \ref{prop:Projection Lemma for Lagrangian Formulation} plugging in $\allocation_t^{\star}$, $\lag_t^{\star}$} 
		\STATE \textbf{Confidence intervals:} for all $a\in[K]$ \par$[\alpha_{t,a},\beta_{t,a}] : \{ \zeta : N_{a,t}d(\Hat{\bmu}_{t,a},\zeta) \leq g(t)\}$\par
		$U_t^a = \max \big\{ \frac{g(t)}{N_{a,t}}, d(\alpha_{t,a},\blambda_{t,a}), d(\beta_{t,a},\blambda_{t,a})\big\}$
		\STATE \textbf{Update loss for regret minimizer:} Update with  $L_t
		= \langle \allocation_{t}^{\star} , U_{t}\rangle - {\lag_t^{\star}}^{\top}\Tilde{A}_t\allocation_t^{\star}$
		%\STATE $t = t + 1$
		
		\ENDWHILE
		\STATE \textbf{Recommended policy:} $\hat{\bpi} = \argmax\limits_{\pol\in\Pi_{\Fest_t}^r} \Hat{\bmu}_t^{\top} \pol$
	\end{algorithmic}
\end{algorithm}%\vspace*{-1em}

% \dbcomment{what GE in general does? whats different?}
\vspace{1mm}\textbf{LAGEX.}  Algorithms based on track and stop mechanism tend to fail in case of larger problems where efficient optimization becomes a challenge due to the use of a max-min oracle per step. To improve on this, we leverage the two-player zero sum game approach introduced in \cite{degenne2019nonasymptotic}. Algorithm \ref{alg:lagex} also starts by playing each arm once to warm up parameter estimates. Then it uses a \textbf{allocation player} (We have used AdaGrad) to optimize the allocation $\allocation_t$ in Line 5 against the most confusing instance w.r.t current estimate of $\bmu$ optimised by a \textbf{instance player} which minimizes $\sumk \allocation_{a,t} d(\Hat{\bmu}_{a,t},\blambda_a) - {\lag_t^{\star}}^{\top} \tildeA_t \allocation_t$ with respect to $\blambda \in \Lambda_{\Fest_t}(\bmu,\pol)$. Since our search space in closed and convex, the allocation player enjoys sub-linear regret of $\mathcal{O}(\sqrt{t\log t})$, whereas the instance player computes the best confusing instance using Proposition \ref{prop:Projection Lemma for Lagrangian Formulation}. Then in Line 11, Adagrad loss function is updated with a loss by introducing optimism as $U_t$ defined in Line 10. Rest of the mechanism goes as usual  
%LAGEX also uses C-tracking as LATS to track the actions taken per step. Finally, LAGEX observes the instantaneous reward and cost to update the estimates.% for the next optimization step.

\begin{tcolorbox}[left=2pt,right=2pt,top=1pt,bottom=1pt]
  \begin{theorem}[Asymptotic Optimality of LAGEX]\label{thm:LAGEX upper bound}
    Under Assumption~\ref{ass:structural} and~\ref{ass:distrubutional}, the expected sample complexity of LAGEX satisfies $$\lim_{\substack{\delta\rightarrow 0}} \frac{\mathbb{E}[\tau_{\delta}]}{\log(1/\delta)} \leq T_{\cF,r}(\bmu).$$
\end{theorem}  
\end{tcolorbox}

Detailed proofs of this section are in Appendix~\ref{app:algo_proofs}.

\section{Experimental Analysis}\label{sec:experiments}%\vspace*{-.5em}
Now, we empirically test performance of proposed algorithms and the baselines. We refer to Appendix~\ref{sec:Exp details} for additional experiments and other details necessary for reproducibility. Code is available at this \href{https://github.com/udvasdas/Pure-exploration-with-unknown-linear-constraints/}{Link}.\footnote{\textbf{Baselines:} ``CTnS-WLag" and ``CGE-WLag" is the CTnS and CGE algorithm of~\cite{carlsson2024pure} under unknown constraints without Lagrangian relaxation. ``PTnS" is Projected Track and Stop algorithm.  }

\iffalse
\begin{table}[h!]
	\centering
	\caption{Different experimental setups}\label{tab:exp}
	\resizebox{\textwidth}{!}{
		\begin{tabular}{cc|cc|c}
			\toprule  \textbf{Setup} & \textbf{Environment} & $\bmu$ & \textbf{Constraints} & $\delta$ \\ 
			\midrule $1$ & Easy & $[1.5, 1.0, 0.5, 0.4, 0.3, 0.2, 0.1]$ & $\pol_1 + \pol_2 + \pol_3 \leq 0.5$ and $\pol_4 + \pol_5 \leq 0.5$ & $0.01$\\
			\midrule$1$ & Hard & $[1.5, 1.0, 1.3, 0.4, 0.3, 0.2, 0.1]$ & $\pol_1 + \pol_2 + \pol_3 \leq 0.5$ and $\pol_4 + \pol_5 \leq 0.5$ & $0.01$\\
			\midrule$2$ & Easy & $[1,0.5,0.4,0.4,0.5]$ & $\pol_1 + \pol_2 \leq 0.5$ and $\pol_3 + \pol_4 \leq 0.5$ & $0.1$\\
			\midrule$2$ & Easy & $[1,0.5,0.4,0.95,0.8]$ & $\pol_1 + \pol_2 \leq 0.5$ and $\pol_3 + \pol_4 \leq 0.5$ & $0.1$ \\
			\midrule$\times$ & IMDB & $[3.67, 2.97, 2.94, 3.52, 3.18, 2.02, 2.79, 2.96, 2.37, 2.53, 2.55, 2.54]$ & "action" $<$ 0.3, "drama" $>$ 0.3, "family" $>$ 0.3 & 0.1\\
			\bottomrule
	\end{tabular}}
	\label{Experimental setups}
	%\caption{Summary of Notations}
\end{table}
\fi
\textbf{Synthetic Data: Setup 1} 
We evaluate with two environments having means $[1.5, 1.0, \mu_3, 0.4, 0.3, 0.2, 0.1]$. 

%\begin{table}
%	\begin{tabular}{|l|l|l|l|l|l|l|}
%		\hline
%		\multirow{2}{*}{Env} &
%		\multicolumn{2}{c}{Setup 1} &
%		\multicolumn{2}{c}{Setup 2} &
%		\multicolumn{2}{c|}{Setup 3} \\
%		& $\mu_{\text{M}}$ & $\sigma$ & $\mu_{\text{M}}$ & $\sigma$ & 	$\mu_{\text{M}}$ & $\sigma$ \\
%		\hline
%		D1 & 2.1\% & 2.1\% & 2.1\% & 2.1\% & 2.1\% & 2.1\% \\
%		\hline
%		D2 & 11.6\% & 11.6\% & 11.6\% & 11.6\% & 11.6\% & 11.6\% \\
%		\hline
%		D3 & 5.5\% & 5.5\% & 5.5\% & 5.5\% & 5.5\% & 5.5\% \\
%		\hline
%end{table}

\textbf{Observation 1: Universality.} We vary $\mu_3$ from $0.5$ to $2.5$. For each environment, we plot the corresponding unconstrained BAI lower bounds (in {\color{red}red}) and lower bounds under constraints (in {\color{blue}blue}) in Figure~\ref{fig:Hardness}. We observe that the constraint problem gets easier with increasing $\mu_3$. In contrast, the BAI problem changes non-monotonically. BAI problem gets harder when $\mu_3$ is around $1.5$ as the suboptimality gap gets very small. But the constraint problem stays easier than BAI. In Figure~\ref{fig:Hardness}, we also plot the median sample complexity of LAGEX across these environments over 500 runs. LAGEX grows parallel to the lower bound under constraints and can track it across environments.

\textbf{Observation 2: Efficiency.} 
We run all algorithms in two environment: (i) \textbf{hard} with $\mu_3 = 0.5$ and (ii) \textbf{easy} with $\mu_3 = 1.3$. We call the first environment hard as it is harder than solving BAI and similarly, the second environment easy. In Figure~\ref{fig:Stopping_times_hard} and~\ref{fig:stopping_times_easy} we observe that (i) among the algorithms with unknown constraints LAGEX incur the least sample complexity, and (ii) we pay a minimal cost than the known constraint Lagrangian algorithms in \textbf{hard env} whereas the price of estimating constraints is prominent in \textbf{easy env}.
\begin{figure*}[t!]
	\centering\vspace*{-1em}
	\begin{minipage}{0.48\textwidth}
		\centering
		\includegraphics[width=0.7\textwidth]{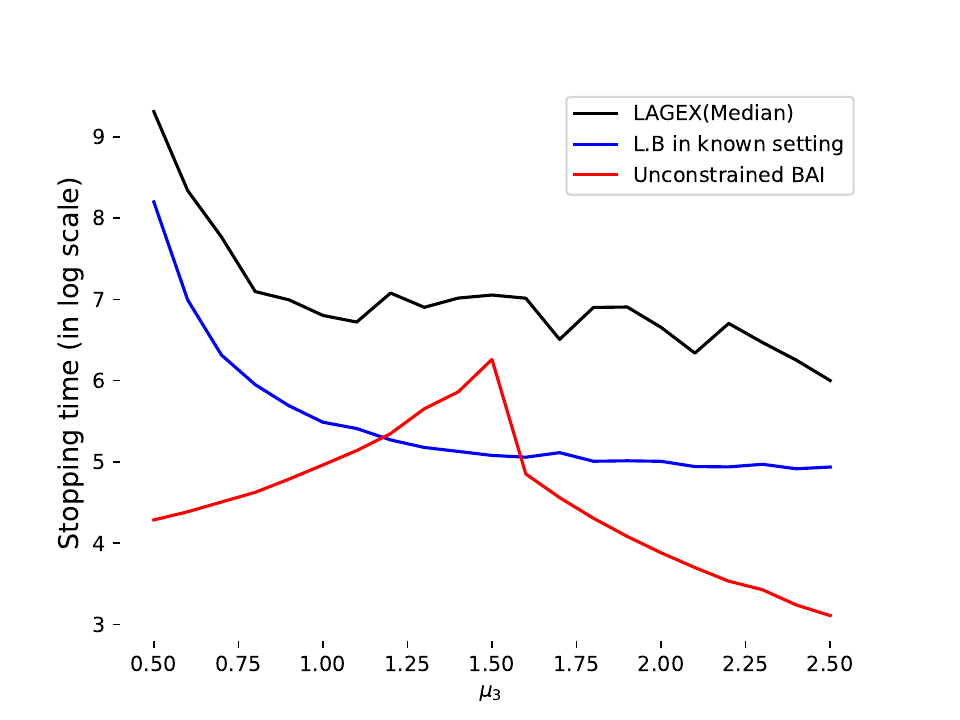}\vspace{-.5em}
		\caption{Lower bounds with and without constraints, and LAGEX for $\mu_3 \in[0.5,2.5]$ in Setup 1.}\label{fig:Hardness}
	\end{minipage}\hfill
	\begin{minipage}{0.48\textwidth}
		\centering
		\includegraphics[width=0.7\textwidth]{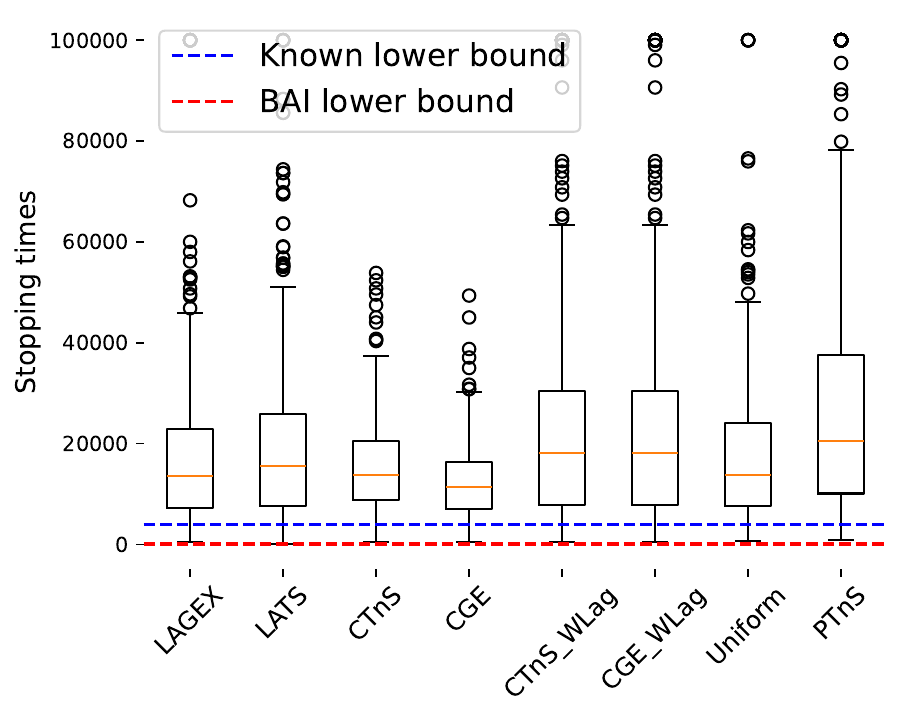}\vspace{-.5em}
		\caption{Sample complexity (median$\pm$std.) of algorithms  for \textbf{hard env.} in Setup 1.}\label{fig:Stopping_times_hard}
	\end{minipage}
\end{figure*}

\begin{figure*}[t!]
    \begin{minipage}{0.48\textwidth}
		\centering
		\includegraphics[width=0.7\textwidth]{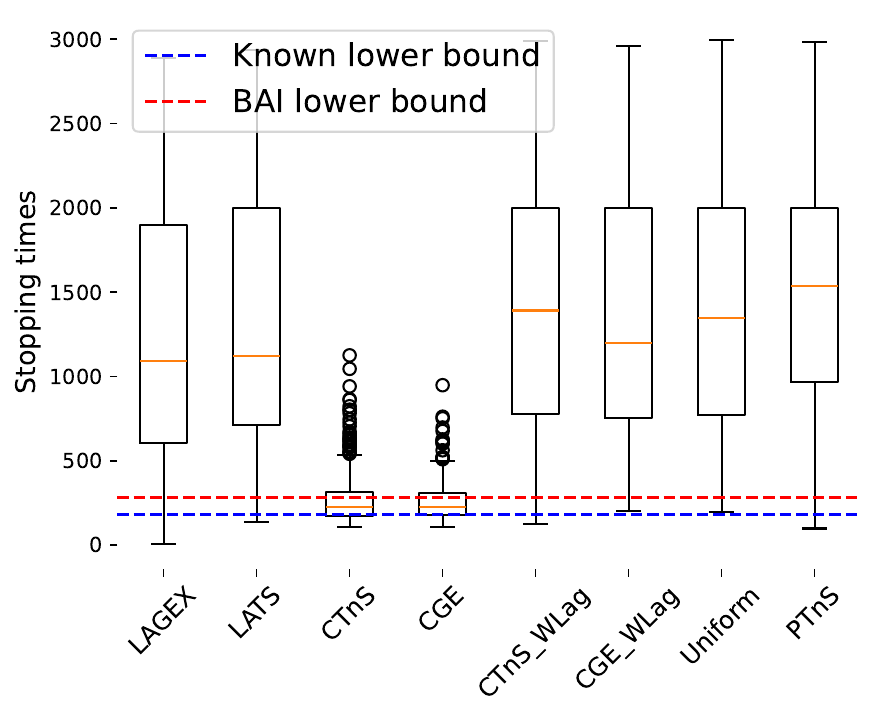}\vspace{-.5em}
		\caption{Sample complexity (median$\pm$std.) of algorithms for \textbf{easy env.}  in Setup 1.}\label{fig:stopping_times_easy}
    \end{minipage}\hfill
    \begin{minipage}{0.48\textwidth}
		\centering
		\includegraphics[width=0.7\textwidth]{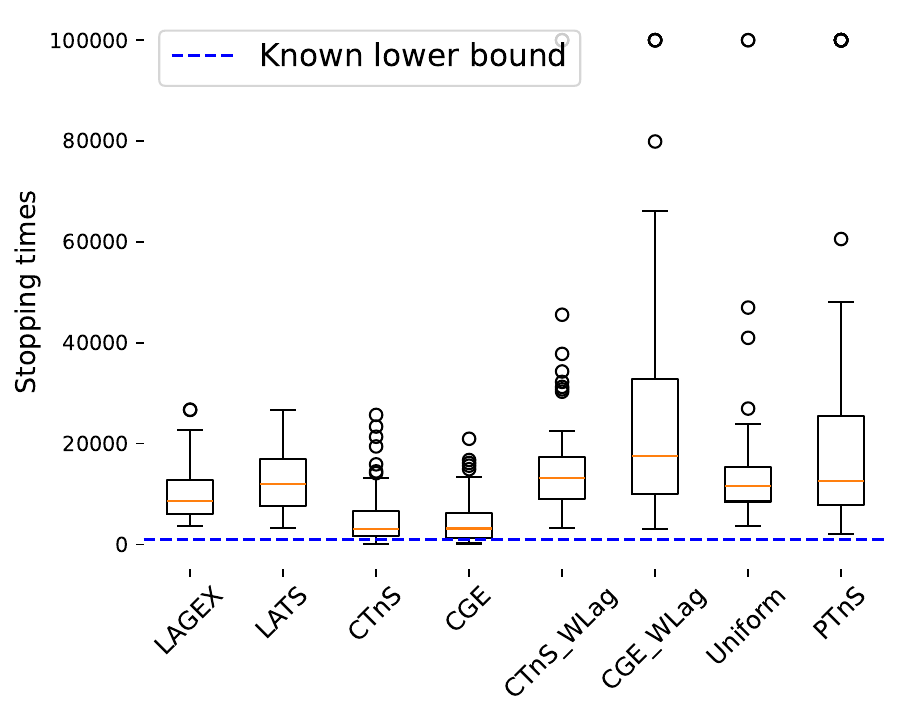}\vspace{-.5em}
		\caption{Sample complexity (median$\pm$std.) of algorithms for \textbf{IMDB-50K}.}\label{fig:stopping_times_imdb}
    \end{minipage}%\vspace*{-1.5em}
\end{figure*}

\iffalse
\begin{figure}[b]
	\begin{minipage}{0.48\textwidth}
		\centering
		\includegraphics[width=0.9\textwidth]{Plots/Stopping_times_emil_hard.pdf}
		\caption{Sample complexity (median$\pm$std.) of algorithm for \textbf{hard env. in Setup 2}.}\label{fig:Stopping_times_emil_hard}
	\end{minipage}\hfill
	\begin{minipage}{0.48\textwidth}
		\centering
		\includegraphics[width=0.9\textwidth]{Plots/Stopping_times_emil_easy.pdf}
		\caption{Sample complexity (median$\pm$std.) of algorithms for \textbf{easy env in Setup 2}.}\label{fig:stopping_times_emil_easy}
	\end{minipage}%\vspace*{-1.5em}
\end{figure}
\fi
\iffalse
\textbf{Synthetic Data: Setup 2} 
We further test our proposed algorithms LATS and LAGEX under the same two environments used in Figure 3b and 3c of ~\citep{carlsson2024pure}. Figure \ref{fig:Stopping_times_emil_hard} and \ref{fig:stopping_times_emil_easy} shows the comparison of sample complexity (Median $\pm$ s.d) for LATS and LAGEX with other baseline algorithms. 

\textbf{Observation: Efficiency.} For the \textbf{hard env} (As described in ~\citep{carlsson2024pure}), it seems clear from Figure \ref{fig:Stopping_times_emil_hard} that though both LATS and LAGEX shows minimally better efficiency than Uniform sampling, LAGEX shows minimum sample compleity in the unknown constraint setup. On the contrary, in \textbf{easy env} the betterment in efficiency for LAGEX is more prominent.   
\fi
\textbf{Real Data: IMDB-50K Dataset}
We evaluate LATS, LAGEX, and the baselines on IMDB 50K dataset~\citep{maas-etal-2011-learning}. For ease of comparison, we use the same bandit environment as in~\citep{carlsson2024pure}. We use 12 movies. We search for the optimal policy which allocates weight at most $0.3$ to action movies and at least $0.3$ to family and drama movies. %The true optimal policy is $[0.3,0.3,0,0,0.4,0,0,0,0,0,0,0]$. We assume $\delta = 0.1$ and compare the same set of algorithms as Setup 1.

\textbf{Observation: LAGEX has better sample complexity.} From Figure \ref{fig:stopping_times_imdb}, we observe that LAGEX performs better than other algorithms in the unknown constraints setting. LATS also performs well on the IMDB environment but notably we cannot distinguish its performance from that of the Uniform explorer.

\vspace*{-.5em}\section{{Discussion \& Future Work}}\label{Conclusion and future scope and limitations}
We study pure exploration under unknown linear constraints. We encode the coupled effect of estimating both mean vector and constraints via a Lagrangian relaxation of the lower bound for known constraints. We further design an optimistic estimate of the feasible set to ensure identification of the $r$-optimal feasible policy- leading to two algorithms, LATS and LAGEX. We prove their sample complexity upper bounds and conduct numerical experiments to observe that LAGEX is the most efficient among baselines. 

From the lower bound perspective, one might be interested to derive the lower bounds when both constraints and rewards are considered as feedback. %In this paper, we focus on linear constraints on policies. %One might also be interested in non-linear constraints over the policy space. 
Algorithmically, it would be intriguing to extend our Lagrangian technique to non-linear constraints.

\vspace*{-1em}\section*{Acknowledgments}
We acknowledge the ANR JCJC project REPUBLIC (ANR-22-CE23-0003-01), the PEPR project FOUNDRY (ANR23-PEIA-0003), and the Inria-ISI Kolkata associate team SeRAI for supporting the project. We also thank Achraf Azize for interesting discussions.

\bibliographystyle{apalike}
\bibliography{ref}

\begin{thebibliography}{}

\bibitem[Abbasi-yadkori et~al., 2011]{NIPS2011_e1d5be1c}
Abbasi-yadkori, Y., P\'{a}l, D., and Szepesv\'{a}ri, C. (2011).
\newblock Improved algorithms for linear stochastic bandits.
\newblock In {\em Advances in Neural Information Processing Systems}. Curran
  Associates, Inc.

\bibitem[Agrawal and Devanur, 2016]{NIPS2016_f3144cef}
Agrawal, S. and Devanur, N. (2016).
\newblock Linear contextual bandits with knapsacks.
\newblock In Lee, D., Sugiyama, M., Luxburg, U., Guyon, I., and Garnett, R.,
  editors, {\em Advances in Neural Information Processing Systems}, volume~29.
  Curran Associates, Inc.

\bibitem[Agrawal and Devanur, 2014]{10.1145/2600057.2602844}
Agrawal, S. and Devanur, N.~R. (2014).
\newblock Bandits with concave rewards and convex knapsacks.
\newblock EC '14, page 989–1006, New York, NY, USA. Association for Computing
  Machinery.

\bibitem[Agrawal et~al., 2016]{pmlr-v49-agrawal16}
Agrawal, S., Devanur, N.~R., and Li, L. (2016).
\newblock An efficient algorithm for contextual bandits with knapsacks, and an
  extension to concave objectives.
\newblock In Feldman, V., Rakhlin, A., and Shamir, O., editors, {\em 29th
  Annual Conference on Learning Theory}, volume~49 of {\em Proceedings of
  Machine Learning Research}, pages 4--18, Columbia University, New York, New
  York, USA. PMLR.

\bibitem[Amani et~al., 2019]{amani2019linear}
Amani, S., Alizadeh, M., and Thrampoulidis, C. (2019).
\newblock Linear stochastic bandits under safety constraints.

\bibitem[Auer et~al., 2002]{auer2002finite}
Auer, P., Cesa-Bianchi, N., and Fischer, P. (2002).
\newblock Finite-time analysis of the multiarmed bandit problem.
\newblock {\em Machine learning}, 47(2-3):235--256.

\bibitem[Aziz et~al., 2021a]{aziz2021multi}
Aziz, M., Kaufmann, E., and Riviere, M.-K. (2021a).
\newblock On multi-armed bandit designs for dose-finding clinical trials.
\newblock {\em The Journal of Machine Learning Research}, 22(1):686--723.

\bibitem[Aziz et~al., 2021b]{JMLR:v22:19-228}
Aziz, M., Kaufmann, E., and Riviere, M.-K. (2021b).
\newblock On multi-armed bandit designs for dose-finding trials.
\newblock {\em Journal of Machine Learning Research}, 22(14):1--38.

\bibitem[Badanidiyuru et~al., 2018]{10.1145/3164539}
Badanidiyuru, A., Kleinberg, R., and Slivkins, A. (2018).
\newblock Bandits with knapsacks.
\newblock {\em J. ACM}, 65(3).

\bibitem[Baudry et~al., 2024]{baudry2024multi}
Baudry, D., Merlis, N., Molina, M.~B., Richard, H., and Perchet, V. (2024).
\newblock Multi-armed bandits with guaranteed revenue per arm.
\newblock In {\em International Conference on Artificial Intelligence and
  Statistics}, pages 379--387. PMLR.

\bibitem[Bechhofer and Blumenthal, 1962]{08293709-a226-3cbe-aa4c-5f6306df76a0}
Bechhofer, R.~E. and Blumenthal, S. (1962).
\newblock A sequential multiple-decision procedure for selecting the best one
  of several normal populations with a common unknown variance, ii: Monte carlo
  sampling results and new computing formulae.
\newblock {\em Biometrics}, 18(1):52--67.

\bibitem[Berge, 1963]{berge1963topological}
Berge, C. (1963).
\newblock {\em Topological Spaces: Including a Treatment of Multi-valued
  Functions, Vector Spaces and Convexity}.
\newblock Macmillan.

\bibitem[Bernasconi et~al., 2024]{bernasconi2024noregret}
Bernasconi, M., Castiglioni, M., and Celli, A. (2024).
\newblock No-regret is not enough! bandits with general constraints through
  adaptive regret minimization.

\bibitem[Boyd and Vandenberghe, 2004]{Boyd_Vandenberghe_2004}
Boyd, S. and Vandenberghe, L. (2004).
\newblock {\em Convex Optimization}.
\newblock Cambridge University Press.

\bibitem[Bubeck et~al., 2009]{bubeck2009pure}
Bubeck, S., Munos, R., and Stoltz, G. (2009).
\newblock Pure exploration in multi-armed bandits problems.
\newblock In {\em Algorithmic Learning Theory: 20th International Conference,
  ALT 2009, Porto, Portugal, October 3-5, 2009. Proceedings 20}, pages 23--37.
  Springer.

\bibitem[Bubeck et~al., 2010]{bubeck2010pure}
Bubeck, S., Munos, R., and Stoltz, G. (2010).
\newblock Pure exploration for multi-armed bandit problems.

\bibitem[Camilleri et~al., 2022a]{Camilleri22}
Camilleri, R., Wagenmaker, A., Morgenstern, J.~H., Jain, L., and Jamieson,
  K.~G. (2022a).
\newblock Active learning with safety constraints.
\newblock In Koyejo, S., Mohamed, S., Agarwal, A., Belgrave, D., Cho, K., and
  Oh, A., editors, {\em Advances in Neural Information Processing Systems},
  volume~35, pages 33201--33214. Curran Associates, Inc.

\bibitem[Camilleri et~al., 2022b]{camilleri2022active}
Camilleri, R., Wagenmaker, A., Morgenstern, J.~H., Jain, L., and Jamieson,
  K.~G. (2022b).
\newblock Active learning with safety constraints.
\newblock {\em Advances in Neural Information Processing Systems},
  35:33201--33214.

\bibitem[Carlsson et~al., 2024]{carlsson2024pure}
Carlsson, E., Basu, D., Johansson, F., and Dubhashi, D. (2024).
\newblock Pure exploration in bandits with linear constraints.
\newblock In {\em International Conference on Artificial Intelligence and
  Statistics}, pages 334--342. PMLR.

\bibitem[Chen et~al., 2014]{NIPS2014_e56954b4}
Chen, S., Lin, T., King, I., Lyu, M.~R., and Chen, W. (2014).
\newblock Combinatorial pure exploration of multi-armed bandits.
\newblock In Ghahramani, Z., Welling, M., Cortes, C., Lawrence, N., and
  Weinberger, K., editors, {\em Advances in Neural Information Processing
  Systems}, volume~27. Curran Associates, Inc.

\bibitem[Chen et~al., 2022a]{chen2022doubly}
Chen, T., Gangrade, A., and Saligrama, V. (2022a).
\newblock Doubly-optimistic play for safe linear bandits.
\newblock {\em arXiv preprint arXiv:2209.13694}.

\bibitem[Chen et~al., 2022b]{chen2022strategies}
Chen, T., Gangrade, A., and Saligrama, V. (2022b).
\newblock Strategies for safe multi-armed bandits with logarithmic regret and
  risk.
\newblock In {\em International Conference on Machine Learning}, pages
  3123--3148. PMLR.

\bibitem[Cheshire et~al., 2021]{cheshire2021influence}
Cheshire, J., Menard, P., and Carpentier, A. (2021).
\newblock The influence of shape constraints on the thresholding bandit
  problem.

\bibitem[Degenne and Koolen, 2019]{Degenne2019PureEW}
Degenne, R. and Koolen, W.~M. (2019).
\newblock Pure exploration with multiple correct answers.
\newblock In {\em Neural Information Processing Systems}.

\bibitem[Degenne et~al., 2019a]{NEURIPS2019_8d1de745}
Degenne, R., Koolen, W.~M., and M\'{e}nard, P. (2019a).
\newblock Non-asymptotic pure exploration by solving games.
\newblock In {\em Advances in Neural Information Processing Systems},
  volume~32. Curran Associates, Inc.

\bibitem[Degenne et~al., 2019b]{degenne2019nonasymptotic}
Degenne, R., Koolen, W.~M., and Ménard, P. (2019b).
\newblock Non-asymptotic pure exploration by solving games.

\bibitem[Even-Dar et~al., 2002a]{EvenDar2002PACBF}
Even-Dar, E., Mannor, S., and Mansour, Y. (2002a).
\newblock Pac bounds for multi-armed bandit and markov decision processes.
\newblock In {\em Annual Conference Computational Learning Theory}.

\bibitem[Even-Dar et~al., 2002b]{even2002pac}
Even-Dar, E., Mannor, S., and Mansour, Y. (2002b).
\newblock Pac bounds for multi-armed bandit and markov decision processes.
\newblock In {\em Computational Learning Theory: 15th Annual Conference on
  Computational Learning Theory, COLT 2002 Sydney, Australia, July 8--10, 2002
  Proceedings 15}, pages 255--270. Springer.

\bibitem[Faizal and Nair, 2022]{faizal2022constrained}
Faizal, F.~Z. and Nair, J. (2022).
\newblock Constrained pure exploration multi-armed bandits with a fixed budget.
\newblock {\em arXiv preprint arXiv:2211.14768}.

\bibitem[Fiez et~al., 2019]{fiez2019sequential}
Fiez, T., Jain, L., Jamieson, K.~G., and Ratliff, L. (2019).
\newblock Sequential experimental design for transductive linear bandits.
\newblock {\em Advances in neural information processing systems}, 32.

\bibitem[Fonseca and Fleming, 1998]{fonseca1998multiobjective}
Fonseca, C.~M. and Fleming, P.~J. (1998).
\newblock Multiobjective optimization and multiple constraint handling with
  evolutionary algorithms. i. a unified formulation.
\newblock {\em IEEE Transactions on Systems, Man, and Cybernetics-Part A:
  Systems and Humans}, 28(1):26--37.

\bibitem[Gangrade et~al., 2024]{testingfeasibilitylinearprograms}
Gangrade, A., Gopalan, A., Saligrama, V., and Scott, C. (2024).
\newblock Testing the feasibility of linear programs with bandit feedback.

\bibitem[Garivier and Kaufmann, 2016]{pmlr-v49-garivier16a}
Garivier, A. and Kaufmann, E. (2016).
\newblock Optimal best arm identification with fixed confidence.
\newblock In Feldman, V., Rakhlin, A., and Shamir, O., editors, {\em 29th
  Annual Conference on Learning Theory}, volume~49 of {\em Proceedings of
  Machine Learning Research}, pages 998--1027, Columbia University, New York,
  New York, USA. PMLR.

\bibitem[Garivier and Kaufmann, 2021]{garivier2021nonasymptotic}
Garivier, A. and Kaufmann, E. (2021).
\newblock Non-asymptotic sequential tests for overlapping hypotheses and
  application to near optimal arm identification in bandit models.

\bibitem[Garivier et~al., 2018]{garivier2018thresholding}
Garivier, A., Ménard, P., Rossi, L., and Menard, P. (2018).
\newblock Thresholding bandit for dose-ranging: The impact of monotonicity.

\bibitem[Hutchinson et~al., 2024]{hutchinson2024directional}
Hutchinson, S., Turan, B., and Alizadeh, M. (2024).
\newblock Directional optimism for safe linear bandits.

\bibitem[Immorlica et~al., 2022]{10.1145/3557045}
Immorlica, N., Sankararaman, K., Schapire, R., and Slivkins, A. (2022).
\newblock Adversarial bandits with knapsacks.
\newblock {\em J. ACM}, 69(6).

\bibitem[Jamieson and Nowak, 2014]{jamieson2014best}
Jamieson, K. and Nowak, R. (2014).
\newblock Best-arm identification algorithms for multi-armed bandits in the
  fixed confidence setting.
\newblock In {\em 2014 48th Annual Conference on Information Sciences and
  Systems (CISS)}, pages 1--6. IEEE.

\bibitem[Jourdan and Degenne, 2022]{jourdan2022choosing}
Jourdan, M. and Degenne, R. (2022).
\newblock Choosing answers in epsilon-best-answer identification for linear
  bandits.
\newblock In {\em International Conference on Machine Learning}, pages
  10384--10430. PMLR.

\bibitem[Jourdan et~al., 2023]{jourdan2023varepsilon}
Jourdan, M., Degenne, R., and Kaufmann, E. (2023).
\newblock An {$\varepsilon$}-best-arm identification algorithm for
  fixed-confidence and beyond.
\newblock {\em Advances in Neural Information Processing Systems},
  36:16578--16649.

\bibitem[Jourdan et~al., 2021]{jourdan2021efficient}
Jourdan, M., Mutn{\'y}, M., Kirschner, J., and Krause, A. (2021).
\newblock Efficient pure exploration for combinatorial bandits with semi-bandit
  feedback.
\newblock In {\em Proceedings of the 32nd International Conference on
  Algorithmic Learning Theory}, volume 132.

\bibitem[Katz-Samuels and Scott, 2018]{katz2018feasible}
Katz-Samuels, J. and Scott, C. (2018).
\newblock Feasible arm identification.
\newblock In {\em International Conference on Machine Learning}, pages
  2535--2543. PMLR.

\bibitem[Kaufmann et~al., 2016]{kaufmann2016complexity}
Kaufmann, E., Cappé, O., and Garivier, A. (2016).
\newblock On the complexity of best arm identification in multi-armed bandit
  models.

\bibitem[Kaufmann and Koolen, 2021]{JMLR:v22:18-798}
Kaufmann, E. and Koolen, W.~M. (2021).
\newblock Mixture martingales revisited with applications to sequential tests
  and confidence intervals.
\newblock {\em Journal of Machine Learning Research}, 22(246):1--44.

\bibitem[Kazerouni et~al., 2017]{kazerouni2017conservative}
Kazerouni, A., Ghavamzadeh, M., Abbasi-Yadkori, Y., and Roy, B.~V. (2017).
\newblock Conservative contextual linear bandits.

\bibitem[Lattimore and Szepesvári, 2020]{lattimore_szepesvari_2020}
Lattimore, T. and Szepesvári, C. (2020).
\newblock {\em Bandit Algorithms}.
\newblock Cambridge University Press.

\bibitem[Li et~al., 2017]{li2017hyperband}
Li, L., Jamieson, K., DeSalvo, G., Rostamizadeh, A., and Talwalkar, A. (2017).
\newblock Hyperband: A novel bandit-based approach to hyperparameter
  optimization.
\newblock {\em The Journal of Machine Learning Research}, 18(1):6765--6816.

\bibitem[Li et~al., 2023]{li2023optimal}
Li, S., Zhang, L., Yu, Y., and Li, X. (2023).
\newblock Optimal arms identification with knapsacks.
\newblock In {\em International Conference on Machine Learning}, pages
  20529--20555. PMLR.

\bibitem[Li et~al., 2021]{pmlr-v139-li21s}
Li, X., Sun, C., and Ye, Y. (2021).
\newblock The symmetry between arms and knapsacks: A primal-dual approach for
  bandits with knapsacks.
\newblock In Meila, M. and Zhang, T., editors, {\em Proceedings of the 38th
  International Conference on Machine Learning}, volume 139 of {\em Proceedings
  of Machine Learning Research}, pages 6483--6492. PMLR.

\bibitem[Libin et~al., 2019]{libin2019bayesian}
Libin, P.~J., Verstraeten, T., Roijers, D.~M., Grujic, J., Theys, K., Lemey,
  P., and Now{\'e}, A. (2019).
\newblock Bayesian best-arm identification for selecting influenza mitigation
  strategies.
\newblock In {\em Machine Learning and Knowledge Discovery in Databases:
  European Conference, ECML PKDD 2018, Dublin, Ireland, September 10--14, 2018,
  Proceedings, Part III 18}, pages 456--471. Springer.

\bibitem[Lindner et~al., 2022]{lindner2022interactively}
Lindner, D., Tschiatschek, S., Hofmann, K., and Krause, A. (2022).
\newblock Interactively learning preference constraints in linear bandits.
\newblock In {\em International Conference on Machine Learning}, pages
  13505--13527. PMLR.

\bibitem[Lindst{\aa}hl et~al., 2022]{lindstaahl2022measurement}
Lindst{\aa}hl, S., Proutiere, A., and Johnsson, A. (2022).
\newblock Measurement-based admission control in sliced networks: A best arm
  identification approach.
\newblock In {\em GLOBECOM 2022-2022 IEEE Global Communications Conference},
  pages 1484--1490. IEEE.

\bibitem[Liu et~al., 2021]{liu2021efficient}
Liu, X., Li, B., Shi, P., and Ying, L. (2021).
\newblock An efficient pessimistic-optimistic algorithm for stochastic linear
  bandits with general constraints.

\bibitem[Ma, 2014]{doi:10.1137/1.9781611973402.85}
Ma, W. (2014).
\newblock {\em Improvements and Generalizations of Stochastic Knapsack and
  Multi-Armed Bandit Approximation Algorithms: Extended Abstract}, pages
  1154--1163.

\bibitem[Maas et~al., 2011]{maas-etal-2011-learning}
Maas, A.~L., Daly, R.~E., Pham, P.~T., Huang, D., Ng, A.~Y., and Potts, C.
  (2011).
\newblock Learning word vectors for sentiment analysis.
\newblock In Lin, D., Matsumoto, Y., and Mihalcea, R., editors, {\em
  Proceedings of the 49th Annual Meeting of the Association for Computational
  Linguistics: Human Language Technologies}, pages 142--150, Portland, Oregon,
  USA. Association for Computational Linguistics.

\bibitem[Magureanu et~al., 2014]{magureanu2014lipschitz}
Magureanu, S., Combes, R., and Proutiere, A. (2014).
\newblock Lipschitz bandits: Regret lower bounds and optimal algorithms.

\bibitem[Mason et~al., 2020]{mason2020finding}
Mason, B., Jain, L., Tripathy, A., and Nowak, R. (2020).
\newblock Finding all {$\varepsilon$}-good arms in stochastic bandits.
\newblock {\em Advances in Neural Information Processing Systems},
  33:20707--20718.

\bibitem[Moradipari et~al., 2020]{moradipari2020safe}
Moradipari, A., Amani, S., Alizadeh, M., and Thrampoulidis, C. (2020).
\newblock Safe linear thompson sampling with side information.

\bibitem[Pacchiano et~al., 2024]{pacchiano2024contextual}
Pacchiano, A., Ghavamzadeh, M., and Bartlett, P. (2024).
\newblock Contextual bandits with stage-wise constraints.
\newblock {\em arXiv preprint arXiv:2401.08016}.

\bibitem[Pacchiano et~al., 2020]{pacchiano2020stochastic}
Pacchiano, A., Ghavamzadeh, M., Bartlett, P., and Jiang, H. (2020).
\newblock Stochastic bandits with linear constraints.

\bibitem[Paulson, 1964]{749e376d-b869-3631-9c77-11e228873a68}
Paulson, E. (1964).
\newblock A sequential procedure for selecting the population with the largest
  mean from k normal populations.
\newblock {\em The Annals of Mathematical Statistics}, 35(1):174--180.

\bibitem[Sankararaman and Slivkins, 2018]{pmlr-v84-sankararaman18a}
Sankararaman, K.~A. and Slivkins, A. (2018).
\newblock Combinatorial semi-bandits with knapsacks.
\newblock In Storkey, A. and Perez-Cruz, F., editors, {\em Proceedings of the
  Twenty-First International Conference on Artificial Intelligence and
  Statistics}, volume~84 of {\em Proceedings of Machine Learning Research},
  pages 1760--1770. PMLR.

\bibitem[Shang et~al., 2023]{shang2023price}
Shang, X., Colin, I., Barlier, M., and Cherkaoui, H. (2023).
\newblock Price of safety in linear best arm identification.

\bibitem[Singh and Joachims, 2019]{NEURIPS2019_9e82757e}
Singh, A. and Joachims, T. (2019).
\newblock Policy learning for fairness in ranking.
\newblock In Wallach, H., Larochelle, H., Beygelzimer, A., d\textquotesingle
  Alch\'{e}-Buc, F., Fox, E., and Garnett, R., editors, {\em Advances in Neural
  Information Processing Systems}, volume~32. Curran Associates, Inc.

\bibitem[Slivkins et~al., 2023]{pmlr-v195-foster23c}
Slivkins, A., Sankararaman, K.~A., and Foster, D.~J. (2023).
\newblock Contextual bandits with packing and covering constraints: A modular
  lagrangian approach via regression.
\newblock In Neu, G. and Rosasco, L., editors, {\em Proceedings of Thirty Sixth
  Conference on Learning Theory}, volume 195 of {\em Proceedings of Machine
  Learning Research}, pages 4633--4656. PMLR.

\bibitem[Slutsky, 1925]{zbMATH02592216}
Slutsky, E. (1925).
\newblock {\"U}ber stochastische {Asymptoten} und {Grenzwerte}.
\newblock Metron 5, {Nr}. 3, 3-89 (1925).

\bibitem[Tirinzoni et~al., 2020]{NEURIPS2020_0f34314d}
Tirinzoni, A., Pirotta, M., Restelli, M., and Lazaric, A. (2020).
\newblock An asymptotically optimal primal-dual incremental algorithm for
  contextual linear bandits.
\newblock In Larochelle, H., Ranzato, M., Hadsell, R., Balcan, M., and Lin, H.,
  editors, {\em Advances in Neural Information Processing Systems}, volume~33,
  pages 1417--1427. Curran Associates, Inc.

\bibitem[Tran-Thanh et~al., 2012]{soton337280}
Tran-Thanh, L., Chapman, A., Rogers, A., and Jennings, N.~R. (2012).
\newblock Knapsack based optimal policies for budget-limited multi-armed
  bandits.
\newblock In {\em Twenty-Sixth AAAI Conference on Artificial Intelligence
  (AAAI-12) (22/07/12 - 22/07/12)}, pages 1134--1140.

\bibitem[Wang et~al., 2021a]{pmlr-v139-wang21b}
Wang, L., Bai, Y., Sun, W., and Joachims, T. (2021a).
\newblock Fairness of exposure in stochastic bandits.
\newblock In Meila, M. and Zhang, T., editors, {\em Proceedings of the 38th
  International Conference on Machine Learning}, volume 139 of {\em Proceedings
  of Machine Learning Research}, pages 10686--10696. PMLR.

\bibitem[Wang et~al., 2021b]{wang2021best}
Wang, Z., Wagenmaker, A., and Jamieson, K. (2021b).
\newblock Best arm identification with safety constraints.

\bibitem[Wu et~al., 2023]{wu2023best}
Wu, Y., Zheng, Z., and Zhu, T. (2023).
\newblock Best arm identification with fairness constraints on subpopulations.
\newblock In {\em 2023 Winter Simulation Conference (WSC)}, pages 540--551.
  IEEE.

\end{thebibliography}

\clearpage
\appendix
\onecolumn

%\newpage
% \section{Appendix}
% \todo[inline]{Index for the Appendix section}
\section{Notations}\label{Notations}

\begin{table}[h!]
\centering
\resizebox{0.9\textwidth}{!}{
\begin{tabular}{c|c}
\toprule \textbf{Notation} & \textbf{Definition} \\
\midrule$\simp$ & K-simplex \\
 $\numarm$ & Number of Arms \\
 $A$ & True constraint set \\
 $d$ & Number of constraints\\
 $\mathcal{F}$ & True feasible set w.r.t $A$, $\mathcal{F} = \{ A\in\mathbb{R}^{d\times K}:A\pol \leq 0  \}$ \\
 $\Tilde{A}_t$ & Optimistic estimate of constraint set at time $t$, $\Tilde{A}_t = \Hat{A} - f(t,\delta)||\bomega_t||_{\Sigma_t^{-1}}$\\
 $\Fest_t$ & Estimated feasible set w.r.t pessimistic estimate $\Tilde{A}$ at time $t$, $\mathcal{F} = \{ \Tilde{A}_t\in\mathbb{R}^{d\times K}:\Tilde{A}_t\pol \leq 0  \}$ \\
 $\mathcal{A}$ & The action set of $K$ possible choices\\
 $\bomega_t$ & Allocation chosen at time $t$ \\
 $a_t$ & Action at time $t$ among $K$ possible actions \\
 $N$ & Number of Constraints \\
 $\Gamma$ & Slack of the optimisation problem, i.e., $\Gamma \defn \min_{i \in [d]}(-A\opt_{\cF})$\\
 $\sigma^2$ & Variance of the reward distribution (Gaussian) of arms \\
 $r_t , \cost_t$ & Reward and cost observed at time $t$ \\
 $\delta$ & Chosen confidence level \\
 $\lag_t$ & The Lagrangian multiplier at time $t$ \\
 $\Sigma_{t}$ & The covariance matrix (Gram matrix) at round t \\
 $\Lambda_{\cF}(\bmu)$ & Set of alternative (confusing) instances for bandit instance $\bmu$ \\
 $\altset$ & Estimated set of alternative (confusing) instances for bandit instance $\bmu$ \\
 $\nu(\opt)$ & Neighboring set of optimal r-good feasible policy $\opt$ \\
 $\nu(\rpolest)$ & Neighboring set of estimated optimal r-good feasible policy $\rpolest$ \\
 $\tau_{\delta}$ & Stopping time of $(1-\delta)$-correct algorithm\\
 $\opt_{\mathcal{F}}$ & True optimal policy w.r.t actual constraint set $A$ \\
 $\opt_{\Fest}$ & Optimal Policy for the estimated feasible set \\
 $\opt$ & Optimal r-good feasible policy with respect to $\opt_{\cF}$\\
 $\rpolest$ & Optimal r-good feasible policy with respect to $\opt_{\Fest}$\\
 $\SP$ & Shadow price $\defn \frac{\Gamma_{\mathrm{max}}}{\Gamma_{\mathrm{min}}}$\\
\bottomrule
\end{tabular}}
\caption{Summary of Notations}
\label{tab:notations}
\end{table}
\newpage

\section{Discussion on The Problem Setting and Connection to Literature}
\subsection{Extended Related Work}\label{sec:Ex_related_work}

\textbf{Historical pioneering works.} Literature on bandits has come a long way since the problem of optimal sequential sampling started with the works of \cite{08293709-a226-3cbe-aa4c-5f6306df76a0} and \cite{749e376d-b869-3631-9c77-11e228873a68} with the assumption of the populations being normally distributed. To talk about pure exploration setting, \cite{even2002pac}, \cite{bubeck2010pure} should be mentioned as the first ones who worked in this specific setting for stochastic bandits.

\vspace{2mm}\textbf{Existing work on adapting known constraints.} In Multi-armed bandit literature, people often introduce constraints as a notion of safety where they impose known constraints on the chosen arm or on the exploration process. \cite{wang2021best} considers pure strategy (only one co-ordinate as chosen action) and imposes a safety threshold on the linear cost feedback of the chosen arm. On the other hand, the setting considered in \cite{carlsson2024pure} is closer as it tracks an optimal policy w.r.t to a known set of known constraints. On the other hand, \cite{liu2021efficient} (Improvement over \cite{pacchiano2020stochastic} in MAB setting) generalized the known constraint regret minimization setting by assuming existence of a set of general constraints. Our work captures the hardness of not knowing the constraint set while tracking the lower bound and also in sample complexity upper bounds of Algorithm \ref{alg:lagts} and \ref{alg:lagex}. Our work also introduce \textit{shadow price} as a novel term in pure exploration literature which characterises the extra cost that arises due to tracking the unknown constraints.

\vspace{2mm}\textbf{Learning unknown constraints.}  \cite{lindner2022interactively} considers constrained linear best-arm identification arm
are vectors with known rewards and a single unknown
constraint (representing preferences) on the actions. Works on adapting to unknown constraints is discussed in the related work section of the main paper.

\vspace{2mm}\textbf{Transductive Linear Bandit.} In this setting formalised by~\cite{fiez2019sequential},  \cite{camilleri2022active} studies this setting with unknown linear constraints where we have to find the best safe arm in a finite set $\mathcal{Z}$ different than actual arm set $\mathcal{A}$. Our setting generalises the setting in the sense that the finite feasible set $\mathcal{Z}$ is not static, rather we track $\mathcal{Z}_t$ per time step $t\in\mathbb{N}$ and explore within that set to find the optimal allocation. At the end of exploration after hitting the stopping criterion at $\tau_{\delta}$ the agent recommends the optimal policy inside the set $\mathcal{Z}_{\tau_{\delta}}$.

\vspace{2mm}\textbf{Regret Minimization with Unknown constraints.} In bandit literature, constraints are often introduced in the setting to study regret minimization. \cite{moradipari2020safe} studies regret minimization using Linear Thompson Sampling (LTS) imposing known safety constraints on the chosen action. \cite{amani2019linear} 
studies contextual bandits under unknown and unobserved linear constraint, whereas \cite{kazerouni2017conservative} \cite{pacchiano2020stochastic} studies UCB based algorithms for regret minimization for linear bandits which assumes existence of a safe action space in case of unknown anytime linear constraint. In line with these, recent works \cite{hutchinson2024directional} \cite{pacchiano2024contextual} \cite{shang2023price} improved on regret guarantees and the first one relaxed the assumption of existence of a pessimistic safe space. \cite{chen2022doubly} introduces doubly-optimistic setting to study safe linear bandit(SLB). \cite{liu2021efficient} generalised the setting of \cite{pacchiano2020stochastic} not only relaxing the condition of having a safe action but also considered a set of general constraints and also captured the notion of both anytime and end-of-time constraints which we also see in \cite{carlsson2024pure}. \cite{liu2021efficient} also shows the trade-off between maximising reward or minimizing regret and constraint violation using Lyapunov drift. In this work we do not focus on regret guaranties but finding the optimal policy with sample complexity as least as possible while tracking and satisfying a set of unknown linear constraints.

\vspace{2mm}\textbf{BAI with Fairness Constraint.} Considering fairness constraint in our setting can be an interesting application to our setting. Recently \cite{wu2023best} studied Best Arm Identification with fairness Constraints on Subpopulations (BAICS), where they have discussed the trade-off in the standard BAI complexity if there are finite number of subpopulations are given and the best chosen arm must perform well (not too bad) on all those subpopulations. Another important line of work \cite{pmlr-v139-wang21b}  \cite{NEURIPS2019_9e82757e} explores regret analysis of BAI with positive merit-based exposure of fairness constraints where the chosen policy has to satisfy some fairness constraint across all its indices. Our setting comes as a direct application to these settings. Further discussion in Section \ref{subsec:motivation}. 

\vspace{2mm}\textbf{BAI with Knapsack constraint.} While the existing literature on bandit with knapsack \cite{10.1145/3164539} \cite{NIPS2016_f3144cef} \cite{10.1145/3557045}, \cite{10.1145/2600057.2602844} \cite{pmlr-v49-agrawal16} \cite{pmlr-v84-sankararaman18a} \cite{doi:10.1137/1.9781611973402.85} focused on mainly regret minimization, our setting aligns more as a special case of the Optimal Arm identification with Knapsack setting in \cite{li2023optimal,soton337280,pmlr-v139-li21s}. Though we aim to find the best policy rather than a specific arm in the constraint space. Our setting should be considered as a special case of these settings. Further discussion in Section \ref{subsec:motivation}. 

\vspace{2mm}\textbf{Algorithms on Pure exploration.}
Algorithm \ref{alg:lagts} is an extension of the Track and Stop(TnS) strategy from \citep{pmlr-v49-garivier16a}, while the motivation for Algorithm \ref{alg:lagex} comes from the Gamified Explorer strategy from \citep{degenne2019nonasymptotic} where the lower bound is treated as a zero-sum game between the allocation and the instance player. We refer to \citep{garivier2021nonasymptotic},\citep{pmlr-v49-garivier16a},\citep{kaufmann2016complexity},\citep{Degenne2019PureEW}, \cite{Boyd_Vandenberghe_2004},\cite{jourdan2021efficient} etc for important concentration inequalities, tracking lemmas.

\vspace{2mm}\textbf{Dose-finding and Thresholding Bandits.} Another special case of our setting is Dose-finding or Thresholding bandits in structured MAB literature \cite{NIPS2014_e56954b4} generalized the problem, then a line of work \cite{JMLR:v22:19-228} \cite{garivier2018thresholding} \cite{cheshire2021influence} aims to find the maximum safe dose for a specific drug in early stages of clinical trials. In some sense our setting generalizes this setting. If we have to administer more than drugs to a patient, our setting generalises to track the best possible proportion in which the drugs should be administered with maximum efficacy. Further discussions in Section \ref{subsec:motivation}.

\subsection{Motivations: Reductions to and Generalisations of The Existing BAI Settings}\label{subsec:motivation}%\vspace*{-.5em}

Before delving into the details of the lower bounds and algorithms, we first clarify our motivation by showing how different setups studied in literature and their variations are special case of our setting.

\vspace{2mm}\textbf{Thresholding Bandits.} Our setting encompasses the thresholding bandit problem~\citep{aziz2021multi}. Thresholding bandit is motivated from the safe dose finding problem in clinical trials, where one wants to identify the highest dose of a drug that is below a known safety level. This has also motivated the studies on safe arm identification~\citep{wang2021best}. Our setting generalises it further to detect the dose of the drug with highest efficacy while it is still below the safety level. We can formulate it as identifying $\opt = \argmax \bmu^{\top}\pol$, such that $\identity\pol \leq \identity \btheta$. 
Rather, generalising the classical thresholding bandits, our formulation can further model the safe doses for the optimal cocktail of drugs, and $\theta$ can have different values across drugs, i.e we can consider different thresholds for different drugs.

\vspace{2mm}\textbf{Optimal policy  under Knapsack.} Bandits under knapsack constraints have been studied both in best-arm identification~\citep{li2023optimal,soton337280,pmlr-v139-li21s} and regret minimisation~\citep{10.1145/3164539, NIPS2016_f3144cef,10.1145/3557045, 10.1145/2600057.2602844,pmlr-v49-agrawal16, pmlr-v84-sankararaman18a,doi:10.1137/1.9781611973402.85} literature. BAI under knapsacks is motivated by the fact that detecting an optimal arm might have additional resource constraints in addition to the number of required samples. This has led to study of BAI with knapsacks only under fixed-budget settings~\citep{li2023optimal}. But as in regret-minimisation literature~\citep{pmlr-v84-sankararaman18a,doi:10.1137/1.9781611973402.85}, one might want to recommend a policy that maximises utility while satisfying knapsack constraints. For example, we want to manage caches where the recommended memory allocation should satisfy a certain resource budget. Thus, the recommended policy has to satisfy $\opt_{\tau} = \argmax_{\pol \in C_{A}} \Hat{\bmu}_{\stopping}^{\top} \pol$, where $C_{A} \defn \{ A\pol_{\stopping} \leq c \}$. Naturally, this is a special case of our problem setting.

\vspace{2mm}\textbf{Feasible arm selection.}
We look at the pure exploration problem of feasible arm selection studied by \cite{katz2018feasible}. Here, we think of a problem of workers having a multi-dimension vector representation where each index denotes the accuracy of that worker being able to identify a specific class label in a classification task in hand. The problem turns to be a feasible arm selection from a simple BAI problem when we impose a feasibility constraint that for example, the chosen worker should show more than 90\% accuracy across all labels. We can generalise this setting in the sense that we are now not looking for a specific worker, rather we want to make a team of workers that has the highest utility. The recommended policy at time $t\in \mathbb{N}$, $\max_{\pol\in\Delta_{K-1}} \bmu_t^{\top} \pol$ such that $f^{\top} \pol \geq \tau$ where $\tau$ is the desired threshold level. The generalisation of the setting pitch in as thresholds of $\tau$ can have different values corresponding to different workers.

\vspace{2mm}\textbf{BAI with fairness across sub-populations.} The Best Arm Identification with fairness Constraints on Sub-population (BAICS) studied in \cite{wu2023best} aims on selecting an arm that must be fair across all sub-populations rather than the whole population in standard BAI setting. Let, there are $\lag$ sub-populations and $\mu_a$ are the means corresponding to the a-th arm. Finding only the optimal arm $K_{\mathrm{BAI}} = \argmax_{k\in [\numarm]} \bmu_k$ may not be enough because it may not perform equally good for all the $l$ sub-populations. Then the arm should belong to a set $C:= \{k\in[\numarm]|\bmu_{k,m}\geq 0 , m\in[l]\}$ where the observation for arm k and population m comes from $\mathcal{N}(\mu_{k,m},1)$ It ensures that the chosen arm does not perform \textit{too bad} for any sub-population. Let us think of a problem where there are $l$ sub-groups of patients and we have $\numarm$ number of drugs to administer with reward means $\mu_{k}, k\in[\numarm]$. We are looking for a combination of drugs rather than a single drug to administer as $\opt = \argmax_{\pi\in\simp} \bmu^{\top}\pi$ such that $\indicator_{\bmu_m\geq0}^{\top} \pol =1 , \forall m\in[l]$. Thus, BAICS is a special case of ours.

\vspace{2mm}\textbf{Fairness of exposure in bandits.} \cite{pmlr-v139-wang21b} introduced positive merit based exposure of fairness constraints~\citep{NEURIPS2019_9e82757e} in stochastic bandits standing against the winner-takes-all allocation strategy that are historically studied. The chosen allocation in this setting should satisfy the fairness constraint $\frac{\opt_a}{f(\mu^{\star}_a)}=\frac{\opt_{a'}}{f(\mu^{\star}_{a'})}, \forall a'\in[\numarm]$ where $f(.)$ transform reward of an arm to a positive merit. Though \cite{pmlr-v139-wang21b} studied this setting in regret analysis, this setting in BAI setting is a direct application of our setting as we are looking for an optimal policy $\opt = \argmax \bmu^{\top}\pol$ such that $\opt$ satisfies $A_{f(\bmu)}^{\top} \pol = 0 $
where $A_{f(\bmu)}$ is of order $\frac{K(K-1)}{2} \times K$ and $A_{f(\bmu)}$ is expressed as,
\begin{align*}
    (A_{f(\bmu)})_{ij} = 
     \begin{cases}
       \frac{1}{f(\mu_a)} &\quad\text{if } aK- \frac{1}{2}(a-1)(a-2)\leq j\leq aK-\frac{1}{2}a(a-1) \text{ and 
        } i=a \, ,\\
       \frac{-1}{f(\mu_{j})} &\quad\text{if }  aK- \frac{1}{2}(a-1)(a-2)\leq j\leq aK-\frac{1}{2}a(a-1) \text{ and 
        } a < i \leq K\, ,\\
       0 &\quad\text{otherwise.}\\ 
     \end{cases}
\end{align*}
For example, when $K=3$, $A_{f(\bmu)} = [[\frac{1}{\mu_1}, -\frac{1}{\mu_2}, 0], [0, \frac{1}{\mu_2}, -\frac{1}{\mu_3}], [\frac{1}{\mu_1}, 0, -\frac{1}{\mu_3}]]$.

% \subsection{Estimation of Parameters}\label{Estimation}
\newpage\section{Strong Duality and the Lagrangian Multiplier: Proof of Theorem \ref{thm:Existence of strong duality and bound on Lagrangian multiplier}}\label{sec:Strong duality proof}
\begin{reptheorem}{thm:Existence of strong duality and bound on Lagrangian multiplier}
    For a bounded sequence of $\{ l_t\}_{t\in\mathbb{N}}$, strong-duality holds for the optimisation problem stated in Equation~\eqref{eqn:lag_relax} i.e.
    \begin{align*}
		&\inf_{\lag\in\reals_{+}^{\numconst}} \min_{A' \in \cC} \sup_{\allocation\in\simp}\max_{\pol\in\rgoodest}\inf_{\blambda\in\altsetr} \allocation^{\top}\klvec - \lag^{\top} {A'}\allocation \notag \\
		=& \sup_{\allocation\in\simp} \min_{\lag\in\cL}\max_{\pol\in\rgoodest}\inf_{\blambda\in\altsetr} \allocation^{\top}\klvec - \lag^{\top} \Tilde{A}\allocation\,.
    \end{align*}
    Here, $\cL \defn \{\lag \in \reals^{\numconst}\mid 0\leq \Vert {\lag}\Vert_1 \leq \frac{1}{\gamma} \mathcal{D}(\allocation,\bmu,\cF)\}$,
    where $\gamma \defn \min_{i\in[1,{\numconst}]} \{ -\Tilde{A}^i \hat{\opt}\}$, i.e. the minimum slack for optimistic constraints w.r.t. the estimated optimal feasible policy.
\end{reptheorem}

\begin{proof}
This proof involves three steps. In the first step we prove convexity and other properties of the sets involved in the main optimisation problem \ref{eq:lag_lb}. In the next step, we show that Slater's sufficient conditions hold for $\bpi$ as a consequence of these properties. Once we prove the unique optimality of $\bpi$ we state bounds on the L1-norm of the Lagrangian multiplier. We conclude by establishing strong duality and proving the statement of the theorem.

\vspace{2mm}\textbf{Step 1: Properties of perturbed feasible set and alt-set.}
Let us first check the properties of $\Fest$, $\altset$ and $\neighr$. For that, let us remind the definitions of these sets. The estimated feasible set is defined as $\Fest \defn \left\{ \bpi \in \simp : \Tilde{A}\bpi \leq 0 \right\}$. The set of alternative (confusing) instances for the optimal policy $\Hat{\pol}^{\star} = \argmax_{\pol\in\rgoodest} \bmu^{\top}\pol$ is $\altset\defn \left\{ \blambda \in \mathbb{D}:r + \max_{\pol \in \rgoodest} \blambda^{\top}\bpi>\blambda^{\top}\Hat{\pol}^{\star}\right\}$. For $\pol'$ being a neighbour of $\Hat{\pol}^{\star}$ or in other words, an extreme point in $\Fest$, we decompose the alternative set as the union of half-spaces as,
\begin{align*}
    &  \altset = \bigcup_{\bpi^{'}\in\nu_{\Fest}(\Hat{\pol}^{\star})}\left\{ \blambda : \blambda^{\top}(\Hat{\pol}^{\star} - \bpi' )<r \right\} 
\end{align*}
We should note, $\pol'$ shares at least $(\numarm-1)$ active constraints with $\Hat{\pol}^{\star}$. It is clear that $\Hat{\mathcal{F}}$ is bounded and convex in $\bpi$.
Since, convex combination of any two extreme point $\bpi'_1, \bpi'_2$ in the neighborhood of the optimal policy $\nu_{\Fest}(\Hat{\pol}^{\star})$ also shares $\numarm-1$ active constraints with $\Hat{\pol}^{\star}$, so $\nu_{\Fest}(\Hat{\pol}^{\star})$ is convex in $\pol'$. 

Let, $\bpi'_1$ and $\bpi'_2$ are two policies in the neighborhood of $\Hat{\pol}^{\star}$, such that for any alternative instance $\blambda$, $\blambda^{\top}(\bpi'_1 - \bpi'_2) \geq 0$, which implies that the policy $\bpi'_1$ is closer to the optimal policy in the neighborhood that the policy $\bpi'_2$\\
Therefore, any convex combination of these neighbourhood policy, $\blambda^{\top} \left( \Hat{\pol}^{\star} - (a\bpi'_1+(1-a)\bpi'_2) \right) = \blambda^{\top}( \Hat{\pol}^{\star} - \bpi'_2) - a\blambda^{\top}(\bpi'_1-\bpi'_2) \leq c$. Therefore, the set $\altset$ is also bounded and convex in $\bpi$.\\
Also, since we are working with optimistic estimate of $A$, the set $\Fest$ will always be non-empty, because we will find at least one $\tildeA_0$ which is non-singular and it's inverse exists.

\vspace{2mm}\textbf{Step 2: Slater's condition.}
From step 1 of this proof we have the following properties
\begin{enumerate}
    \item $\Fest$ is non-empty, bounded and convex in $\bpi$. 
    \item The perturbed neighborhood $\neighr$ is convex for any $\bpi'\in \neighr$
    \item $\altset$ is also bounded and convex in $\bpi$.
\end{enumerate}
Leveraging these three results we claim that there exists a $\Hat{\pol}^{\star}$ that uniquely solves the optimisation problem in Equation \eqref{eq:lag_lb} and satisfy the constraints with strict inequality. Thus, we claim Slater's sufficient conditions hold for $\pol$.

\vspace{2mm}\textbf{Step 3: Bound on the Lagrangian multiplier.}
Here, we try to bound the L-1 norm of the Lagrangian multiplier. Since, $\|\lag\|_1$ cannot be less than 0, then we already have a lower bound.  

Now we refer to lemma \ref{lemm:Dantzig lemma} for the upper bound. An immediate implication of this result is that for any dual optimal solution $\lag^*$, we have
$\left\|\lag^*\right\|_1 \leq \frac{1}{\gamma}\left(f(\bar{x})-q^*\right)$. 
Since Slater's conditions hold in our case for $\bpi$, we can write that the  optimal solution of the Lagrangian dual,

\begin{align*}
    &0\leq ||{l^{\star}}||_1 \leq \frac{1}{\gamma}\mathcal{D}(\allocation^{\star},\Hat{\bmu},\cF)  \\
    &\text{where, } \gamma \defn \min_{i\in[1,{\numconst}]} \{ -\Tilde{A}^i \Hat{\pol}^{\star}\}
\end{align*}
Where, $\Hat{\pol}^{\star}$ is the $r$-optimal feasible policy. 
%We can replace the dual function with primal $\mathcal{D}(\omega_{t}^{\star},\Hat{\bmu}_t,\Fest_t,l_{t-1}^{\star})$ because it is always upper bounded by the dual function, so we don't have to explicitly calculate the dual function. Though it is not tractable anyway due to not knowing the pure exploration solution.

\vspace{2mm}\textbf{Step 4: Establishing strong duality.}
Therefore the domain of the Lagrangian multiplier is also bounded and convex. So again we say that $l_t^{\star}$ uniquely minimises Equation \ref{eq:lag_lb}. We define $\cL \defn \{\lag \in \reals^{\numconst}\mid 0\leq \Vert {\lag}\Vert_1 \leq \frac{1}{\gamma} T_{\Fest,r}^{-1}(\hat{\bmu})\}$, where $\gamma \defn \min_{i\in[1,{\numconst}]} \{ -\Tilde{A}^i \allocation^*\}$ Then according to \textbf{Heine-Borel's theorem} (Theorem \ref{Heine-Borel theorem}) we can say that these sets are compact as well. We can then conclude that Strong duality holds which means that it perfectly make sense of solving the Lagrangian dual formulation of the primal optimisation problem because there is no duality gap. We later on will consider this formulation as two player zero sum game. Due to strong duality we claim that the agent wile playing this game, Nash equilibrium will be eventually established. 

Now that everything is put into place we can conclude with the very statement of the theorem that due to strong duality the following holds 
    \begin{align*}
    	&\inf_{\lag\in\reals_{+}^{\numconst}} \min_{A' \in \cC} \sup_{\allocation\in\simp}\max_{\pol\in\rgoodest}\inf_{\blambda\in\altsetr} \allocation^{\top}\klvec - \lag^{\top} {A'}\allocation \notag \\
    	=& \sup_{\allocation\in\simp} \min_{\lag\in\cL}\max_{\pol\in\rgoodest}\inf_{\blambda\in\altsetr} \allocation^{\top}\klvec - \lag^{\top} \Tilde{A}\allocation\,.
    \end{align*}  
\end{proof}

\newpage
\section{Lagrangian Relaxation of Projection Lemma: Proof of Theorem \ref{thm:Asymptotic Convergence of the Projection Lemma}}\label{sec:prop of Fhat}
\begin{reptheorem}{thm:Asymptotic Convergence of the Projection Lemma}
    For a sequence $\{ \Fest_t \}_{t\in\mathbb{N}}$ and $\{\Hat{\blambda}_t \}_{t\in\mathbb{N}}$, we first show that
    (a) $\lim_{t\rightarrow\infty}\Fest_t \rightarrow \cF$,
    (b) $\blambda^{\star}$ is unique, and 
    (c) $\lim_{t\rightarrow\infty}\Hat{\blambda}_t \rightarrow \blambda^{\star}$. 
    Thus, for any $\allocation \in \cF$ and $\bmu$, 
    $$\lim_{t\rightarrow\infty}\mathcal{D}(\allocation,\bmu,\Fest_t) \rightarrow \mathcal{D}(\allocation,\bmu,\cF),
    $$
    where $\blambda^* \in \argmin_{\blambda \in \altsetreal}\allocation^{\top} \klvec$. %Proof is in Appendix \ref{sec:prop of Fhat}.\vspace*{-.5em}
\end{reptheorem}

\begin{proof} Here, we prove the three parts of the theorem consecutively.

\textbf{Statement (a): Convergence of the limit $\lim_{t\rightarrow \infty} \Fest$.}
To begin with the proof of the first statement of Theorem \ref{thm:Asymptotic Convergence of the Projection Lemma} we leverage the results stated in Theorem \ref{thm:Dantzig theorem on continuity}. Let $H(\Tilde{A}) \defn \left\{ \bpi\in \simp : \Tilde{A}\bpi  \leq 0 \right\}$
and the set function $\Tilde{A} \rightarrow H(\Tilde{A})\cap C$
where $C=\Fest$ is a non-empty compact (proven in Section \ref{sec:Strong duality proof} subset of $\simp$
Then the set $H(\Tilde{A})\cap C$ can be written as
\begin{align*}
    H(\Tilde{A})\cap C = \left\{ \bpi\in \Fest : \Tilde{A}\bpi  \leq 0 \right\}
\end{align*}

To apply Theorem \ref{thm:Dantzig theorem on continuity}, $\{\Tilde{A}^r,r\in\mathbb{N}\}$, must be a convergent sequence of affine function. It is evident that $\tilde{A}^r$ for any $r\in\mathbb{N}$ is an affine function since $A$ is linear in $A$ and the induced pessimism works as a translation. Then we can proceed to the next part of the proof of statement 1 where we prove that $\{\tilde{A}^r\}_{r\in\mathbb{N}}$ is a convergent sequence of functions. For ease of notation we will denote $\tilde{A}_t$ for the t-th element of the sequence $\{\tilde{A}^r\}_{r\in\mathbb{N}}$ for $t\in\mathbb{N}$.

The definition of the confidence radius for any constraint $i\in[\numconst]$ follows from the Definition \ref{def:conf_ellipse} as $f(\delta,t) \defn 1+\sqrt{\frac{1}{2}\log\frac{K}{\delta}+\frac{1}{4}\log\det\Sigma_t}$.
It is evident from the definition that $f(t,\delta)$ is a non-decreasing function w.r.t time and it grows with order of at least $\mathcal{O}({\sqrt{\log t}})$

We have from the definition of the confidence set, for all $i\in [\numconst]$
\begin{align*}
    & \mathbb{P}\left( \Hat{A}_t^i - f(t,\delta)||\omega_t||_{\Sigma_t^{-1}} \leq A^i \leq \Hat{A}_t^i + f(t,\delta)||\omega_t||_{\Sigma_t^{-1}} \right) \geq 1-\delta \\
    \implies& \mathbb{P}\left( -\frac{f(t,\delta)||\omega_t||_{\Sigma_t^{-1}}}{\sigma(\Hat{A}_t^i)} \leq \frac{\Hat{A}_t^i - A^i}{\sigma(\Hat{A}_t^i)} \leq \frac{f(t,\delta)||\omega_t||_{\Sigma_t^{-1}}}{\sigma(\Hat{A}_t^i)} \right) \geq 1-\delta\\
    \implies& \mathbb{P}\left( -\frac{f(t,\delta)||\omega_t||_{\Sigma_t^{-1}}}{\sigma(\Hat{A}_t^i)} \leq Z \leq \frac{f(t,\delta)||\omega_t||_{\Sigma_t^{-1}}}{\sigma(\Hat{A}_t^i)} \right) \geq 1-\delta\\
    \implies& 2\Phi\left( \frac{f(t,\delta)||\omega_t||_{\Sigma_t^{-1}}}{\sigma(\Hat{A}_t^i)}\right) \geq 2-\delta\\
    \implies& \frac{f(t,\delta)||\omega_t||_{\Sigma_t^{-1}}}{\sigma(\Hat{A}_t^i)} \geq \Phi^{-1} \left( 1-\frac{\delta}{2} \right)\\
    \implies& f(t,\delta)||\omega_t||_{\Sigma_t^{-1}} \geq \sigma(\Hat{A}_t^i) \Phi^{-1} \left( 1-\frac{\delta}{2} \right)
\end{align*}
where $Z\defn \frac{\Hat{A}_t^i - A^i}{\sigma(\Hat{A}_t^i)}$ and $\Phi(.)$ is the cumulative distribution function of $\mathcal{N}(0,1)$ distribution.
$\lim_{t\rightarrow \infty} f(t,\delta)\|\allocation_t\|_{\Sigma_t^{-1}} \rightarrow 0$ since $ \sigma(\Hat{A}_t^i) = \mathcal{O}(\sqrt{\frac{\log t}{t}})$. 
Leveraging CLT at this point we say
\begin{align*}
   \Hat{A}_t^i \xrightarrow{d} A^i, \forall i \in [\numconst] 
\end{align*}
Then by Slutsky's theorem~\citep{zbMATH02592216}, we conclude $\Hat{A}_t^i - f(t,\delta)||\omega_t||_{\Sigma_t^{-1}} \xrightarrow{d} A^i , \forall i \in [\numconst]$. 

It implies that $\{\tilde{A}_r\}_{r\in\mathbb{N}}$ is a convergent sequence of function for $A$. Now, we use Theorem \ref{thm:Dantzig theorem on continuity} and get the following properties of the feasible set.

\begin{enumerate}
    \item  $ H(\Tilde{A})\cap C\subset\lim_{r\rightarrow\infty} H(\Tilde{A}^{r})\cap C    $
    \item $\lim_{r\rightarrow\infty} H(\Tilde{A}^{r})\cap C$
    is a closed convex superset of $H(\Tilde{A})\cap C$.
    \item $H(\Tilde{A})\cap C$ has non-empty interior because of the feasibility condition and no component in $\tilde{A}$ is identically $\mathbf{0}$.
    \begin{align*}
        \lim_{r\rightarrow\infty} H(\Tilde{A}^{r})\cap C = H(\Tilde{A}) \cap C
    \end{align*}
    \item Even if the set $H(\Tilde{A})\cap C$ has empty interior or some component if $\Tilde{A}$ is identically zero, by the last statement of the Theorem \ref{thm:Dantzig theorem on continuity} we can say for any closed convex set Q of $H(\Tilde{A})\cap C$ we can design the function $\{\Tilde{A}^{r}\}$ in such a way that $\lim_{r\rightarrow\infty}H(\Tilde{A}^{r})\cap C$ includes $Q$.
\end{enumerate}

As the convergence of $\Tilde{A}_t$ is guaranteed now asymptotically, we can guaranty convergence of the following limit $\lim_{t\rightarrow\infty} \Hat{\mathcal{F}}_t \rightarrow \mathcal{F}$.

\textbf{Statement (b) : Proof of Uniqueness of $\blambda^{\star}$}
Here, we prove if there exists a confusing instance $\blambda^* \in \altset$ which uniquely minimises the the function $\mathcal{D}(.)$ defined as
\begin{align*}
    \mathcal{D}(\allocation,\bmu,\Fest) \defn \inf_{\lag\in\cL}\inf_{\blambda\in\altset}\allocation^{\top} \klvec - \lag^{\top}\Tilde{A}\allocation
\end{align*}
We observe that only the leading quantity on the R.H.S associated with the KL is dependent on $\blambda$. So, in this proof we will only show that $\blambda^{\star}$ minimizes the KL divergence uniquely and since the KL is linearly dependent on the expression, proving this will be enough to ensure uniqueness of $\blambda^{\star}$.\\

Now, from the properties of KL we know that $d(\bmu,\blambda)$ is convex on the pair $(\bmu,\blambda)$. But it is also strictly convex on $\blambda$ if $\mathrm{supp}(\blambda)\subseteq \mathrm{supp}(\bmu)$ which is true in our case, since $\bmu,\blambda\in\mathcal{D}\subseteq \mathbb{R}^k$.

Let us assume there are two local minima $\blambda_1$ and $\blambda_2$, with the condition,
\begin{align*}
    d(\bmu,\blambda_1) \leq d(\bmu,\blambda_2)
\end{align*}
Then, we can write from the property of strict convexity,
for some $\{h:0<h<1\}$,
\begin{align*}
    d(\bmu,h\blambda_1+(1-h)\blambda_2) < hd(\bmu,\blambda_1)+(1-h)d(\bmu,\blambda_2)
\end{align*}
Now, from the assumed condition on $\blambda_1$ and $\blambda_2$, we can write ---
\begin{align*}
     &d(\bmu,\blambda_1) \leq d(\bmu,\blambda_2)\\
     \implies& hd(\bmu,\blambda_1) \leq hd(\bmu,\blambda_2) \text{    , since } h>0\\
     \implies& hd(\bmu,\blambda_1)+(1-h)d(\bmu,\blambda_2) \leq hd(\bmu,\blambda_2)+(1-h)d(\bmu,\blambda_2)\\
     \implies& hd(\bmu,\blambda_1)+(1-h)d(\bmu,\blambda_2) \leq d(\bmu,\blambda_2)
\end{align*}

Putting this result in the strict convexity condition we get
\begin{align*}
    d(\bmu,h\blambda_1+(1-h)\blambda_2) < d(\bmu,\blambda_2)
\end{align*}
which is a contradiction.

Thus, we conclude that for a strictly convex function f(x) with supp(x) being convex as well, the set of minimisers is either empty or singleton. Then, we can say $\blambda^{\star}$ uniquely minimizes the KL, or say $\mathcal{D}(\allocation,\bmu,\Fest)$. Let us now once again remind the definition of perturbed alt-set $\altset \defn \{ \blambda\in\mathbb{D}: \max_{\pol\in\Fest} \blambda^{\top}(\Hat{\pol}^{\star}-\pol) < r\}$. Let us denote $\nu(\rpolest)$ as the neighborhood of $\rpolest$. Any $\bpi\in\Fest$ is called a neighbor of $\rpolest$, if it is an extreme point of $\Fest$ and shares (K-1) active constraints with $\rpolest$. Then, we can decompose the perturbed alt-set as $\altset = \bigcup_{\bpi\in\nu_{\Fest}(\rpolest)}\left\{ \blambda : \blambda^{\top}(\rpolest - \bpi )<0 \right\}$, which is a union of half-spaces for each neighbor. From this decomposition we can observe that $\rpolest$ is not the r-optimal policy for $\blambda$, i.e, $\{ \exists \bpi'\in\altset:\blambda^{\top}(\opt_{\Fest} - \bpi' )<0 \}$. Then, it follows similar argument in \cite{carlsson2024pure} to argue that the most confusing instance w.r.t $\bmu$ lies in the boundary of the normal cone, which lands us to Proposition \ref{prop:Projection Lemma for Lagrangian Formulation}.
    
For any $\allocation \in \Fest$ and $\bmu\in\cD$, the following projection lemma holds for the Lagrangian relaxation,
    \begin{align*}
        \mathcal{D}(\allocation,\bmu,\Fest) = & \min_{\lag\in\cL}\max_{\pol \in \rgoodest } \min_{\pol^{'}\in\nu_{\Fest}(\pol)} \min_{\blambda:\blambda^{\top}(\Hat{\pol}^{\star}-\pol^{'})=0} \allocation^{\top}\klvec - \lag^{\top} \Tilde{A}\allocation \\
        =& \min_{\lag\in\cL} \min_{\pol^{'}\in\nu_{\Fest}(\rpolest)} \min_{\blambda:\blambda^{\top}(\rpolest-\pol^{'})=0} \allocation^{\top}\klvec - \lag^{\top} \Tilde{A}\allocation
    \end{align*}

\textbf{Statement (c): Convergence of the sequence $\{\Hat{\blambda}_n\}_{n\in\mathbb{N}}$}.
In known constraint setting the agent has access to $\mathcal{F}$. That means there is the actual sequence $\{ \blambda_n\}_{n\in\mathbb{N}}$ for which $\blambda_n\to \blambda^*$ as $n\to\infty$ since $\mathcal{D}(\allocation,\bmu,\Fest)$ is convex and continuous on $\Lambda_{\Fest}(\bmu,\pol)$. But in this setting we try to estimate $\mathcal{F}$ as $\Fest_n$ at each time step $n\in\mathbb{N}$. So there exists the  $\{ \Hat{\blambda}_n \}_{n\in\mathbb{N}}$ such that $\Hat{\blambda}_n \in \Lambda_{\Fest_n}(\bmu,\pol)$ and we have to ensure it converges to the unique optimal $\blambda^*$ i.e $\blambda^{\star} \in \altsetreal \subseteq \lim_{n\rightarrow \infty} \Lambda_{\Fest_n}(\bmu,\pol)$ implies $\{ \Hat{\blambda}_n \} \to \blambda^*$ as $n\to\infty$

We use the fundamental theorem of limit to carry out this proof with the help of properties of the sets $\Fest$ and $\Lambda$. The properties we have already proven for these sets are
\begin{enumerate}
    \item $\Fest_n$ for any $n\in\mathbb{N}$ is a superset of $\mathcal{F}$ due to the optimistic choice of $A$. 
    \item $\Fest_n$ is a non-empty compact subset of $\simp$ and $\lim_{n\to\infty}\Fest_n = \mathcal{F}$.
    \item $\Lambda_{\Fest_n}(\bmu,\pol)$ is a closed convex set and it also is a superset of the real alt-set $\altset$.
\end{enumerate}

Leveraging these properties we claim that for any $\bmu \in \mathcal{D}$, $\lim_{n\to\infty}\Lambda_{\Fest_n}(\bmu,\pol) = \altset$. Since we have already proven uniqueness of $\blambda$ in statement 2, we say $\Hat{\blambda}_n$ uniquely minimises $\Lambda_{\Fest_n}(\bmu,\pol)$. Now from the $(\epsilon,\delta)$-definition of limits we say if $\Lambda_{\Fest_n}(\bmu,\pol)$ is an $\epsilon$-cover of $\altset$ for $\epsilon>0$, then $|\Hat{\blambda}_n - \blambda^*| \leq \delta$ for $\delta>0$ sufficiently small. It implies for a sequence of $\{\Hat{\lambda_n}\}_{n\in \mathbb{N}}$ we claim $\lim_{n\to\infty}\Hat{\blambda}_n = \blambda^*$ i.e the sequence convergence. Therefore we conclude by the statement itself

\begin{align*}
    \{\Hat{\blambda}_n\}_{n\in\mathbb{N}} \rightarrow \blambda^{\star}
\end{align*}
Hence, proved.
\end{proof}

\newpage
\section{Characterization of the Unique Optimal Policy: Proof of Theorem \ref{thm:Optimal Policy}}\label{sec:unique_pol}

\begin{reptheorem}{thm:Optimal Policy}
    For all $\bmu\in\cD$, $\allocation^{\star}(\bmu)$ satisfies the conditions \\
    1. Both the sets $\Fest$ and $\allocation^{\star}(\bmu)$ are closed and convex.\\
    2. For all $\bmu\in\cD$ and $\allocation\in\Fest$,  $\lim_{t\rightarrow\infty}\mathcal{D}(\allocation,\Hat{\bmu}_t,\Fest_t)$ is continuous.\\
    3. Reciprocal of the characteristic time $\lim_{t\rightarrow\infty} T^{-1}_{\Fest_t,r}(\bmu)$ is continuous for all $\bmu \in \cD$.\\
    4. For all  $\bmu \in \cD$, $\bmu \rightarrow \allocation^{\star}(\bmu)$ is upper hemi-continuous.\\
    Thus, the optimization problem $\max_{\pol\in\Hat{{\cF}}} \bmu^{\top}\pol$ has a unique solution.

\end{reptheorem}

\begin{proof}
The theorem has four statements as the sufficient condition for the existence of unique optimal policy. So naturally we will dictate the proof structure in four steps and prove the statements one by one.\\

\textbf{Statement 1: Convexity of feasible space and optimal set function.} Let us first analyse the properties of $\Fest$. For any two member of $\allocation_1 , \allocation_2 \in \Fest$ satisfying $\tilde{A}\allocation_1\leq 0$ and $\tilde{A}\allocation_2\leq 0$, their convex combination for any $\alpha\in[0,1]$,
\begin{align*}
   \tilde{A}\left(\alpha\allocation_1 + (1-\alpha)\allocation_2\right) = \alpha\tilde{A}\allocation_1 + (1-\alpha)\tilde{A}\allocation_2 \leq 0 
\end{align*}
Therefore we can say $\Fest$ is convex because it is closed under convex operation. We claim $\Fest$ is also closed since 

    (a) The complement of $\Fest$, $\Fest^c \defn \{\pol\in\simp:\tilde{A}\pol>0\}$ is an open set.
    (b) we have already proven the limit of $\Fest$ to be $\mathcal{F}$ which is always contained by $\Fest$.

The elements in the domain of optimal allocation set function must be included in $\Fest$. So compactness of $\allocation^* (\bmu)$ is a direct consequence of compactness of $\Fest$.

\textbf{Statement 2: Continuity of limit.} We have already proven in Section \ref{sec:prop of Fhat} that $\lim_{t\to\infty}\Fest_t \to \mathcal{F}$. Also by convexity of KL and CLT we claim $\Hat{\bmu}_t \to \bmu$ as $t\to\infty$ and since $\allocation$ is linear in $\mathcal{D}(\allocation_t,\bmu,\Fest_t)$ it will converge to $\allocation^* (\bmu)$ as $t\to \infty$, also due to convergence of $\Hat{\bmu}_t$. Then we can say that the limiting value is same as the value if we plug in the limits in $\mathcal{D}$ i.e $\lim_{t\to\infty}\mathcal{D}(\allocation_t,\Hat{\bmu}_t,\Fest_t) = \mathcal{D}(\allocation^*(\bmu),\bmu,\mathcal{F})$. So we ensure the continuity of $\lim_{t\rightarrow\infty}\mathcal{D}(\allocation_t,\Hat{\bmu}_t,\Fest)$.

\textbf{Statement 3: Continuity of limit of inverse sampling complexity.} This statement directly follows from the statement 2. Due to convexity of KL-divergence and convergence of $\Fest$, the limiting value exists and it is equal to the inverse of characteristic time with the limiting value.

\textbf{Statement 4: Upper hemi-continuity of optimal allocation function.} We refer to \citep{magureanu2014lipschitz} (see Lemma \ref{lemm:hemi-cont}) for this proof. We denote $Q(\tilde{A}') \defn \lim_{\Fest \to \mathcal{F}}\max_{\allocation\in\simp}\bigg\{ \sumk\allocation_a \klvec - \lag^{\top}\tilde{A}''\allocation \bigg| \tilde{A}''\allocation \leq 0, \allocation_a \geq 0 \forall i \in [\numarm]\bigg\} = \allocation(\bmu)$ where $\tilde{A}'' \in \reals^{\numarm\times\numarm}$ is the rank-1 update of $\tilde{A}'$ which is a sub-matrix of $\tilde{A}$ with $\numarm$ number of active constraints. We define limiting set as 
\begin{align*}
    Q^{\star}(\tilde{A}'') = \bigg\{ \allocation: \lim_{\Fest \to \mathcal{F}}\sumk \allocation_a\klvec = Q(\tilde{A}'') \bigg| \tilde{A}''\allocation \leq 0, \allocation_a \geq 0 \forall i \in [\numarm] \bigg\} = \allocation^* (\bmu)
\end{align*}

As a direct consequence of Lemma \ref{lemm:hemi-cont} we get the following results
\begin{enumerate}
    \item The function $\allocation^* (\bmu)$ is continuous in $(\reals^{\numarm\times\numarm})\times \reals^{\numarm}$
    \item $\allocation^* (\bmu)$ is upper-hemicontinuous on $(\reals^{\numarm\times\numarm})\times \reals^{\numarm}$
\end{enumerate}

Leveraging these four sufficient statements ensure that there exist unique solution for the optimization problem $\max_{\pol\in\Hat{{\cF}}} \bmu^{\top} \pol, \forall \bmu\in\mathcal{D}$ i.e the image set of the set-valued function $\allocation^*(.)$ is singleton.  
\end{proof}

\newpage
\section{Lagrangian Lower Bound for Gaussians: Proof of Theorem \ref{Characteristic time for Gaussian Distribution} }\label{sec:gaussian lb}

\ifdoublecol
\begin{reptheorem}{Characteristic time for Gaussian Distribution}
   Let $\{ P_a \}_{a\in[K]}$ be Gaussian distributions with equal variance $\sigma^2 > 0$  
   \begin{align*}
       T_{\Fest,r}^{-1}(\bmu)\notag  = \max_{\allocation\in\Fest}\min_{\lag\in\cL}\max_{\pol\in\rgoodest}\min_{\pol' \in \neighr}\left\{\frac{(r-\bmu^{\top}(\pol-\pol'))^2}{2\sigma^{2}\Vert\pol-\pol'\Vert_{\mathrm{Diag}(1/\allocation_a)}^2}  - \lag^{\top} \Tilde{A}\allocation \right\}\,.
   \end{align*}
   where $\mathrm{Diag}(1/\allocation_a)$ is a $K$-dimensional diagonal matrix with $a$-th diagonal entry $1/\allocation_a$.
\end{reptheorem}

\else
\begin{reptheorem}{Characteristic time for Gaussian Distribution}
	Let $\{ P_a \}_{a\in[K]}$ be Gaussian distributions with equal variance $\sigma^2 > 0$  
	\begin{align*}
		T_{\Fest,r}^{-1}(\bmu)\notag  = \max_{\allocation\in\Fest}\min_{\lag\in\cL}\max_{\pol\in\rgoodest}\min_{\pol' \in \neighr}\left\{\frac{(r-\bmu^{\top}(\pol-\pol'))^2}{2\sigma^{2}\Vert\pol-\pol'\Vert_{\mathrm{Diag}(1/\allocation_a)}^2}  - \lag^{\top} \Tilde{A}\allocation \right\}\,.
	\end{align*}
	where $\mathrm{Diag}(1/\allocation_a)$ is a $K$-dimensional diagonal matrix with $a$-th diagonal entry $1/\allocation_a$.
\end{reptheorem}
\fi

\ifdoublecol
\begin{proof}
We start the proof by the definition of $\mathcal{D}(\allocation,\bmu,\Fest)$ as per Equation \eqref{eqn:Projection Lemma for Lag} 

\begin{align}
    \mathcal{D}(\allocation,\bmu,\Fest) &= \min_{\lag\in\cL}\max_{\pol\in\rgoodest}\min_{\blambda\in\Lambda_{\Fest}(\bmu)} \left \{\sum_{a=1}^k \allocation_a d(\bmu_a,\blambda_a) - \lag^{\top}\Tilde{A}\allocation \right \} \notag\\ 
    &= \min_{\lag\in\cL} \max_{\pol\in\rgoodest}\min_{\pol' \in \neigh} \min_{\blambda:\blambda^{\top} (\pol-\pol')=r} \left \{\sum_{a=1}^k \allocation_a d(\bmu_a,\blambda_a) - \lag^{\top}\Tilde{A}\allocation \right \} {\color{red}\rightsquigarrow \text{ via Proposition } \ref{prop:Projection Lemma for Lagrangian Formulation}} 
\end{align}
The Lagrangian formulation of $\mathcal{D}(\allocation,\bmu,\Fest)$ is written as 
\begin{align*}
   \mathcal{L}(\allocation,\bmu,\Fest) = \min_{\lag\in\cL}\max_{\pol\in\rgoodest}\min_{\pol' \in \neighr}\left\{\sumk\allocation_a\kl - \lag^{\top} \Tilde{A}\allocation - \gamma \left(\sumk\blambda_a \va - r \right)\right\}
\end{align*}
where $\va \defn (\pol-\pol')_a$.

We assume both the instances $\bmu$ and $\blambda$ follow Gaussian distribution with same variance $\sigma^2$.\\
Then, we can rewrite the Lagrangian putting the value of the KL as ---
\begin{align}\label{eq.gaussian kl eq}
    \mathcal{L}(\allocation,\bmu,\Fest) = \min_{\gamma\in\mathbb{R}_{+}}\min_{\lag\in\cL}\max_{\pol\in\rgoodest} \min_{\pol' \in \neighr}\left\{\sumk\allocation_a \frac{(\bmu_a -\blambda_a)^2}{2\sigma^2} - \lag^{\top}\Tilde{A}\allocation - \gamma \left(\sumk\blambda_a \va -r \right)\right\}
\end{align}
Differentiating the Lagrangian w.r.t $\blambda_a$ and equating it to 0, we get
\begin{align*}
    &\nabla_{\blambda_a} \mathcal{L}(\blambda,\allocation,\bmu) = 0\\
    \text{or, }& - \frac{\allocation_a(\bmu_a-\blambda_a)}{\sigma^2} - \gamma \va = 0\\
    \text{or, }& \blambda_a = \bmu_a + \frac{\gamma \va \sigma^2}{\allocation_a}
\end{align*}
Then putting back the value of $\blambda_a$ in Equation \ref{eq.gaussian kl eq} we get 
\begin{align}\label{Lag}
    \mathcal{L}(\allocation,\bmu,\Fest) = \min_{\gamma\in\mathbb{R}_{+}}\min_{\lag\in\cL}\max_{\pol\in\rgoodest}\min_{\pol' \in \neighr}\left\{-\sumk \gamma^2 \frac{\va^2 \sigma^2}{2\allocation_a} - \lag^{\top}\Tilde{A}\allocation - \gamma\left(\sumk\bmu_a \va - r \right)\right\}
\end{align}
Again differentiating the Lagrangian w.r.t $\gamma$ and equating it to 0, we get 
\begin{align*}
    \nabla_{\gamma}\mathcal{L}(\allocation,\bmu,\Fest) = 0
    \implies & r -\sumk\bmu_a \va - \gamma\sumk \frac{\sigma^2}{\allocation_a}\va = 0\\
    \implies & \gamma = \frac{(r-\bmu^{\top}v)}{\sumk \frac{\sigma^2}{\allocation_a}\va^2} 
\end{align*}
Putting the value of $\gamma$ in Equation \ref{Lag}, we get ---
\begin{align*}
    \mathcal{L}(\allocation,\bmu,\Fest) &= \min_{\lag\in \cL}\max_{\pol\in\rgoodest}\min_{\pol' \in \neighr}\left\{\frac{(r-\bmu^{\top} v)^2}{2\sigma^2\sumk\frac{\va^2}{\allocation_a}} - \lag^{\top} \Tilde{A}\allocation \right\}\\
    &= \min_{\lag\in \cL}\max_{\pol\in\rgoodest}\min_{\pol' \in \neighr}\left\{ \frac{1}{2\sigma^2}\frac{(r - \bmu^{\top} (\pol-\pol'))^2}{\|\pol-\pol'\|_{\mathrm{Diag}(1/\allocation_a)}^2} - \lag^{\top} \Tilde{A}\allocation \right\}
\end{align*}

Therefore inverse characteristic time for Lagrangian relaxation with unknown constraints satisfies,
\begin{align*}
    T_{\Fest,r}^{-1}(\bmu) = \max_{\allocation\in\Fest}\min_{\lag\in \cL}\max_{\pol\in\rgoodest}\min_{\pol' \in \neighr}\left\{ \frac{1}{2\sigma^2}\frac{(r - \bmu^{\top} (\pol-\pol'))^2}{\|\pol-\pol'\|_{\mathrm{Diag}(1/\allocation_a)}^2} - \lag^{\top} \Tilde{A}\allocation \right\}
\end{align*}
\end{proof}

\else
\begin{proof}
	We start the proof by the definition of $\mathcal{D}(\allocation,\bmu,\Fest)$ as per Equation \eqref{eqn:Projection Lemma for Lag} 
	
	\begin{align}
		\mathcal{D}(\allocation,\bmu,\Fest) &= \min_{\lag\in\cL}\max_{\pol\in\rgoodest}\min_{\blambda\in\Lambda_{\Fest}(\bmu)} \left \{\sum_{a=1}^k \allocation_a d(\bmu_a,\blambda_a) - \lag^{\top}\Tilde{A}\allocation \right \} \notag\\ 
		&= \min_{\lag\in\cL} \max_{\pol\in\rgoodest}\min_{\pol' \in \neigh} \min_{\blambda:\blambda^{\top} (\pol-\pol')=r} \left \{\sum_{a=1}^k \allocation_a d(\bmu_a,\blambda_a) - \lag^{\top}\Tilde{A}\allocation \right \} {\color{red}\rightsquigarrow \text{ via Proposition } \ref{prop:Projection Lemma for Lagrangian Formulation}} 
	\end{align}
	The Lagrangian formulation of $\mathcal{D}(\allocation,\bmu,\Fest)$ is written as 
	\begin{align*}
		\mathcal{L}(\allocation,\bmu,\Fest) = \min_{\lag\in\cL}\max_{\pol\in\rgoodest}\min_{\pol' \in \neighr}\left\{\sumk\allocation_a\kl - \lag^{\top} \Tilde{A}\allocation - \gamma \left(\sumk\blambda_a \va - r \right)\right\}
	\end{align*}
	where $\va \defn (\pol-\pol')_a$.
	
	We assume both the instances $\bmu$ and $\blambda$ follow Gaussian distribution with same variance $\sigma^2$.\\
	Then, we can rewrite the Lagrangian putting the value of the KL as ---
	\begin{align}\label{eq.gaussian kl eq}
		\mathcal{L}(\allocation,\bmu,\Fest) = \min_{\gamma\in\mathbb{R}_{+}}\min_{\lag\in\cL}\max_{\pol\in\rgoodest} \min_{\pol' \in \neighr}\left\{\sumk\allocation_a \frac{(\bmu_a -\blambda_a)^2}{2\sigma^2} - \lag^{\top}\Tilde{A}\allocation - \gamma \left(\sumk\blambda_a \va -r \right)\right\}
	\end{align}
	Differentiating the Lagrangian w.r.t $\blambda_a$ and equating it to 0, we get
	\begin{align*}
		&\nabla_{\blambda_a} \mathcal{L}(\blambda,\allocation,\bmu) = 0\\
		\text{or, }& - \frac{\allocation_a(\bmu_a-\blambda_a)}{\sigma^2} - \gamma \va = 0\\
		\text{or, }& \blambda_a = \bmu_a + \frac{\gamma \va \sigma^2}{\allocation_a}
	\end{align*}
	Then putting back the value of $\blambda_a$ in Equation \ref{eq.gaussian kl eq} we get 
	\begin{align}\label{Lag}
		\mathcal{L}(\allocation,\bmu,\Fest) = \min_{\gamma\in\mathbb{R}_{+}}\min_{\lag\in\cL}\max_{\pol\in\rgoodest}\min_{\pol' \in \neighr}\left\{-\sumk \gamma^2 \frac{\va^2 \sigma^2}{2\allocation_a} - \lag^{\top}\Tilde{A}\allocation - \gamma\left(\sumk\bmu_a \va - r \right)\right\}
	\end{align}
	Again differentiating the Lagrangian w.r.t $\gamma$ and equating it to 0, we get 
	\begin{align*}
		\nabla_{\gamma}\mathcal{L}(\allocation,\bmu,\Fest) = 0
		\implies & r -\sumk\bmu_a \va - \gamma\sumk \frac{\sigma^2}{\allocation_a}\va = 0\\
		\implies & \gamma = \frac{(r-\bmu^{\top}v)}{\sumk \frac{\sigma^2}{\allocation_a}\va^2} 
	\end{align*}
	Putting the value of $\gamma$ in Equation \ref{Lag}, we get ---
	\begin{align*}
		\mathcal{L}(\allocation,\bmu,\Fest) &= \min_{\lag\in \cL}\max_{\pol\in\rgoodest}\min_{\pol' \in \neighr}\left\{\frac{(r-\bmu^{\top} v)^2}{2\sigma^2\sumk\frac{\va^2}{\allocation_a}} - \lag^{\top} \Tilde{A}\allocation \right\}\\
		&= \min_{\lag\in \cL}\max_{\pol\in\rgoodest}\min_{\pol' \in \neighr}\left\{ \frac{1}{2\sigma^2}\frac{(r - \bmu^{\top} (\pol-\pol'))^2}{\|\pol-\pol'\|_{\mathrm{Diag}(1/\allocation_a)}^2} - \lag^{\top} \Tilde{A}\allocation \right\}
	\end{align*}
	
	Therefore inverse characteristic time for Lagrangian relaxation with unknown constraints satisfies,
	\begin{align*}
		T_{\Fest,r}^{-1}(\bmu) = \max_{\allocation\in\Fest}\min_{\lag\in \cL}\max_{\pol\in\rgoodest}\min_{\pol' \in \neighr}\left\{ \frac{1}{2\sigma^2}\frac{(r - \bmu^{\top} (\pol-\pol'))^2}{\|\pol-\pol'\|_{\mathrm{Diag}(1/\allocation_a)}^2} - \lag^{\top} \Tilde{A}\allocation \right\}
	\end{align*}
\end{proof}
\fi

\subsection{Bounds on Sample complexity: Proof of Corollary \ref{cor:Bounds on characteristic time} Part (a) }\label{sec:cor 1}
\begin{repcorollary}{cor:Bounds on characteristic time} \textbf{Part (a)}
    Let $d_{\pol}^2 \defn  \frac{\Vert\opt_{\cF}-\pol\Vert_{\bmu\bmu^{\top}}^2}{\Vert\opt_{\cF}-\pol\Vert_2^2}$ be the norm of the projection of $\bmu$ on the policy gap $(\opt_{\cF}-\pol)$.\\
    \emph{Part i.} Then, we get 
    $\frac{2\sigma^2 K}{C_{\mathrm{known}}}(1+\SP_{\Tilde{A}})\leq T_{\Fest,0}(\bmu)\leq \frac{2\sigma^2 K}{C_{\mathrm{known}}}$, where $C_{\mathrm{known}} = \min_{\pol'' \in \neighreal}d_{\pol''}^2$. \\
    \emph{Part ii.} Additionally, $T_{\Fest,0}(\bmu) \geq \frac{H}{\kappa^2}(1+\SP_{\Tilde{A}})$, where $H$ is inversely proportional to the sum of squares of gaps and $\kappa_{\mathrm{known}}$ is the condition number of a sub-matrix of $A$ consisting $\numarm$ linearly independent active constraints for $\opt$. 
\end{repcorollary}

\begin{proof}
Here, we derive explicit expression for gaussian characterisation of the lower and upper bound on the characteristic time. We start the proof
with the difference in sample complexity between unknown and known constraint setting
\begin{align} \label{diff}
    \mathcal{D}(\allocation,\bmu,\Fest) - \mathcal{D}(\allocation,\bmu,\mathcal{F})
    &= \min_{\lag\in \cL}\min_{\pol' \in \neigh}\left\{\frac{1}{2\sigma^2}\frac{\|\opt_{\Fest}-\pol'\|_{\bmu\bmu^{\top}}^2}{\|\opt_{\Fest}-\pol'\|_{\mathrm{Diag}(1/\allocation_a)}^2} - \lag^{\top} \Tilde{A}\allocation \right\} \notag \\
    &\quad -\min_{\pol'' \in \neighreal}\frac{1}{2\sigma^2} \frac{\| \opt_{\mathcal{F}}-\pol''\|_{\bmu\bmu^{\top}}^2}{\|\opt_{\mathcal{F}}-\pol''\|_{\mathrm{Diag}(1/\allocation_a)}^2} 
\end{align}
Let us remind, due to pessimistic choice of $\tilde{A}$, $\mathcal{F} \subseteq \Fest$. $\pol'$ is a neighbor of $\opt_{\Fest}$ if it is an extreme point in the polytope $\Fest$ and shares (\numarm-1) active constraints with $\opt_{\Fest}$. Then $\opt_{\mathcal{F}}$ and $\pol''$ lies in the interior of $\Fest$ i.e, they can be expressed as a convex combination of $\opt_{\Fest}$ and $\pol'$. 
Let, $\exists 0\leq t_1 \leq 1:  \opt_{\mathcal{F}} = t_1\opt_{\Fest} + (1-t_1)\pol'$ and $\exists 0\leq t_1 \leq 1:  \pol'' = t_2\opt_{\Fest} + (1-t_2)\pol'$
Then, $(\opt_{\mathcal{F}}-\pol') = t_1 (\opt_{\Fest}-\pol')$ and $(\pol'' -\pol') = t_2 (\opt_{\Fest}-\pol')$. Then,
%\begin{align*}
%    \|\opt_{\mathcal{F}}-\pol''\|_{\mathrm{Diag}(1/\allocation_a)}^2
%    &=\|(\opt_{\mathcal{F}}-\pol')-(\pol''-\pol')\|_{\mathrm{Diag}(1/\allocation_a)}^2\\
%    &\leq 2\bigg\{ \|\opt_{\mathcal{F}}-\pol'\|_{\mathrm{Diag}(1/\allocation_a)}^2 + \|\pol''-\pol'\|_{\mathrm{Diag}(1/\allocation_a)}^2 \bigg\}\\
%    &= 2\{ t_1^2 + t_2^2 \} \|\opt_{\Fest}-\pol'\|_{\mathrm{Diag}(1/\allocation_a)}^2
%\end{align*}
%
%Also, 
\begin{align*}
	\|\opt_{\mathcal{F}} - \pol''\|_{\bmu\bmu^{\top}}^2 &= \|(\opt_{\mathcal{F}} -\pol') - (\pol''-\pol')\|_{\bmu\bmu^{\top}}^2 = \| t_1 (\opt_{\Fest} -\pol') - t_2 (\opt_{\Fest} -\pol')\|_{\bmu\bmu^{\top}}^2 = (t_1-t_2)^2\|\opt_{\Fest} -\pol'\|^2_{\bmu\bmu^{\top}}%\\
	%    &\geq 2\Bigg\{ \|\opt_{\mathcal{F}} -\pol'\|_{\bmu\bmu^{\top}}^2 - \|(\pol''-\pol')\|_{\bmu\bmu^{\top}}^2 \Bigg\}\\
	%    &= 2\Bigg\{ t_1^2\|\opt_{\Fest} -\pol'\|_{\bmu\bmu^{\top}}^2 - t_2^2\|(\opt_{\Fest}-\pol')\|_{\bmu\bmu^{\top}}^2 \Bigg\}\\
	%    &= 2(t_1^2 - t_2^2) \|\opt_{\Fest} -\pol'\|_{\bmu\bmu^{\top}}^2
\end{align*}
Applying this equality in Equation \ref{diff}, we get ---
\begin{align}
	&~~~\mathcal{D}(\allocation,\bmu,\cF) - \mathcal{D}(\allocation,\bmu,\Fest) \notag \\
	&=  \min_{\pol'' \in \neighreal}\frac{1}{2\sigma^2} \frac{\| \opt_{\mathcal{F}}-\pol''\|_{\bmu\bmu^{\top}}^2}{\|\opt_{\mathcal{F}}-\pol''\|_{\mathrm{Diag}(1/\allocation_a)}^2} -\min_{\lag\in \cL}\min_{\pol' \in \neigh}\left\{\frac{1}{2\sigma^2}\frac{\|\opt_{\Fest}-\pol'\|_{\bmu\bmu^{\top}}^2}{\|\opt_{\Fest}-\pol'\|_{\mathrm{Diag}(1/\allocation_a)}^2} - \lag^{\top} \Tilde{A}\allocation \right\}\notag \\
	&= \min_{\pol' \in \neigh}\frac{1}{2\sigma^2} \frac{(t_1-t_2)^2\| \opt_{\Fest}-\pol''\|_{\bmu\bmu^{\top}}^2}{(t_1-t_2)^2\|\opt_{\Fest}-\pol'\|_{\mathrm{Diag}(1/\allocation_a)}^2} - \min_{\lag\in \cL}\min_{\pol' \in \neigh}\left\{\frac{1}{2\sigma^2}\frac{\|\opt_{\Fest}-\pol'\|_{\bmu\bmu^{\top}}^2}{\|\opt_{\Fest}-\pol'\|_{\mathrm{Diag}(1/\allocation_a)}^2} - \lag^{\top} \Tilde{A}\allocation \right\} \notag\\
	&= \min_{\lag\in \cL}  \lag^{\top} \Tilde{A}\allocation \notag\\
	&\leq \mathcal{D}(\allocation,\bmu,\Fest)\SP_{\Tilde{A}}\label{eq:diff of bounds} %= - \lag^{\top}_{\min} \Tilde{A}\allocation = \lag^{\top}_{\min} {A}\allocation - \lag^{\top}_{\min} \Tilde{A}\allocation + \lag^{\top}_{\min} {A}\allocation\\
%	&\geq  - \|\lag_{\min}\|_1 \|\Tilde{A}\allocation\|_{\infty}\\
%	&\geq - \|\lag_{\min}\|_1 (\|\Tilde{A}\allocation -A\allocation\|_{\infty} + \|A\allocation\|_{\infty})\\
%	&\geq - \frac{\mathcal{D}(\allocation,\bmu,\Fest)}{\gamma_{\min}} (\|\Tilde{A}\allocation -A\allocation\|_{\infty} + \gamma_{\max})\\
%		&\geq - \mathfrak{s}\mathcal{D}(\allocation,\bmu,\Fest) - \|\Tilde{A}\allocation -A\allocation\|_{\infty} \frac{\mathcal{D}(\allocation,\bmu,\Fest)}{\gamma_{\min}}
\end{align}
The last inequality is a direct consequence of Lemma~\ref{lemm:constraint part upper bound}.\\
%This implies that
%\begin{align*}
%	\mathcal{D}(\allocation,\bmu,\Fest) (1+\mathfrak{s} +\frac{\|\Tilde{A}\allocation -A\allocation\|_{\infty}}{\gamma_{\min}}) \geq  \mathcal{D}(\allocation,\bmu,\mathcal{F}) 
%\end{align*}
%%%%%%%%%%%%%%%%%%%%%%%%%%%%%%%%%%%
%Path 2
\textbf{Part I.} Since $\min_{\lag\in \cL} - \lag^{\top} \Tilde{A}\allocation \geq 0$ by the design of the Lagrangian and the constraints, we get
\begin{align}\label{eq:15}
	&\mathcal{D}(\allocation,\bmu,\Fest) \geq \mathcal{D}(\allocation,\bmu,\mathcal{F})\notag\\
	\implies &T_{\Fest,0}^{-1}(\bmu) \geq T_{\cF,0}^{-1}(\bmu)\,
\end{align}

This leads us to the upper bound. 
%\todo[inline]{write the upper bound proof.}
For ease of comparison, we follow the same notation used in \cite{carlsson2024pure} and define for any $\pol\in\Delta_{\numarm-1}$, $d_{\pol}^2 \defn  \frac{\| \opt_{\mathcal{F}}-\pol\|_{\bmu\bmu^{\top}}^2}{\|\opt_{\mathcal{F}}-\pol\|_2^2}$, which is the squared distance between $\bmu$ and the hyperplane $(\opt_{\mathcal{F}}-\pol) = \mathbf{0}$. Therefore,
using \cite[Corollary 1]{carlsson2024pure} and Equation \eqref{eq:15}, we get 

\begin{align*}
	T_{\Fest,0}(\bmu) \leq \frac{2\sigma^2 \numarm}{\min_{\pol'' \in \neighreal}d_{\pol''}^2 }
\end{align*}

%\textbf{Part II.} Now, to get the lower bound, we observe that
%\begin{align*}
%	\min_{\lag\in \cL} - \lag^{\top} \Tilde{A}\allocation\ = - \lag^{\top}_{\min} \Tilde{A}\allocation &= \lag^{\top}_{\min} {A}\allocation - \lag^{\top}_{\min} \Tilde{A}\allocation + \lag^{\top}_{\min} {A}\allocation\\
%	&\leq \lag^{\top}_{\min} (\Tilde{A}-A)\allocation \\
%	&\leq  \|\lag_{\min}\|_1 \|(\Tilde{A}-A)\allocation \|_{\infty}\\ &\leq \frac{\|(\Tilde{A}-A)\allocation \|_{\infty}}{\gamma}\mathcal{D}(\allocation,\bmu,\Fest)
%\end{align*}
Along with Equation~\eqref{eq:diff of bounds}, this implies that
\begin{align}\label{eq:16}
	&\mathcal{D}(\allocation,\bmu,\Fest) \left(1+\SP_{\Tilde{A}}\right) \leq  \mathcal{D}(\allocation,\bmu,\mathcal{F})\notag \\
	\implies &T_{\Fest,0}^{-1}(\bmu)\left(1+\SP_{\Tilde{A}}\right) \leq T_{\cF,0}^{-1}(\bmu) 
\end{align}
%\todo[inline]{write the lower bound proof.}
We again leverage \cite[Corollary 1]{carlsson2024pure} to get lower bound on the characteristic time as,

\begin{align*}
	T_{\Fest,0}(\bmu) \geq \left(1+\SP_{\Tilde{A}}\right) \frac{2\sigma^2}{\min_{\pol'' \in \neighreal}d_{\pol''}^2 }
\end{align*}

%Let us assume that when constraint are well approximated $\|\Tilde{A}\allocation -A\allocation\|_{\infty} \leq \epsilon \leq \gamma$ holds for some $\epsilon>0$, which yields
%
%\begin{align*}
%	T_{\Fest,0}(\bmu) \geq \left(1-\frac{\epsilon}{\gamma}\right) \frac{2\sigma^2 }{\min_{\pol'' \in \neighreal}d_{\pol''}^2 }
%\end{align*}

Further, we define $C_{\mathrm{Known}} \defn \min_{\pol'' \in \neighreal}d_{\pol''}^2$, which concludes the proof.

\end{proof}

\subsection{Impact of Unknown Linear Constraints: Proof of Corollary \ref{cor:Bounds on characteristic time} Part (b)}

\begin{repcorollary}{cor:Bounds on characteristic time} \textbf{Part (b)}
    Let $d_{\pol}^2 \defn \frac{\Vert\opt_{\cF}-\pol\Vert_{\bmu\bmu^{\top}}^2}{\Vert\opt_{\cF}-\pol\Vert_2^2}$ be the norm of the projection of $\bmu$ on the policy gap $(\opt_{\cF}-\pol)$. Then,  the characteristic time $T_{\Fest,0}(\bmu)$ satisfies  $T_{\Fest,0}(\bmu) \geq \frac{H}{\kappa^2 }\left(1-\frac{\epsilon}{\gamma}\right)$. $H$ is the sum of squares of gaps. $\kappa$ is the condition numbers of a sub-matrix of $A$ that consists $\numarm$ linearly independent active constraints for $\opt$.
\end{repcorollary}

\begin{proof}
	This result a direct implication of Corollary 2 in ~\cite{carlsson2024pure}, that states, 

\begin{align*}
	T_{\cF,0}^{-1}(\bmu) \leq \frac{\kappa^2}{H}
\end{align*}

Here, $H = \frac{2\sigma^2}{\Delta^2}$, where $\Delta^2 = \sumk (\bmu^* - \bmu_a)^2$, i.e the sum of squares of sub-optimal gaps in the arms and $\kappa$ is the condition number of a sub-matrix of $A$ consisting at least $\numarm$ linearly independent active constraints for $\opt$. Referring to the proof structure of Corollary 2 in ~\citep{carlsson2024pure} is of independent interest.

Now we leverage equation Equation \eqref{eq:16} to get
\begin{align*}
	 T_{\Fest,0}(\bmu) \geq \frac{H}{\kappa^2 }\left(1+\SP_{\Tilde{A}}\right)
\end{align*}
\end{proof}
\iffalse
\begin{proof} 
We dictate this proof in total four steps. The first step mainly deals with a sub-matrix $\tilde{A}'$of $\tilde{A}$ with $\numarm$ number if of active constraints where it will be shown that we can get any neighbouring policy of $\opt_{\Fest}$ just by a rank-1 update of $\tilde{A}'$ which means we just one of the constraints inactive for $\opt_{\Fest}$. Using this result we will find an upper bound in the next step on the deviation between $\opt_{\Fest}$ and any of it's neighbouring policy sharing $\numarm-1$ number of active constraints. Once we find the upper bound on this deviation, in the third step we give a new expression for the characteristic time and we conclude the fourth step by reducing a new lower bound characterised the condition number of $\tilde{A}'$ and $\Hat{A}$, where $\Hat{A}$ is the sub-matrix of the actual constraint matrix $A$ with at least $\numarm$ number of active constraints.

\textbf{Step 1 : Optimal policy to neighbouring policy via rank-1 update.} We have the pessimistic estimate of $A$ as $\Tilde{A}$, which gives us the perturbed feasible space $\Fest$. Let, $\tilde{A}'$ be a sub-matrix of $\tilde{A}$ such that it consists of $\numarm$ linearly independent rows of $\Tilde{A}$ active at $\opt_{\Fest}$. We can then say $\opt_{\Fest} \in \nullspace$, where, $\nullspace$ is the null space of $\tilde{A}'$. Now for some neighbour of $\opt_{\Fest}$, $\pol' \in \neigh$ belongs to the null space of some $\tilde{A}''$, where this $\tilde{A}''$ can be expressed as a rank-1 update of $\tilde{A}'$. Specifically, $\tilde{A}'' = \tilde{A}' + e_r ({a''}_r - {a'}_r)^{\top}$. Here, ${a'}_r$ is the column corresponding to the r-th constraint of $\tilde{A}'$. We just want to replace this column with a new column ${a''}_r$, so we set $e_r$ as a vector which has 1 at the r-th position and 0 everywhere else.

\textbf{Step 2: Bounding distance between neighbor and optimal policy.} From Equation \ref{Charac main}, we have ---
\begin{align}\label{Eq. 19}
    T_{\Fest,0}^{-1}(\bmu) < T_{\mathcal{F},0}^{-1}(\bmu) + \frac{1}{2\sigma^2}\min_{\pol' \in \neigh} 2\frac{(\bmu^{\top} (\opt_{\Fest}-\pol'))^2}{\|\opt_{\Fest}-\pol'\|_{2}^2} 
\end{align}

So, it becomes evident that we need to get an upper bound on the quantity $\|\opt_{\Fest} - \pol'\|$. Let us start with, 
\begin{align*}
    \tilde{A}'(\pol' - \opt_{\Fest}) &= \tilde{A}'\pol'  \text{ since, } \opt_{\Fest}\in\nullspace\\
    &= \left\{ \tilde{A}'' + e_r ({a'}_r - {a''}_r)^{\top} \right\}\pol'\\
    &= e_r ({a'}_r - {a''}_r)^{\top} \pol'\\
    &= e_r {a'}_r^{\top} \pol' \text{ , since } {a''}_r\in \text{ column space of } \tilde{A}''\\
    \implies (\pol' - \opt_{\Fest}) &= {\tilde{A}'}^{-1}(e_r {a'}_r^{\top})\pol'
\end{align*}

We denote $\xi \defn {a'}_r^{\top} \pol'$ as it is the slack for the new r-th row/constraint. We also define $\sigma_{\mathrm{min}}(\tilde{A}')$ and $\sigma_{\mathrm{max}}(\tilde{A}')$ be respectively the minimum and maximum singular value of $\tilde{A}'$. Also, let $\kappa_{\mathrm{unknown}} \defn \frac{\sigma_{\mathrm{max}}(\tilde{A}')}{\sigma_{\mathrm{min}}(\tilde{A}')}$ be the minimum condition number for $\tilde{A}'$. Then by property of singular value of a matrix, it follows ---
\begin{align*}
    &\frac{1}{\minsing} = \minsinginv  \leq \min_{v:\|v\|_2 = 1} \|{\tilde{A}'}^{-1}v\|\leq \|{\tilde{A}'}^{-1}e_r\| \leq \max_{v:\|v\|_2 = 1} \|{\tilde{A}'}^{-1}v\| \\
    &\quad\quad\leq \maxsinginv = \frac{1}{\maxsing} \\
    \implies & \frac{|\xi|}{\minsing} \leq \|\opt_{\Fest}-\pol'\|_2 \leq \frac{\xi}{\maxsing}\,
\end{align*}

\textbf{Step 3 : A new expression for Characteristic time.} Plugging in the above obtained bound on $\|\opt_{\Fest}-\pol'\|_2$ in the expression of inverse of characteristic time in Theorem \ref{Characteristic time for Gaussian Distribution}, we get a new expression for the characteristic time.

\begin{align*}
     \frac{1}{2\sigma^2}\frac{\|\opt_{\Fest}-\pol'\|_{\bmu\bmu^{\top}}^2}{\|\opt_{\Fest}-\pol'\|_{\mathrm{Diag}(1/\allocation_a)}^2} &= \frac{1}{2\sigma^2}\frac{\|{\tilde{A}'}^{-1}(e_r {a'}_r^{\top})\pol'\|_{\bmu\bmu^{\top}}^2}{\|{\tilde{A}'}^{-1}(e_r {a'}_r^{\top})\pol'\|_{\mathrm{Diag}(1/\allocation_a)}^2}\\
     &= \frac{1}{2\sigma^2}\frac{\|{\tilde{A}'}^{-1}e_r \xi
    \|_{\bmu\bmu^{\top}}^2}{\|{\tilde{A}'}^{-1}e_r \xi \|_{\mathrm{Diag}(1/\allocation_a)}^2}\\
    &= \frac{1}{2\sigma^2}\frac{(\Delta^{\top} {\tilde{A}'}^{-1}e_r)^2}{\|{\tilde{A}'}^{-1}e_r  \|_{\mathrm{Diag}(1/\allocation_a)}^2}
\end{align*}
where $\Delta$ is the vector of sub-optimality gaps of arms, i.e $\Delta_a \defn \bmu^{\star} - \bmu_a$. Then, the new expression for the inverse of characteristic time is as follows,
\begin{align}\label{Eq. 20}
    T_{\Fest,0}^{-1}(\bmu) = \max_{\allocation \in \Fest} \min_{\lag\in \cL} \min_{\pol' \in \neigh} \left\{ \frac{1}{2\sigma^2}\frac{(\Delta^{\top} {\tilde{A}'}^{-1}e_r)^2}{\|{\tilde{A}'}^{-1}e_r  \|_{\mathrm{Diag}(1/\allocation_a)}^2} - \lag^{\top} \tilde{A} \allocation\right\} 
\end{align}

%\todo[inline]{clean the next steps}
\textbf{Step 4 : New expression for the lower bound.}

Combining Equation \ref{Charac main} and the new expression in Equation \ref{Eq. 20}, we get ---
\begin{align*}
    T_{\Fest,0}^{-1}(\bmu)&< T_{\mathcal{F},0}^{-1}(\bmu) + \frac{1}{\sigma^2}\max_{\allocation} \min_{\pol' \in \neigh} \frac{\| \opt_{\mathcal{F}}-\pol'\|_{\bmu\bmu^{\top}}^2}{\|\opt_{\mathcal{F}}-\pol'\|_{\mathrm{Diag}(1/\allocation_a)}^2} \\
    &\leq T_{\mathcal{F},0}^{-1}(\bmu) + \frac{1}{\sigma^2} \min_{\pol' \in \neigh} \frac{\| \opt_{\mathcal{F}}-\pol'\|_{\bmu\bmu^{\top}}^2}{\|\opt_{\mathcal{F}}-\pol'\|_{2}^2}\\
    &=  T_{\mathcal{F},0}^{-1}(\bmu) + \min_{\pol' \in \neigh}\frac{1}{2\sigma^2}\frac{(\Delta^{\top} {\tilde{A}'}^{-1}e_r)^2}{\|{\tilde{A}'}^{-1}e_r  \|_{2}^2} \\
    &\leq T_{\mathcal{F},0}^{-1}(\bmu) + \min_{\pol' \in \neigh} \frac{\Delta^2}{2\sigma^2} \frac{\maxsing}{\minsing}
\end{align*}
For ease of comparing, we denote $H \defn \frac{2\sigma^2}{\Delta^2}$. Therefore, the new expression of characteristic time satisfies,
\begin{align*}
    T_{\Fest,0}^{-1}(\bmu) \leq \min_{\pol''\in\neighreal} \frac{\kappa_{\mathrm{known}}^2}{H} + \min_{\pol'\in\neigh} \frac{\kappa_{\mathrm{unknown}}^2}{H} = \frac{1}{H}\left( \kappa_{\mathrm{known}}^2 +\kappa_{\mathrm{unknown}}^2\right)
\end{align*}
where, $\kappa_{\mathrm{known}} \defn \frac{\sigma_{\mathrm{max}}(\Hat{A})}{\sigma_{\mathrm{min}}(\Hat{A})}$ and $\kappa_{\mathrm{unknown}} \defn \frac{\sigma_{\mathrm{max}}(\tilde{A}')}{\sigma_{\mathrm{min}}(\tilde{A}')}$, $\Hat{A}$ and $\tilde{A}'$ being the sub-matrix of $A$ and $\tilde{A}$ having \numarm linearly independent rows.

Consequently expected stopping time is then lower bounded by
\begin{align*}
    \mathbb{E}[\tau_{\delta}] \geq \frac{H}{\kappa_{\mathrm{known}}^2 +\kappa_{\mathrm{unknown}}^2} \mathrm{kl}(\delta\|1-\delta)
\end{align*}
\end{proof}
\else

\newpage
\section{Sample Complexity Upper Bounds (Analysis of Algorithms)}\label{app:algo_proofs}
\subsection{Proof of Lemma \ref{lemm:new_stopping_time} : Implication of $(1-\delta)$-correctness}\label{delta correct}
\begin{replemma}{lemm:new_stopping_time}
	If the recommended policy is $(1-\delta)$-correct then it is $(1-\delta)$-feasible. 
\end{replemma}
\begin{proof}
	Let the recommended policy $\pol$ is $(1-\delta)$-correct. It means
	\begin{align*}
		&\mathbb{P}(\pol = \pol_{\mathcal{F}}^{\star}) \geq 1 -\delta\\
		\implies & \mathbb{P}(\pol = \pol_{\mathcal{F}}^{\star} \wedge A\pol_{\mathcal{F}}^{\star}\leq 0) \geq 1-\delta\\
		%\implies & \mathbb{P}( \{A\pol \leq A\pol_{\mathcal{F}}^{\star} \vee A\pol > A\pol_{\mathcal{F}}^{\star}\}\wedge A\pol_{\mathcal{F}}^{\star}\leq 0) \leq \delta\\
		%\implies & \mathbb{P}(A\pol \leq 0 \vee -\Gamma^{\star} < A\pol) \leq \delta \quad\text{,where} \hspace{0.1cm} \Gamma^{\star} \hspace{0.1cm} \text{is the slack value at stopping time}\\
		\implies & \mathbb{P}(A\pol \leq 0) \geq 1-\delta
	\end{align*}
	
	Hence $(1-\delta)$-correctness automatically implies $(1-\delta)$-feasibility.
\end{proof}
\subsection{Stopping Criterion}\label{subsec:stopping}

\begin{reptheorem}{Chernoff's Stopping rule with unknown constraints.}
The Chernoff \textbf{stopping rule to ensure $(1-\delta)$-correctness and $(1-\delta)$-feasibility} is
\begin{align*}\vspace*{-.5em}
\max_{\pol_t\in\Pi_{\Fest_t}^r}\inf_{\blambda \in \Lambda_{\Fest_t}(\Hat{\bmu}_t)} \sumk N_{a,t} d(\Hat{\bmu}_{a,t},\lambda_a) > \beta( t,\delta)\,, \vspace*{-.5em}
\end{align*}
where $\beta(t,\delta) \defn  3S_0 \log(1+\log N_{a,t}) + S_0\mathcal{T}\left( \frac{(\numarm\wedge\numconst) + \log\frac{1}{\delta}}{S_0}\right)$ and $0\leq S_0 \leq \numarm$.
\end{reptheorem}

\ifdoublecol
\begin{proof}
We dictate the proof in 2 steps. In the first step, we prove that the stopping time $\tau_{\delta}$ is finite. Then, in next step, we give an explicit expression of the stopping threshold by upper bounding probability of the bad event for stopping time $\tau_{\delta}$.

Let us first go through some notations. 

$\pol_t \defn \argmax_{\pol\in\Pi_{\Fest_t}^r} \Hat{\bmu}_t^{\top} \pol$, where $\Hat{\bmu}_t \in \argmin_{\blambda\in\Lambda_{\Fest_t}(\Hat{\bmu}_t)} \sumk N_{a,t}d(\Hat{\bmu}_{a,t},\blambda_a) - \lag_t^{\top}\tilde{A}_t N_{a,t}$. 

Algorithm \ref{alg:lagts} and \ref{alg:lagex} stops at a finite $\tau_{\delta} \in \mathbb{N}$ if the events $\inf_{\blambda\in\Lambda_{\Fest}(\Hat{\bmu}_t)}\{ \exists t\in\mathbb{N}: \sumk N_{a,t}d(\Hat{\bmu}_{a,t},\blambda_a)\} > \beta(t,\delta)$ and $\{\exists t\in\mathbb{N}:\|(\tilde{A}_t - A)N_t\|_{\infty} \leq t\rho(t,\delta)\}$ jointly occurs.

\textbf{Step 1: Finiteness of the stopping time.}
$A$ stopping time is finite if the parameters in the system converges to their true values in finite time, in our case the means of arms and the constraint matrix. Let us define $\mathcal{A} \defn \{ a\in[\numarm] : \lim_{t\to\infty}N_{a,t} < \infty \}$ as a sampling rule i.e if an arm belongs to this set $\mathcal{A}$, it has been sampled finitely and otherwise the arm has been sampled enough number of times so that the mean of that arm has converged to it's true value and the column in the constraint matrix corresponding to that arm as also converged. For arms $a\in[\numarm] \text{ and } a\in \mathcal{A}^c, \Hat{\bmu}_{a,t} \rightarrow \Tilde{\bmu}_a\neq \bmu_a$. When all parameters are concentrated $\mathcal{A} = \emptyset$, we say $\forall a\in [\numarm] : \Hat{\bmu}_{a}\to \bmu_a$. We also define the limit of this empirical sampling rule as $\allocation_{\infty} = \lim_{t\to\infty} \frac{N_{a,t}}{t} \forall a\in [\numarm]$.
We then write the stopping condition in a new way $\inf_{\blambda\in\Lambda_{\Fest}(\Hat{\bmu}_t)}\{ \exists t\in\mathbb{N}: \sumk \frac{N_{a,t}}{t}d(\Hat{\bmu}_{a,t},\blambda_a) > \frac{\beta(t,\delta)}{t}\}$. 
By continuity properties and knowing $\beta(t,.) \to \log\log t$ and $\rho(t,.) \to 0$ as $t\to\infty$, we claim by taking taking asymptotic limits both sides  $\inf_{\blambda\in\Lambda_{\Fest}(\Hat{\bmu}_t)}\sumk \frac{\allocation_{\infty,a}}{t}d(\Hat{\bmu}_{a,t},\blambda_a) > 0 $. 
We get strict inequality for the both the cases by the virtue of construction of the set $\mathcal{A}$ such that for arms $a\in \mathcal{A}, \allocation_{\infty} \neq 0$ and the KL-divergence is non-zero as $\blambda_a \neq \bmu$ since $\blambda \in \Lambda_{\Fest_t}(\Hat{\bmu}_t)$. 

\textbf{Step 2: Probability of bad event to Stopping threshold.}
Let $\allocation_t$ is the allocation associated to $N_t$. Then we define the bad event as
\begin{align*}
    U_t \defn \{\pol_{\tau_{\delta}}\notin \rgood\} = \bigcup_{\pol\notin \rgood } \Bigg\lbrace \exists t\in\mathbb{N}: \pol_{t+1} = \pol \wedge
     \max_{\pol\in\Pi_{\Fest_t}^r}\inf_{\blambda\in\Lambda_{\Fest_t}(\Hat{\bmu}_t)}\sumk N_{a,t}d(\Hat{\bmu}_{a,t},\blambda_a) > \beta(t,\delta)  \Bigg\rbrace
\end{align*}

Therefore probability of this bad event 
\begin{align}\label{eq.19}
    \mathbb{P}(U_t)
    = & \mathbb{P}\left(\bigcup_{\pol\notin\rgood} \left\{  \exists t\in\mathbb{N}: \pol_{t+1} = \pol \wedge \max_{\pol\in\Pi_{\Fest_t}^r}\inf_{\blambda\in\Lambda_{\Fest_t}(\Hat{\bmu}_t)} \sumk N_{a,t}d(\Hat{\bmu}_{a,t},\blambda_a)> \beta(t,\delta) \right\}\right) \notag\\ 
    \leq& \sum_{\pol\notin\rgood}\mathbb{P}\left\{  \exists t\in\mathbb{N}:  \max_{\pol\in\Pi_{\Fest_t}^r}\inf_{\blambda\in\Lambda_{\Fest_t}(\Hat{\bmu}_t)}\sumk N_{a,t}d(\Hat{\bmu}_{a,t},\blambda_a)\} > \beta(t,\delta) \right\}\notag\\
    \leq&\sum_{\pol\notin\rgood}\mathbb{P}\left\{  \exists t\in\mathbb{N}:  \sumk N_{a,t}d(\Hat{\bmu}_{a,t},\bmu_a) > \beta(t,\delta) \right\}
\end{align}

%    =& \sum_{\pol\notin\rgood}\mathbb{P}\Bigg\lbrace  \exists t\in\mathbb{N}:\inf_{\blambda\in\Lambda_{\Fest_t}(\Hat{\bmu}_t_t)}\inf_{\blambda'\in\Lambda_{\mathcal{F}}(\Hat{\bmu}_t)}\sumk N_{a,t}\left( d(\Hat{\bmu}_{a,t},\blambda_a)-d(\Hat{\bmu}_{a,t},{\blambda'}_a)\right)\\
%    &\quad\quad\quad\quad + \inf_{\blambda'\in\Lambda_{\mathcal{F}}(\Hat{\bmu}_t)}\sumk N_{a,t}d(\Hat{\bmu}_{a,t},{\blambda'}_a) > \beta(t,\delta)\Bigg\rbrace\\
%    \leq& \sum_{\pol\notin\rgood}\mathbb{P}\left\{ \exists t\in\mathbb{N}: \inf_{\blambda'\in\Lambda_{\mathcal{F}}(\Hat{\bmu}_t)}\sumk N_{a,t}d(\Hat{\bmu}_{a,t},{\blambda'}_a) > \beta(t,\delta) \right\} 
%\end{align*}
%The last inequality holds because
%\begin{align*}
%    \inf_{\blambda\in\Lambda_{\Fest_t}(\Hat{\bmu}_t_t)}\inf_{\blambda'\in\Lambda_{\mathcal{F}}(\Hat{\bmu}_t)}\sumk N_{a,t}\left( d(\Hat{\bmu}_{a,t},\blambda_a)-d(\Hat{\bmu}_{a,t},{\blambda'}_a)\right)\leq0
%\end{align*} 
%since under $\Fest_t$ and the bad event we are assuming that the estimated alt-set is still bigger than the actual alt-set. So any $\blambda \in \Lambda_{\Fest_t}(\Hat{\bmu}_t_t)$  will have a bigger distance with the estimate of $\bmu$ than any $\blambda' \in \Lambda_{\mathcal{F}}(\Hat{\bmu}_t)$. 
We define $I_{\pol} \defn \mathrm{Supp}(\rpolest)\Delta \mathrm{Supp}(\pol)$ and also $S_0 \defn \max_{\pol}|I_{\pol}|$. We note that $0\leq S_0\leq \numarm$.

We get from Theorem~\ref{lemm:Kauffman martinagle stopping} in \citep{JMLR:v22:18-798} with the notation of $\mathcal{T}(.)$ follows from Lemma \ref{lemm:Kauffman martinagle stopping}
\begin{align*}
	\sum_{\pol\notin\rgood}\mathbb{P}\left\{  \exists t\in\mathbb{N}: \sum_{a\in I_{\pol}} N_{a,t}d(\Hat{\bmu}_{a,t},\bmu_a) \geq \sum_{a\in I_{\pol}} 3\log(1+\log N_{a,t}) + |S_0|\mathcal{T}\left( \frac{\log\frac{|\pol_{t}|-1}{\delta}}{S_0}\right) \right\} \leq \delta
\end{align*}
where $\delta$ is chosen to be $\frac{\delta}{|\pol_{t}|-1}$ such that $\log \frac{|\pol_{t}|-1}{\delta}\leq \log(\frac{2^{\numarm}}{\delta})\leq (\numarm\wedge\numconst) + \log\frac{1}{\delta}$

Also $\sum_{a\in I_{\pol}} 3\log(1+\log N_{a,t} \leq 3S_0 \log(1+\log N_{a,t})$. Therefore the stopping threshold is given by
\begin{align*}
    \beta(t,\delta) =  3S_0 \log(1+\log N_{a,t}) + S_0\mathcal{T}\left( \frac{(\numarm\wedge\numconst) + \log\frac{1}{\delta}}{S_0}\right)
\end{align*}
In practice, we use $S_0 = \numarm$.
\end{proof}
\else
\begin{proof}
We dictate the proof in 2 steps. In the first step, we prove that the stopping time $\tau_{\delta}$ is finite. Then, in next step, we give an explicit expression of the stopping threshold by upper bounding probability of the bad event for stopping time $\tau_{\delta}$.

Let us first go through some notations. 

$\pol_t \defn \argmax_{\pol\in\Fest} \Hat{\bmu}_t^{\top} \pol$, where $\Hat{\bmu}_t \in \argmin_{\blambda\in\Lambda_{\Fest_t}(\Hat{\bmu}_t)} \sumk N_{a,t}d(\Hat{\bmu}_{a,t},\blambda_a) - \lag_t^{\top}\tilde{A}_t N_{a,t}$. 

Algorithm \ref{alg:lagts} and \ref{alg:lagex} stops at a finite $\tau_{\delta} \in \mathbb{N}$ if the events $\inf_{\blambda\in\Lambda_{\Fest}(\Hat{\bmu}_t)}\{ \exists t\in\mathbb{N}: \sumk N_{a,t}d(\Hat{\bmu}_{a,t},\blambda_a)\} > \beta(t,\delta)$ and $\{\exists t\in\mathbb{N}:\|(\tilde{A}_t - A)N_t\|_{\infty} \leq t\rho(t,\delta)\}$ jointly occurs.

\textbf{Step 1: Finiteness of the stopping time.}
$A$ stopping time is finite if the parameters in the system converges to their true values in finite time, in our case the means of arms and the constraint matrix. Let us define $\mathcal{A} \defn \{ a\in[\numarm] : \lim_{t\to\infty}N_{a,t} < \infty \}$ as a sampling rule i.e if an arm belongs to this set $\mathcal{A}$, it has been sampled finitely and otherwise the arm has been sampled enough number of times so that the mean of that arm has converged to it's true value and the column in the constraint matrix corresponding to that arm as also converged. For arms $a\in[\numarm] \text{ and } a\in \mathcal{A}^c, \Hat{\bmu}_{a,t} \rightarrow \Tilde{\bmu}_a\neq \bmu_a$ and $(\tilde{A})_{a,t}\to (A')_{a}\neq (A)_a$. When all parameters are concentrated $\mathcal{A} = \emptyset$, we say $\forall a\in [\numarm] : \Hat{\bmu}_{a}\to \bmu_a$ and $\tilde{A} \to A$. We also define the limit of this empirical sampling rule as $\allocation_{\infty} = \lim_{t\to\infty} \frac{N_{a,t}}{t} \forall a\in [\numarm]$.
We then write the stopping condition in a new way $\inf_{\blambda\in\Lambda_{\Fest}(\Hat{\bmu}_t)}\{ \exists t\in\mathbb{N}: \sumk \frac{N_{a,t}}{t}d(\Hat{\bmu}_{a,t},\blambda_a) > \frac{\beta(t,\delta)}{t}\}$ and $\{\exists t\in\mathbb{N}:\|(\tilde{A}_t - A)\frac{N_t}{t}\|_{\infty} \leq \rho(t,\delta)\}$. 
By continuity properties and knowing $\beta(t,.) \to \log\log t$ and $\rho(t,.) \to 0$ as $t\to\infty$, we claim by taking taking asymptotic limits both sides  $\inf_{\blambda\in\Lambda_{\Fest}(\Hat{\bmu}_t)}\sumk \frac{\allocation_{\infty,a}}{t}d(\Hat{\bmu}_{a,t},\blambda_a) > 0 $ and also $\|(\tilde{A}_t - A)\allocation_{\infty,a}\|_{\infty} < 0$. 
We get strict inequality for the both the cases by the virtue of construction of the set $\mathcal{A}$ such that for arms $a\in \mathcal{A}, \allocation_{\infty} \neq 0$ and the KL-divergence is non-zero as $\blambda_a \neq \bmu$ since $\blambda \in \Lambda_{\Fest_t}(\Hat{\bmu}_t)$. Also for the second strict inequality, since $\tilde{A}_t$ is the optimistic estimate of $A$ at time t the condition will hold.

\textbf{Step 2: Probability of bad event to Stopping threshold.}
Let $\allocation_t$ is the allocation associated to $N_t$. Then we define the bad event as
\begin{align*}
    &U_t \defn \{\pol_{\tau_{\delta}}\neq \opt_{\mathcal{F}}\} \\
    =& \bigcup_{\pol\notin\rgood} \Bigg\lbrace \exists t\in\mathbb{N}: \pol_{t+1} = \pol\\
    &\quad\qquad \qquad \qquad\wedge \left\{ \inf_{\blambda\in\Lambda_{\Fest_t}(\Hat{\bmu}_t)}\sumk N_{a,t}d(\Hat{\bmu}_{a,t},\blambda_a) > \beta(t,\delta) \wedge A\allocation_t \leq 0 \right\} \Bigg\rbrace\\
    (a) \atop \subseteq& \bigcup_{\pol\notin\rgood} \Bigg\lbrace \exists t\in\mathbb{N}: \pol_{t+1} = \pol\\
    &\quad\qquad \qquad \qquad\wedge \left\{ \inf_{\blambda\in\Lambda_{\Fest_t}(\Hat{\bmu}_t)}\sumk N_{a,t}d(\Hat{\bmu}_{a,t},\blambda_a) > \beta(t,\delta) \wedge \|(\tilde{A}_t - A)\allocation_t\|_{\infty} \leq \rho(t,\delta) \right\} \Bigg\rbrace
\end{align*}
The argument (a) holds because
\begin{align*}
    \mathbb{P}\{ 0\geq A\allocation_t \} = \mathbb{P}\{ -\tilde{A}_t\allocation_t \leq (\tilde{A}_t-A)\allocation_t\} &\leq  \mathbb{P}\{ \|\tilde{A}_t\allocation_t\|_1 > \|(\tilde{A}_t-A)\allocation_t\|_1 \geq \|(\tilde{A}_t-A)\allocation_t\|_{\infty}\}\\
    \leq& \mathbb{P}\{ \|\tilde{A}_t\allocation_t\|_1 \leq \rho(t,\delta)\}\mathbb{P}\{\|(\tilde{A}_t-A)\allocation_t\|_{\infty} \leq \rho(t,\delta) \}\\
    &=\mathbb{P}\{\|(\tilde{A}_t-A)\allocation_t\|_{\infty} \leq \rho(t,\delta) \}
\end{align*}
since the event $\{\|\tilde{A}_t\allocation_t\|_1 \leq \rho(t,\delta)\}$ is a sure event.

Therefore probability of this bad event 
\begin{align}\label{eq.19}
    \mathbb{P}(U_t)\notag
    \leq &\bigcup_{\pol\notin\rgood} \mathbb{P}\left\{  \exists t\in\mathbb{N}: \pol_{t+1} = \pol \wedge \inf_{\blambda\in\Lambda_{\Fest_t}(\Hat{\bmu}_t)}\sumk N_{a,t}d(\Hat{\bmu}_{a,t},\blambda_a)\} > \beta(t,\delta) \right\} \notag\\ &\quad \mathbb{P}\left\{ \exists t\in\mathbb{N}: \pol_{t+1} = \pol \wedge \|(\tilde{A}_t - A)\allocation_t\|_{\infty} > \rho(t,\delta)\right\}\notag\\
    =& \sum_{\pol\neq\opt_{\mathcal{F}}}\mathbb{P}\left\{  \exists t\in\mathbb{N}:  \inf_{\blambda\in\Lambda_{\Fest_t}(\Hat{\bmu}_t)}\sumk N_{a,t}d(\Hat{\bmu}_{a,t},\blambda_a)\} > \beta(t,\delta) \right\} \notag\\ &\quad+ \sum_{\pol\neq\opt_{\mathcal{F}}}\mathbb{P}\left\{ \exists t\in\mathbb{N}: \|(\tilde{A}_t - A)\allocation_t\|_{\infty} > \rho(t,\delta)\right\}
\end{align}

The second cumulative probability can be bound using Lemma \ref{lemm:bad event prob}, i.e
\begin{align*}
    \sum_{\pol\neq\opt_{\mathcal{F}}}\mathbb{P}\left\{ \exists t\in\mathbb{N}: \|(\tilde{A}_t - A)\allocation_t\|_{\infty} > \rho(t,\delta)\right\} \leq \frac{1}{t}
\end{align*}
for the choice of $\rho(t,\delta)$ given in Lemma \ref{lemm:bad event prob}.
We work with the first term in R.H.S of Equation \eqref{eq.19}.
\begin{align*}
    &\sum_{\pol\neq\opt_{\mathcal{F}}}\mathbb{P}\left\{  \exists t\in\mathbb{N}:  \inf_{\blambda\in\Lambda_{\Fest_t}(\Hat{\bmu}_t)}\sumk N_{a,t}d(\Hat{\bmu}_{a,t},\blambda_a)\} > \beta(t,\delta) \right\}\\
    =& \sum_{\pol\neq\opt_{\mathcal{F}}}\mathbb{P}\Bigg\lbrace  \exists t\in\mathbb{N}:\inf_{\blambda\in\Lambda_{\Fest_t}(\Hat{\bmu}_t)}\inf_{\blambda'\in\Lambda_{\mathcal{F}}(\Hat{\bmu}_t)}\sumk N_{a,t}\left( d(\Hat{\bmu}_{a,t},\blambda_a)-d(\Hat{\bmu}_{a,t},{\blambda'}_a)\right)\\
    &\quad\quad\quad\quad + \inf_{\blambda'\in\Lambda_{\mathcal{F}}(\Hat{\bmu}_t)}\sumk N_{a,t}d(\Hat{\bmu}_{a,t},{\blambda'}_a) > \beta(t,\delta)\Bigg\rbrace\\
    \leq& \sum_{\pol\neq\opt_{\mathcal{F}}}\mathbb{P}\left\{ \exists t\in\mathbb{N}: \inf_{\blambda'\in\Lambda_{\mathcal{F}}(\Hat{\bmu}_t)}\sumk N_{a,t}d(\Hat{\bmu}_{a,t},{\blambda'}_a) \leq \beta(t,\delta) \right\} 
\end{align*}
The last inequality holds because
\begin{align*}
    \inf_{\blambda\in\Lambda_{\Fest_t}(\Hat{\bmu}_t)}\inf_{\blambda'\in\Lambda_{\mathcal{F}}(\Hat{\bmu}_t)}\sumk N_{a,t}\left( d(\Hat{\bmu}_{a,t},\blambda_a)-d(\Hat{\bmu}_{a,t},{\blambda'}_a)\right)\leq0
\end{align*} 
since under $\Fest_t$ and the bad event we are assuming that the estimated alt-set is still bigger than the actual alt-set. So any $\blambda \in \Lambda_{\Fest_t}(\Hat{\bmu}_t)$  will have a bigger distance with the estimate of $\bmu$ than any $\blambda' \in \Lambda_{\mathcal{F}}(\Hat{\bmu}_t)$. We define $I_{\pol} \defn \mathrm{Supp}(\opt_{\mathcal{F}})\Delta \mathrm{Supp}(\pol)$ and also $S_0 \defn \max_{\pol}|I_{\pol}|$. We note that $0\leq S_0\leq \numarm$.

We get from Lemma \ref{lemm:Kauffman martinagle stopping} in \citep{JMLR:v22:18-798} with the notation of $\mathcal{T}(.)$ follows from Lemma \ref{lemm:Kauffman martinagle stopping}
\begin{align*}
    &\sum_{\pol\neq\opt_{\mathcal{F}}}\mathbb{P}\left\{  \exists t\in\mathbb{N}: \sum_{a\in I_{\pol}} N_{a,t}d(\Hat{\bmu}_{a,t},\bmu_a) \geq \sum_{a\in I_{\pol}} 3\log(1+\log N_{a,t}) + |S_0|\mathcal{T}\left( \frac{\log\frac{|\pol_{\Fest_t}|-1}{\delta}}{S_0}\right) \right\}\\
    &\leq \frac{\delta}{t} \leq \delta
\end{align*}
where $\delta$ is chosen to be $\frac{\delta}{|\pol_{\Fest_t}|-1}$ such that $\log \frac{|\pol_{\Fest_t}|-1}{\delta}\leq \log(\frac{2^{\numarm}}{\delta})\leq (\numarm\wedge\numconst) + \log\frac{1}{\delta}$

Also $\sum_{a\in I_{\pol}} 3\log(1+\log N_{a,t} \leq 3S_0 \log(1+\log N_{a,t})$. Therefore the stopping threshold is given by
\begin{align*}
    \beta(t,\delta) =  3S_0 \log(1+\log N_{a,t}) + S_0\mathcal{T}\left( \frac{(\numarm\wedge\numconst) + \log\frac{1}{\delta}}{S_0}\right)
\end{align*}
In practice, we use $S_0 = \numarm$.
\end{proof}
\fi

\subsection{Upper Bound of LATS}\label{lats ub}
\begin{reptheorem}{thm:LATS upper bound}
The sample complexity upper bound of LATS to yield a $(1-\delta)$-correct optimal policy is
\begin{align*}
  \lim_{\delta\rightarrow 0} \frac{\mathbb{E}[\tau_{\delta}]}{\log(1/\delta)} \leq \alpha(1+\mathfrak{s})\charknown,
\end{align*}
where $\mathfrak{s}$ is the shadow price of the true constraint $A$, $\charknown$ is the characteristic time under the true constraint (Equation~\eqref{eqn:char_time_known}), and $\delta \in (0,1]$.
\end{reptheorem}
\begin{proof}
We will prove this theorem in 5 steps. In the first step, we define what is considered to be the good event in our unknown constraint setting, then we go on bounding the probability of the complement of this good event in step 2. Once the parameter concentrations are taken care of, we show how we can lower bound the instantaneous complexity of the algorithm in step 3. In step 4, we finally prove the upper bound on stopping time for both good and bad events. We conclude with the asymptotic upper bound on stopping time i.e when $\delta \rightarrow 0$ and $\epsilon \rightarrow 0$ in step 5.

\textbf{Step 1: Defining the good event.} Given an $\epsilon > 0$, for $h(T) = \max{\sqrt{T},f(T,\delta)}$ we define the good event $G_T$ as,
\begin{align*}
    G_T \triangleq \bigcap_{t=h(T)}^{\top} \{ ||\Hat{\bmu}_t - \bmu||_{\infty} \leq \xi(\epsilon) \wedge ||(\tilde{A}_t^i - A^i)\allocation||_{\infty} \leq \phi(\epsilon) \forall \allocation \in \Fest \}
\end{align*}
where, $\xi(\epsilon) \leq \max\limits_{\pol \in \rgoodest}\max\limits_{\pol'\in\neighr} \frac{1}{5}\bmu^{\top} (\pol-\pol')$, and $\phi(\epsilon) \triangleq \max(1,\epsilon)$ for a given $\epsilon > 0$.
The good event implies that the means and constraints are well concentrated in an $\epsilon$-ball around their true values. Thus, we have to now bound the extra cost of their correctness and the number of samples required to reach the good events.

We also observe that $||\bmu' - \bmu||_{\infty} \leq \xi(\epsilon)$ and $||(\tilde{A}_t^i - A^i)\allocation||_{\infty} \leq \phi(\epsilon)$ implies that $\sup\limits_{\allocation'\in\allocation^{\star}(\bmu')}\sup\limits_{\allocation\in\allocation^{\star}(\bmu)}||\allocation'-\allocation||\leq \epsilon$ due to upper hemicontinuity of $\allocation^*(\bmu)$ (Theorem \ref{thm:Optimal Policy}). 

\textbf{Step 2: Samples to Achieve the Good Event.} Now, let us bound the probability of bad event,
\begin{align*}
    \mathbb{P} (G_T^c) &= \sum_{t=h(T)}^{\top} \left( \mathbb{P}\left\{ ||\Hat{\bmu}_t - \bmu||_{\infty} >\xi(\epsilon)\right\} + \mathbb{P}\left\{  ||(\tilde{A}_t^i - A^i)\allocation||_{\infty} > \phi(\epsilon) \right\} \right)\\
    &\leq \sum_{t=h(T)}^{\top} \mathbb{P}\left\{ ||\Hat{\bmu}_t - \bmu||_{\infty} >\xi(\epsilon)\right\} + \sum_{t=h(T)}^{\top} \mathbb{P}\left\{  ||(\tilde{A}_t^i - A^i)\allocation||_{\infty} > \phi(\epsilon) \right\}\\
    &\leq BT\exp\left(-CT^{\frac{1}{8}}\right) + \numarm \sum_{t=h(T)}^{\top} \frac{1}{t}
\end{align*}
The first inequality is due to the union bound. The second inequality is due to the Lemma ~\ref{lemm:cumulative bad event for mu} (Lemma 19 of~\cite{pmlr-v49-garivier16a}), which states that
\begin{align*}
    \sum_{t=h(T)}^{\top} \mathbb{P}\left\{ \|\Hat{\bmu}_t - \bmu\|_{\infty} >\xi(\epsilon)\right\} \leq BT\exp\left(-CT^{\frac{1}{8}}\right)\, ,
\end{align*}
and also due to Lemma~\ref{lemm:bad event prob} that proves concentration bound of the constraint matrix over time.

\textbf{Step 3: Tracking argument.} Now, we state how concentrating on means and constraints leads to good concentration on the allocations too. Since we use C-tracking, we can leverage the concentration in allocation by(Lemma 17 of~\cite{degenne2019nonasymptotic}). We use this lemma than D-tracking or the tracking argument in (Lemma 7 of~\cite{pmlr-v49-garivier16a}) because the optimal allocations might not be unique but the set $\allocation^*(\bmu)$ is convex (Theorem~\ref{thm:Optimal Policy}).

Hence, there exists a $T_{\epsilon}$ such that under the good event and $t\geq \max(T_{\epsilon},h(T))$, we have,
\begin{align*}
    &|(\bmu-\Hat{\bmu}_t)^{\top} \opt| \leq |(\bmu-\Hat{\bmu}_t)^{\top} \rpolest| + |(\bmu-\Hat{\bmu}_t)^{\top} (\rpolest - \opt_{\Fest})| \leq 4\xi(\epsilon) \\
    &\quad\quad\quad\quad\quad\quad\quad\qquad\qquad\leq \max\limits_{\pol\in\rgoodest}\max\limits_{\pol'\in\neighr} \bmu^{\top} (\pol - \pol')
\end{align*}
We have replaced the perturbed alt-set with the true one because for $t\geq \max(T_{\epsilon},h(T))$, we can ensure the convergence of $\Fest$ almost surely i.e $\Fest \overset{\mathrm{a.s}}{\to} \mathcal{F}$ 

\textbf{Step 3: Complexity of identification under good event and constraint.} Now, we want to understand how hard it is to hit the stopping rule even under the good event.
First, we define 
\begin{align*}
    C_{\epsilon,\Fest}\defn \inf_{\substack{\bmu':||\bmu' - \bmu||_{\infty}\leq \xi(\epsilon)\\\allocation':||\allocation' -\allocation||_{\infty}\leq 3\epsilon\\\tilde{A}':||(\tilde{A}-A)\allocation||_{\infty}\leq \phi(\epsilon)}}\mathcal{D}(\bmu',\allocation',\Fest) - \lag^{\top} \tilde{A}\allocation \,.
\end{align*}

Now leveraging Lemma~\ref{lemm:constraint part upper bound},  we obtain
\begin{align*}
    (1+\psi) \inf_{\substack{\bmu':||\bmu' - \bmu||_{\infty}\leq \xi(\epsilon)\\\allocation':||\allocation' -\allocation||_{\infty}\leq 3\epsilon\\\tilde{A}:||(\tilde{A}-A)\allocation||_{\infty}\leq \phi(\epsilon)}}\mathcal{D}(\bmu',\allocation',\Fest) \geq \inf_{\substack{\bmu':||\bmu' - \bmu||_{\infty}\leq \xi(\epsilon)\\\allocation':||\allocation' -\allocation||_{\infty}\leq 3\epsilon\\\tilde{A}':||(\tilde{A}-A)\allocation||_{\infty}\leq \phi(\epsilon)}}\mathcal{D}(\bmu',\allocation',\Fest) - \lag^{\top} \tilde{A}\allocation \, ,
\end{align*}
where definition of $\psi$ follows from Lemma~\ref{lemm:constraint part upper bound}. It quantifies how the Lagrangian lower bound relates with the Likelihood Ratio Test-based quantity in the stopping time.

% Let, $C_{\epsilon,\Fest} \defn \inf\limits_{\substack{\bmu':||\bmu' - \bmu||_{\infty}\leq \xi(\epsilon)\\\allocation':||\allocation' -\allocation||_{\infty}\leq 3\epsilon\\\tilde{A}':||(\tilde{A}-A)\allocation||_{\infty} \leq \beta(\epsilon)}}\mathcal{D}(\bmu',\allocation',\Fest)$
Therefore by the C-tracking argument in Theorem~\ref{lemm:C-tracking Remy}, we can state
\begin{align}\label{eq.lb_D}
    \mathcal{D}(\Hat{\bmu}_t, N_t, \Fest) \geq t\inf_{\substack{\bmu':||\bmu' - \bmu||_{\infty}\leq \xi(\epsilon)\\\allocation':||\allocation' -\allocation||_{\infty}\leq 3\epsilon\\\tilde{A}:||(\tilde{A}-A)\allocation||_{\infty}\leq \phi(\epsilon)}}\mathcal{D}(\bmu',\allocation',\Fest) \geq \frac{tC_{\epsilon,\Fest}}{1+\psi}\, .
\end{align}
Here, LHS is the quantity that we use to stop and yield a $(1-\delta)$-correct policy.

\textbf{Step 4: Bounding the stopping time with good and bad events.} We denote $\tau_{\delta}$ as the stopping time. So for $T\geq T_{\epsilon}$ we can write upper bound on this stopping time for both good and bad events as
\begin{align*}
    \min(\tau_{\delta},T) \leq \max(\sqrt{T},f(T,\delta)) + \sum_{t=T_{\epsilon}}^{\top} \indicator_{\tau_{\delta}>t}
\end{align*}
By the correctness of the stopping time, the event $\tau(\delta) > t$ happens if $\mathcal{D}(\Hat{\bmu}_t, N_t, \Fest) \leq \beta(t,\delta)$ for any $t\leq T$.

Now  using the lower bound on $\mathcal{D}(\Hat{\bmu}_t, N_t, \Fest)$ (Equation~\ref{eq.lb_D}), we get
\begin{align*}
     \min(\tau_{\delta},T) \leq \beta(t,\delta)) &\leq \max(\sqrt{T},\beta(T,\delta)) + \sum_{t=T_{\epsilon}}^{\top} \indicator\Bigl(t\frac{C_{\epsilon,\Fest}}{1+\psi} \leq \beta(t,\delta)\Bigr)\\
     &\leq \max(\sqrt{T},\beta(T,\delta)) + \sum_{t=T_{\epsilon}}^{\top} \indicator\Bigl(t\frac{C_{\epsilon,\Fest}}{1+\psi} \leq \beta(T,\delta)\Bigr)\\
    &\leq\max(\sqrt{T},\beta(T,\delta)) + \frac{\beta(T,\delta)(1+\psi)}{C_{\epsilon,\Fest}}
\end{align*}
Let us define $T_{\delta} \triangleq \inf\{ T\in\mathbb{N}:\max(\sqrt{T},f(T,\delta))+\frac{\beta(T,\delta)(1+\psi)}{C_{\epsilon,\Fest}} \leq T\}$. To find a lower bound on $T_{\delta}$, we refer to \citep{pmlr-v49-garivier16a}. Specifically, let us define $c(\eta) \triangleq  \inf\lbrace{T:T-\max(\sqrt{T},\beta(T,\delta)) \geq \frac{T}{1+\eta}}\rbrace$ for some $\eta>0$. Therefore,
\begin{align*}
    T_{\delta} \leq c(\eta) + \inf\{ T\in\mathbb{N}: T\frac{C_{\epsilon,\Fest}}{(1+\psi_{\Fest})(1+\eta)} \geq \beta(T,\delta) \}
\end{align*}

Thus, finally combining the results, we upper bound the stopping time as
\begin{align*}
    \mathbb{E}[\tau_{\delta}] \leq T_{\epsilon} + T_{\delta} + T_{\mathrm{bad}} \, .
\end{align*}
Here, $T_{\mathrm{bad}} = \sum_{t=1}^{\infty}BT\mathrm{exp}\left(-CT^{1/8}\right) + \numarm\zeta(1) < \infty$ is the sum of probability of the bad events over time. $\zeta(.)$ denotes the Euler-Riemann Zeta function.

\textbf{Step 5. Deriving the asymptotics.} Now, we leverage the continuity properties of the Lagrangian characteristic time under approximate constraint to show that we converge to traditional hardness measures as $\epsilon$ and $\delta$ tends to zero.

First, we observe that for some $\alpha > 1$
\begin{align*}
  \limsup_{\delta\rightarrow 0} \frac{\mathbb{E}[\tau_{\delta}]}{\log(1/\delta)} \leq \frac{\alpha(1+\psi_{\Fest})(1+\eta)}{C_{\epsilon,\Fest} } 
\end{align*}
Now, if also $\epsilon \rightarrow 0$, by the Equation ~\ref{eq:lb_known}, we get $\tilde{A} \rightarrow A$, and thus, $\Fest \rightarrow F$.

Thus, by continuity properties in Theorem~\ref{thm:Optimal Policy} and Theorem~\ref{thm:Asymptotic Convergence of the Projection Lemma}, we get that 
\begin{align*}
    \mathcal{D(\Fest,.) \rightarrow \mathcal{D}(F,.)}~\text{and}~\psi_{\Fest} \rightarrow  \frac{\max_{i\in[1,N]}\Gamma}{\min_{i\in[1,N]}\Gamma} \defn \frac{\Gamma_{\mathrm{max}}}{\Gamma_{\mathrm{min}}} \defn \mathfrak{s}\,.
\end{align*}
Here, $\mathfrak{s}$ is the shadow price of the true constraint matrix, and quantifies the change in the constraint values due to one unit change in the policy vector.

Hence, we conclude that 
\begin{align*}
    \lim_{ \delta\rightarrow 0} \frac{\mathbb{E}[\tau_{\delta}]}{\log(1/\delta)} \leq \alpha T_{\mathcal{F},r}(\bmu) (1+\SP) \, .
\end{align*}
\end{proof}

\subsection{Upper Bound for LAGEX}\label{lagex ub}

\begin{reptheorem}{thm:LAGEX upper bound}
    The expected sample complexity of LAGEX $\lim_{\substack{\delta\rightarrow 0}} \frac{\mathbb{E}[\tau_{\delta}]}{\log(1/\delta)} \leq T_{\cF,r}(\bmu).$
\end{reptheorem}

\begin{proof} We will do this proof in two parts. In part (a) we will assume that the current recommended policy is the correct policy and try to find an upper bound on the sample complexity of LAGEX. In the next part (b) we break that assumption and try to get an upper bound on the number of steps the recommended policy is not the correct policy.

\textbf{Part (a) : Current recommended policy is correct.}
Proof structure of this part involves several steps. We start with defining the good event where we introduce a new event associated with the concentration event of the constraint set, then proceeding to prove concentration on that good event. Third step starts with the stopping criterion explained in \ref{subsec:stopping}. In step 4 we define LAGEX as an approximate saddle point algorithm. The next step further transforms the stopping criterion with the help of allocation and instance player's regret that play the zero-sum game. We conclude with the asymptotic upper on the sample complexity characterised by the additive effect of the novel quantity shadow price $\SP$.

\textbf{Step 1: Defining the good event.} We start the proof first by defining the good event as 
\begin{align*}
    G_t = \{ \forall t\leq T, \forall a\in[\numarm]: N_{a,t}d(\Hat{\bmu}_{a,t},\bmu_a) \leq g(t) \wedge \|(\tilde{A}_t - A)\allocation \|_{\infty} \leq \rho(t,\delta)  \}
\end{align*}

where, $g(t) = 3\log t + \log\log t$ and $\rho(t,\delta)$ is defined in Lemma~\ref{lemm:Concentration of the infinite norm}. The choice of $g(t)$ is motivated from~\citep{NEURIPS2019_8d1de745} which originates from the negative branch of the Lambert's W function. This eventually helps us upper bounding the cumulative probability of the bad event. 

\textbf{Step 2: Concentrating to the good event}
We denote $G_t^c$ as the bad event where any one of the above events does not occur.
Cumulative probability of this bad event 
\begin{align*}
     \sum_{s=1}^{\top}\mathbb{P}\left( G_t^c \right) &= \sum_{s=1}^{\top} \mathbb{P}\left( \sumk N_{a,s}d(\Hat{\bmu}_{a,s},\bmu_a)>g(s) \right)+ \sum_{s=1}^{\top} \mathbb{P}\left(\|(\tilde{A}_s - A)\allocation\|_{\infty} > \rho(s,\delta) \forall \allocation\in\simp\right)
\end{align*}
We get the upper bound on 
\[\sum_{s=1}^{\top} \mathbb{P} \left( \sumk N_{a,s} d(\Hat{\bmu}_{a,s},\bmu_a) > g(s) \right) \leq 1+ \sum_{s=2}^{\infty} \frac{\exp(2)}{s^3 \log s}(g(s)+g^2(s) \log s) \leq 19.48\]
as a direct consequence of (Lemma 6 of~\cite{NEURIPS2019_8d1de745}). The second cumulative probability is bounded by $\zeta(1) < 0.578$ using Lemma \ref{lemm:bad event prob}, which is finite.

In the next step, we work with the stopping criterion where we do not have access to $\mathcal{F}$ rather a bigger feasible set $\Fest_t$. 

\textbf{Step 3: Working with the stopping criterion.} The stopping criterion implies that
\begin{align*}
    \beta(t,\delta) \geq \max_{\pol\in\trgoodest}\inf_{\blambda\in\Lambda_{\Fest_t}(\Hat{\bmu}_t)} \sumk N_{a,t}d(\Hat{\bmu}_{a,t},\blambda_a) \,,
\end{align*}
where the exact expression of $\beta(t,\delta)$ is defined in Theorem \ref{Chernoff's Stopping rule with unknown constraints.}. 

We use the C-tracking lemma (Lemma~\ref{lemm:Kaufmann tracking}) to express the stopping in terms of allocations 
\begin{align}\label{eq.20}
    \beta(t,\delta) \geq \max_{\pol\in\trgoodest}\inf_{\blambda\in\Lambda_{\Fest_t}(\Hat{\bmu}_t)} \sum_{s=1}^t\sumk \allocation_{a,s}d(\Hat{\bmu}_{a,t},\blambda_a) -(1+\sqrt{t})\numarm 
\end{align}
L-Lipschitz property of KL divergence gives
\begin{align}\label{eq.21}
    |d(\bmu_a,\blambda_a)-d(\Hat{\bmu}_{a,s},\blambda_a)| \leq L\sqrt{2\sigma^2\frac{g(s)}{N_{a,s}}}
\end{align}
Using this result in Equation \eqref{eq.20} we get

\begin{align*}
    \sum_{s=1}^t\sumk \allocation_{a,s}d(\Hat{\bmu}_{a,s},\blambda_a) &\geq \sum_{s=1}^t\sumk \allocation_{a,s}d(\Hat{\bmu}_{a},\blambda_a) - L\sqrt{2\sigma^2 \numarm t g(t)}\\
    &\geq \sum_{s=1}^t\sumk \allocation_{a,s}d(\Hat{\bmu}_{a,s},\blambda_a) - L\sqrt{2\sigma^2 \numarm t g(t)}-2L\sqrt{2\sigma^2 g(t)}(\numarm^2+2\sqrt{2\numarm t})\\
    &\geq \sum_{s=1}^t\sumk \allocation_{a,s}d(\Hat{\bmu}_{a,s},\blambda_a)-\mathcal{O}(\sqrt{t\log t})
\end{align*}
the penultimate inequality yields from using the Equation \eqref{eq.21}. Using this result in Equation \eqref{eq.20}
\begin{align}\label{eq. 22}
    \beta(t,\delta) \geq \max_{\pol\in\trgoodest}\inf_{\blambda\in\Lambda_{\Fest_t}(\Hat{\bmu}_t)} \sumk \allocation_{a,s}d(\Hat{\bmu}_{a,t},\blambda_a) -(1+\sqrt{t})\numarm - \mathcal{O}(\sqrt{t\log t})
\end{align}

\textbf{Step 4: LAGEX (Algorithm \ref{alg:lagex}) as an optimistic saddle point algorithm}
We follow the definition of approximate saddle point algorithm in \citep{NEURIPS2019_8d1de745}. LAGEX acts as an approximate saddle point algorithm if 
\begin{align}\label{eq.23}
    &\max_{\pol\in\trgoodest}\inf_{\blambda\in\Lambda_{\Fest_t}(\Hat{\bmu}_t)} \sum_{s=1}^t 
    \sumk \allocation_{s,a} d(\Hat{\bmu}_{a,s},\blambda_a)\notag \\ \geq &\max_{\allocation\in\Fest_t} \sumk \sum_{s=1}^t \allocation_a U_{a,s} - R_t^{\allocation} - \sum_{s=1}^t\sumk \allocation_{a,s}C_{a,s} + \lag_t^{\top}\Tilde{A}_t \allocation_t\,, %- \lag_t^{\top}\tilde{A}_t\allocation_t 
\end{align}

where $U_{a,s} \defn \max \bigg\{ \frac{g(s)}{N_{a,s}},\max_{\xi\in[\alpha_{a,s},\beta_{a,s}]}d(\xi,\lambda_{a,s}) \bigg\}$, % and $x_t \defn R_t^{\allocation} + \sum_{s=1}^t\sumk \allocation_{a,s}C_{a,s} - \lag_t^{\top}\Tilde{A}_t \allocation_t$. 
 $C_{a,s} \defn U_{a,s} - d(\Hat{\bmu}_{a,s},\lambda_{a,s}),$ 
and $R_t^{\allocation}$ is defined in Equation~\eqref{eq:allocation_regret}.

\textbf{Step 5: Bounds cumulative regret of players}
Algorithm \ref{alg:lagex} at each step solves a two player zero-sum game. First one is the allocation player who uses AdaGrad to maximize the inverse of  characteristic time function to find the optimal allocation. %The regret of the AdaGrad player is defined at time $t\in\mathbb{N}$ as

\textbf{Step 5a. $\blambda$-player's regret.}
\begin{align*}
    R_{\blambda}^{t} \defn \sum_{s=1}^t \sumk \allocation_{a,s}d(\Hat{\bmu}_{a,s},\lambda_{a,s}) -  \max_{\pol\in\trgoodest}\inf_{\blambda\in\Lambda_{\Fest_t}(\Hat{\bmu}_t)}\sum_{s=1}^t \sumk \allocation_{a,s}d(\Hat{\bmu}_{a,s},\blambda_a) \leq 0
\end{align*}
The last inequality holds because we take infimum over $\blambda$ in the perturbed alt-set. now let us prove that LAGEX is a optimistic saddle point algorithm.  From the definition of regret of the $\blambda$-player we get
\begin{align*}
    \max_{\pol\in\trgoodest}\inf_{\blambda\in\Lambda_{\Fest_t}(\Hat{\bmu}_t)} \sum_{s=1}^t \sumk \allocation_{a,s}d(\Hat{\bmu}_{a,s},\blambda_a)
    \geq \sum_{s=1}^t \sumk \allocation_{a,s}U_{a,s}-\sum_{s=1}^t \sumk\allocation_{a,s}C_{a,s} \,.
\end{align*}
Then we have from Equation \eqref{eq. 22} and Equation \eqref{eq.23}
\begin{align}\label{eq.final_beta_lb}
    \beta(t,\delta) &\geq \sum_{s=1}^t \sumk \allocation_{a,s}U_{a,s}-\sum_{s=1}^t \sumk\allocation_{a,s}C_{a,s} -(1+\sqrt{t})\numarm - \mathcal{O}(\sqrt{t\log t})\,,%\notag\\
%    &\geq  \sum_{s=1}^t \sumk \allocation_{a,s}U_{a,s}-\sum_{s=1}^t \sumk\allocation_{a,s}C_{a,s}  -(1+\sqrt{t})\numarm - \mathcal{O}(\sqrt{t\log t}) \notag\\
%    &\geq \sum_{s=1}^t \sumk \allocation_{a,s}U_{a,s}-\sum_{s=1}^t \sumk\allocation_{a,s}C_{a,s}   -(1+\sqrt{t})\numarm- \mathcal{O}(\sqrt{t\log t}) \,,
\end{align}
%where $\psi_t$ is defined as Lemma \ref{lemm:constraint part upper bound}. The penultimate inequality holds as we have replaced $\inf_{\lag\in\cL_t}\lag$ with $\lag_t$ i.e the optimised Lagrangian multiplier at the last step. Whereas the last inequality follows from Lemma \ref{lemm:constraint part upper bound}.   
\textbf{Step 5b. Allocation player's regret.}
\begin{align}\label{eq:allocation_regret}
	R_t^{\allocation} \defn \max_{\allocation\in\simp}\sum_{s=1}^t \sumk \allocation_a U_{a,s} - \sum_{s=1}^t \sumk\allocation_{a,s} U_{a,s}
\end{align}
We note that AdaGrad enjoys regret of order $R_t^{\allocation}\leq \mathcal{O}(\sqrt{Qt})$ where Q is an upper bound on the losses such that $Q\geq \max_{x,y\in[\bmu_{\mathrm{min}},\bmu_{\mathrm{max}}]}d(x,y)$.

Now, from the allocation player's regret, we have
\begin{align*}
    \inf_{\blambda\in\altset}\sum_{s=1}^t\sumk \allocation_{a,s}d(\Hat{\bmu}_{a,s},\lambda_s)\geq \max_{\allocation\in\Fest} \sum_{s=1}^t\sumk \allocation_a U_{a,s} - R_t^{\allocation} - \underset{T1}{\underbrace{\sum_{s=1}^t\sumk\allocation_{a,s}C_{a,s}}} +  \lag_t^{\top}\Tilde{A}_t \allocation_t\\
    \geq \max_{\allocation\in\Fest} \sum_{s=1}^t\sumk \allocation_a U_{a,s} - \mathcal{O}(\sqrt{Qt}) - \underset{T1}{\underbrace{\sum_{s=1}^t\sumk\allocation_{a,s}C_{a,s}}} +  \lag_t^{\top}\Tilde{A}_t \allocation_t
\end{align*}
which shows that LAGEX is a approximate saddle point algorithm with a slack $R_t^{\allocation} + \sum_{s=1}^t\sumk\allocation_{a,s}C_{a,s} - \lag_t\Tilde{A}_t \allocation_t$. 

\textbf{Step 5c. Bounding T1.} Now we have to ensure that $T1$ in the slack is bounded.

\begin{align}
    \sum_{s=\numarm+1}^t\sumk \allocation_{a,s}C_{a,s} &\leq \sum_{s=\numarm+1}^t\sumk\allocation_{a,s} \left(\frac{g(s)}{N_{a,s}} + \max_{\xi\in[\alpha_{a,s},\beta_{a,s}]}d(\xi,\lambda_{a,s})  - d(\Hat{\bmu}_{a,s},\lambda_{a,s})\right) \notag \\
     &\leq \sum_{s=\numarm+1}^t\sumk\allocation_{a,s}\bigg\{\frac{g(s)}{N_{a,s}}+L\sqrt{2\sigma^2 \frac{g(s)}{N_{a,s}}}\bigg\}\label{ineq: bound cas}\\
    &\leq g(t)\sum_{s=\numarm+1}^t\sumk \frac{\allocation_{a,s}}{N_{a,s}}+L\sqrt{2\sigma^2 g(t)}\sum_{s=\numarm+1}^t\sumk\frac{\allocation_{a,s}}{\sqrt{N_{a,s}}}\notag\\
    &\leq g(t)\left( \numarm^2+2\numarm\log\frac{t}{\numarm}\right)+ L\sqrt{2\sigma^2 g(t)}(\numarm^2+2\sqrt{2\numarm t})\notag\\
    &\leq \mathcal{O}(\sqrt{t\log t}) \,.\label{eq.Cas bound}
\end{align}
Equation~\eqref{ineq: bound cas} holds due to  L-Lipschitz property of KL under the good event $G_T$. Specifically, 
\begin{align*}
    &|d(\bmu_a,\blambda_a) - d(\Hat{\bmu}_{a,s},\blambda_a)|\leq L\sqrt{2\sigma^2\frac{g(s)}{N_{a,s}}}\,.%\\
    %\implies& U_{a,s} - d(\xi,\lambda_{a,s})\leq \max\bigg\{2L\sqrt{2\sigma^2\frac{g(s)}{N_{a,s}}},\frac{g(s)}{N_{a,s}}\bigg\}
\end{align*}
\textbf{Step 6. From the two-player regrets to stopping time.} Finally, combining the results in Step 5, i.e. Equation \eqref{eq.final_beta_lb} and Equation \eqref{eq.Cas bound}, we get
\begin{align}
    \beta(t,\delta) &\geq \max_{\allocation\in\simp} \left(\sumk\sum_{s=1}^t \allocation_a d(\bmu_a,\lambda_{a,s})\right)-\mathcal{O}(\sqrt{Qt})- \mathcal{O}(\sqrt{t\log t})-(1+\sqrt{t})\numarm + \lag_t^{\top}\Tilde{A}_t \allocation_t \notag\\
    &\geq \max_{\allocation\in\simp} \sumk\sum_{s=1}^t \allocation_a d(\bmu_a,\lambda_{a,s}) -\mathcal{O}(\sqrt{Qt})- \mathcal{O}(\sqrt{t\log t})-(1+\sqrt{t})\numarm + \lag_t^{\top}\Tilde{A}_t \allocation_t \notag\\
    &\geq \max_{\allocation\in\simp} \sumk\sum_{s=1}^t \allocation_a d(\bmu_a,\lambda_{a,s}) -\mathcal{O}(\sqrt{Qt})- \mathcal{O}(\sqrt{t\log t})-(1+\sqrt{t})\numarm - \frac{\psi_t}{T_{\Fest_t}(\bmu)}\,.\label{eq:final_lb_beta}
\end{align}
The last inequality is true due to Lemma~\ref{lemm:constraint part upper bound}, i.e. $ - \lag_t^{\top}\Tilde{A}_t \allocation_t \leq \frac{\psi_t}{T_{\Fest_t}(\bmu)}$.

%\textbf{Step 6: Characteristic time} 
Now, we observe that 
\begin{align*}
    \max_{\allocation\in\simp}\sum_{s=1}^t \sumk \allocation_a d(\bmu_a,\lambda_{a,s})
    \geq t\max_{\allocation\in\simp} \max_{\pol\in\trgoodest} \inf_{\blambda\in\Lambda_{\Fest_t}(\Hat{\bmu}_t)} \sumk \allocation_a d(\bmu_a,\blambda_a) = tT_{\Fest_t}^{-1}(\bmu)
\end{align*}
Replacing this inequality in Equation~\eqref{eq:final_lb_beta}, we finally obtain
\begin{align*}
    t\leq T_{\Fest_t}(\bmu)\left(\beta(t,\delta) +R_t^{\allocation}+\mathcal{O}(\sqrt{t\log t}) \right) + \psi_t
\end{align*}

\textbf{Part (b) : Current recommended policy is wrong (not r-good).}
To get on with the proof for this part we will use similar argument as \citep{carlsson2024pure}. Though the argument was motivated by the work \citep{degenne2019nonasymptotic}. We define the event 

\ifdoublecol
\begin{align*}
	B_t \defn \bigg\{ \opt \defn \argmax_{\pol\in\trgoodest}\Hat{\bmu}_t^{\top}\pol  \notin \rgood\bigg\}
\end{align*}
\else
\begin{align*}
    B_t \defn \bigg\{ \opt \neq \argmax_{\pol\in\Fest_t}\Hat{\bmu}_t^{\top}\pol \bigg\}
\end{align*}
\fi
i.e the current recommendation policy is not correct which implies that the mean estimate or the constraint estimate has not been concentrated yet. If we define Chernoff's information function as $\mathrm{ch}(u,v) \defn \inf_{z\in\mathcal{D}}(d(u,x)+d(v,x))$. Therefore the current mean estimate will yield positive Chernoff's information since it has not been converged yet i.e $\exists \epsilon>0: \mathrm{ch}(\Hat{\bmu}_{a,t},\bmu_a)>\epsilon$. Consequently under the good event $G_T$ defined earlier
\begin{align*}
    \frac{g(t)}{N_{a,t}} \leq \epsilon
\end{align*}
since $\mathrm{ch}(\Hat{\bmu}_{a,t},\bmu_a)\leq d(\Hat{\bmu}_{a,t},\bmu_a)$. At time $s\in\mathbb{N}$, let $\pol'$ be an extreme point in $\Fest_s$ that is not the r-good optimal policy. But since it is an extreme point in $\Fest_s$ that shares $(\numarm -1)$ active constraints with $\opt$, it has to be an r-good optimal policy w.r.t some $\blambda\in\Lambda_{\Fest_s}(\Hat{\bmu}_s, \pol):\pol' =\argmax_{\pol\in\Pi_{\Fest_s}^r} \blambda^{\top}\pol$. So we again define the event $B_t$ as
\begin{align*}
    B_t \defn \bigg\{\blambda\in\Lambda_{\Fest_t}(\Hat{\bmu}_t):\pol' =\argmax_{\pol\in\trgoodest} \blambda^{\top}\pol \notin \rgood \bigg\}
\end{align*}
We again define $n_{\pol'}(t)$ be the number of steps when $\pol_s = \pol', s\in[t]$. Therefore
\begin{align}\label{eps_T}
    \epsilon = \sum_{s=1,B_s}\sumk \allocation_{a,s}d(\Hat{\bmu}_{a,s},\mu_a) \geq 	\sum_{\pol'\notin\rgood}\inf_{\blambda\in B_s}\sum_{s=1,\pol_s=\pol'}^t \sumk \allocation_{a,s}d(\Hat{\bmu}_{a,s},\mu_a) 
\end{align}

Now to break the RHS of the above inequality we go back to step 5 of the proof of part (a) where we showed LAGEX is an approximate saddle point algorithm. In this case the slack will be $x_t = R_{n_{\pol'}(t)}^{\allocation}+\sum_{s=1,\pol_s=\pol'}^t \sumk \allocation_{a,s}C_{a,s} - \lag_t\Tilde{A}_t\allocation_t$. Therefore we can write the RHS of Equation \eqref{eps_T} as
\begin{align}\label{eq.Decomp}
    &\sum_{\pol'\notin\rgood}\inf_{\blambda\in B_s}\sum_{s=1,\pol_s=\pol'}^t \sumk \allocation_{a,s}d(\Hat{\bmu}_{a,s},\mu_a) \notag \\
    \geq& \max_{\pol\in\Fest_s}  \underbrace{\sum_{s=1,\pol_s=\pol'}^t\sumk \allocation_{a,s} U_{a,s} - R_{n_{\pol'}(t)}^{\allocation}}_{T1} - \underbrace{\sum_{s=1,\pol_s=\pol'}^t\sumk\allocation_{a,s} C_{a,s}}_{T2} + \lag_t^{\top}\Tilde{A}_t\allocation_t\,.
\end{align}

We apply the same logic as in \citep{carlsson2024pure} and \citep{degenne2019nonasymptotic} that $\exists a'\in[\numarm]$ for which $U_{a',s}\geq \epsilon$. That means the term $T1$ grows at most linear with $n_{\pol'}(t)$. From the proof of part (a) it is clear that the term $T2 = \mathcal{O}\left(\sqrt{n_{\pol'}(t)\log n_{\pol'}(t)}\right) \leq \mathcal{O}(\sqrt{t\log t})$ and the allocation player regret is bounded by $R_{n_{\pol'}(t)} = \mathcal{O}(\sqrt{Q n_{\pol'}(t)}) \leq \mathcal{O}(\sqrt{Qt})$. That means the number of times the event $B_t$ occurs is upper bounded by $\mathcal{O}(\sqrt{t\log t})$. Now the extra term in Equation \eqref{eq.Decomp} appearing with term $T1$ and $T2$ induces same implication as part (a).

Then incorporating part (a) and (b) to get the upper bound on the expected stopping time asymptotically
\begin{align*}
    &\mathbb{E}[\tau_{\delta}]\leq T_{\Fest_t}(\bmu)\left(\beta(t,\delta) +\mathcal{O}\left(\sqrt{t\log t}\right) + \mathcal{O}\left(\sqrt{Qt}\right) \right) + 2\psi_t \, 
\end{align*}
as $\Fest_t \to\mathcal{F}$ and $\psi_t \to \SP$, we have the explicit upper bound

\begin{align}\label{eq:non-asymp_lagex}
	\mathbb{E}[\tau] \leq T_0 (\delta) + 2\SP + C\numarm
\end{align}
where $T_0 (\delta) \defn \max\left\{ t\in\mathbb{N}: t \leq T_{\cF}(\bmu)\left(\beta(t,\delta)+\mathcal{O}\left(\sqrt{t\log t}\right) + \mathcal{O}\left(\sqrt{Qt}\right) \right) \right\}$ and $C$ is a problem independent constant that appears due to the concentration guaranty to good event $G_t$ i.e concentration of mean vector and constraint matrix. Thus, dividing both sides by $\log(\frac{1}{\delta})$ and taking the limit $\delta \rightarrow 0$ we conclude the proof. Note that, we do not see the effect of not knowing the constraints in the asymptotic guaranty, though in non-asymptotic sense in Equation \eqref{eq:non-asymp_lagex}, we see an additive effect of the shadow price on the sample complexity.  
\end{proof}

\subsection{Applications to Existing Problems} 

 \textbf{End-of-time Knapsack.} We can model the BAI problem with end-of-time knapsack constraints as discussed in Section \ref{subsec:motivation}. In such a setting the shadow price comes out to be $ \SP \leq c$ i.e the maximum consumable resource. So if we were to implement LATS for this the asymptotic sample complexity upper bound will translate to $\alpha T_{\Fest}(\bmu)(1+c)$, the multiplicative part being the effect of the end of time knapsack constraint. In case of LAGEX the unknown knapsack constraint will leave a additive effect quantified by $2c$. Recently people have deviated from only devising a no-regret learner in BwK rather people are interested to also give good sub-optimal guaranties on constraint violation as well. We think algorithms like LAGEX will perform well if we translate this model to our setting since it has shown not only good sample complexity but also better constraint violation guaranies as well. \\

 \textbf{Fair BAI across subpopulations.} This problem is a direct consequence of our setting. The shadow price in this setting $\SP = \frac{\max_{i\in \numarm}\pol_i^{\star}}{\min_{i\in \numarm}\pol_i^{\star}}>1$ i.e the ratio between maximum and minimum non-zero weight in the recommended policy. Similar to the knapsack scenario, here also this ratio will appear as a extra cost of not knowing the fairness constraint in multiplicative way in case of LATS and gets added to the sample complexity upper bound of LAGEX.\\

 \textbf{Pure exploration with Fairness of exposure.} We can think of a problem where we want to select a pool of employees from different sub-sections of a whole population for a task. As we want to maximise the reward or utility of this selected group we also must also give fair exposure to all race or say gender. As discussed earlier in Section \ref{subsec:motivation}, a direct application of our algorithms LATS and LAGEX to use them in the problem of pure exploration with unknown constraints on fairness of exposure. The shadow price in such a setting would be $\SP = \frac{\max_{i,j\in[\numarm]}\left(\frac{1}{\mu_i}-\frac{1}{\mu_j}\right)}{\min_{i,j\in[\numarm]}\left(\frac{1}{\mu_i}-\frac{1}{\mu_j}\right)} = \frac{\max_{i,j\in[\numarm]}(\bmu_i-\bmu_j)}{\min_{i,j\in[\numarm]}(\bmu_i-\bmu_j)}\geq1$.\\

 \textbf{Thresholding bandits.} 
 The problem of Thresholding bandit is motivated from the safe dose finding problem in clinical trials, where one wants to identify the highest dose of a drug that is below a known safety level. From the translated optimisation problem in Section \ref{subsec:motivation} we easily find out the shadow price for this setting to be $\SP = \frac{\max_{i\in[\numarm]}(\pol-\btheta)^i}{\min_{i\in[\numarm]}(\pol-\btheta)^i}\geq 1$. This shadow price is similar to ours because the constraint structure is very similar. Our setting generalises thresholding bandit problem by giving the liberty of choosing different threshold levels for different support index of $\pol$. Similarly to other settings this shadow price will come as a price of handling different unknown thresholds for every arm as addition in case of LAGEX and as multiplication for LATS.\\

\textbf{Feasible arm selection.}
Feasible arm selection problem is motivated by the spirit of recommending an optimal arm which should satisfy a performance threshold. For example one might be interested to find a combination of food among a plethora of options which maximises the nutrient intake, rather the nutrient value of the food combination should exceed a threshold value. The structure of the optimisation problem for such a setting is discussed in detail in Section \ref{subsec:motivation}.  Then the shadow price comes out as $\SP = \frac{\tau - f_{\min}}{\tau - f_{\mathrm{max}}}\geq 1$ where $f \in \reals^{\mathrm{Supp}(\pol)}$ can be compared to a utility function. In our setting we will not have access to the true utility function rather we have to track it per step. This shadow price again get multiplied to the LATS sample complexity upper bound as a cost of not knowing the true utility of the arms, whereas we see a additive cost incurred in case of LAGEX.\newpage
\section{Constraint Violations During Exploration}
\subsection{Upper Bound on Constraint Violation}
In a linear programming problem we say constraint is violated if the chosen allocation fails to satisfy any of the true linear constraints. In other words when the event $A\allocation_t \geq 0$. We start with the optimization problem relaxed with slack if the constraints were known,
\begin{align*}
    &\max_{\pol\in\rgoodest} \bmu^{\top} \pol\\
    \text{such that, }& A\pol + \Gamma \leq 0
\end{align*}
where, $\Gamma$ is the slack. 
Cumulative violation of constraints can be expressed as $\mathcal{V}_t = \sum_{s=1}^t \max_{i\in[K]} \left[A^i\allocation_t \right]_{+}$ where $[z]_{+} = \max\{z,0\}$. Then, at any time step $t\in[T]$, instantaneous violation is given by, $v_t = \max_{i\in[k]} \left[A^i\allocation_t \right]_{+}$. Since, $A$ is feasible, we define the game value as,
$\eta = \max_{\allocation\in\mathcal{F}} \max_{i\in[K]} A^i\allocation \leq -\Gamma$ and $ \eta = \max_{\allocation\in\mathcal{F}}\max_{i\in[K]} A^i\allocation  \geq \min_{\Tilde{A}\in\mathcal{C}_A^t}\max_{\allocation\in \Fest_t} \max_{i\in[K]}\Tilde{A}_t^i \allocation_t = \Tilde{A}_t^{i_t}\allocation_t$, holds due to pessimistic estimate of $A$.

Again, we define, $i_{\text{min}}(\allocation) = \arg \min_{i\in[K]}\Tilde{A}^i \allocation$.
\begin{align*}
    \max_{i\in[K]} \left[A^i\allocation_t \right] &= \max_{i\in[K]} (A^i - \Tilde{A}_t^i)\allocation_t +\max_{i\in[K]} \Tilde{A}_t^i \allocation_t \\
    &\leq \max_{i\in[K]} ||(A^i - \Tilde{A}_t^i)||_{\Sigma_t}||\allocation_t||_{\Sigma_t^{-1}} +\max_{i\in[K]} \Tilde{A}_t^i \allocation_t\\
    &\leq f(t,\delta)||\allocation_t||_{\Sigma_t^{-1}} + \max_{i\in[K]} \Tilde{A}_t^i \allocation_t 
    \leq \rho(t,\delta) - \Gamma
\end{align*}
where the last inequality holds because $\max_{i\in[K]} \Tilde{A}_t^i \allocation_t \leq \max_{i\in[K]} A^i \allocation_t \leq -\Gamma$

%Then, the instantaneous violation becomes,
%\begin{align*}
    %v_t &\leq [f(t,\delta)||\allocation_t||_{\Sigma_t^{-1}} + \Gamma]_{+}\\
    %&\leq \underbrace{f(t,\delta)||\allocation_t||_{\Sigma_t^{-1}}}_{\rho(t,\delta)}+|\Gamma|
%\end{align*}

Let stopping time is denoted by $\tau_{\delta} < T$ following the expression from the Stopping criterion section. Then expected cumulative constraint violation is denoted by,
%\begin{align*}
    %\mathcal{V}_{\tau_{\delta}} &= \sum_{t\leq\tau_{\delta}}s_t\\
    %&=\sum_{t\leq\tau_{\delta}}\rho(t,\delta)+|\Gamma|\\
    %&\leq 2\sum_{t\leq\tau_{\delta}} |\Gamma| \indicator\{ \rho(t,\delta) \leq |\Gamma| \} + 2\sum_{t\leq\tau_{\delta}}\rho(t,\delta)\indicator\{ |\Gamma|\leq \rho(t,\delta)\}\\
    %&\leq 2\tau_{\delta}|\Gamma| + 2\sum_{t\leq\tau_{\delta}} \frac{(\rho(t,\delta)^2}{|\Gamma|}, \text{  where,  } \indicator(u\leq v)\leq \frac{v}{u}\\
    %&\leq 2\tau_{\delta}|\Gamma| + \frac{6\numconst^2\log^2(1+\frac{\tau_{\delta}+1}{d})+12\numconst\log(1+\frac{\tau_{\delta}+1}{d})(1+\log\frac{\numarm}{\delta})}{|\Gamma|}
%\end{align*}
\begin{align*}
	\mathbf{E}\big[\frac{\mathcal{V}_{\tau_{\delta}}}{\tau_{\delta}}\big] = \sum_{t\leq\mathbf{E}[\tau_{\delta}]}s_t \leq\sum_{t\leq\mathbf{E}[\tau_{\delta}]}[\rho(t,\delta)-\Gamma]_{+}
	=&\sum_{t\leq\mathbf{E}[\tau_{\delta}]}\{\rho(t,\delta)-\Gamma\}\, \hspace{0.2cm} \text{because }\rho(t,\delta)-\Gamma \hspace{0.1cm}\text{ is positive.}\\	
	\leq& \sqrt{2\numconst f^2(\mathbf{E}[\tau_{\delta}],\delta)\log \left( 1+\frac{1+\mathbf{E}[\tau_{\delta}]}{\numconst}\right)} -\Gamma 
\end{align*}
The last inequality is a direct consequence of Lemma \ref{lemm:Concentration of the infinite norm}. Thus the expected constraint violation is of order $\Tilde{\mathcal{O}}\left( t\sqrt{d} \right)$, ignoring the lower order dependence on $t$.

\subsection{Experimental Results on Constraint Violation}\label{sec:additional exp}
 We compare the constraint violation i.e the number of times $A\allocation_t >0$, where $A$ is the true constraint matrix. 

\begin{figure}[h]
 \centering
\begin{minipage}{0.48\textwidth}
  \centering
    \includegraphics[width=0.7\textwidth]{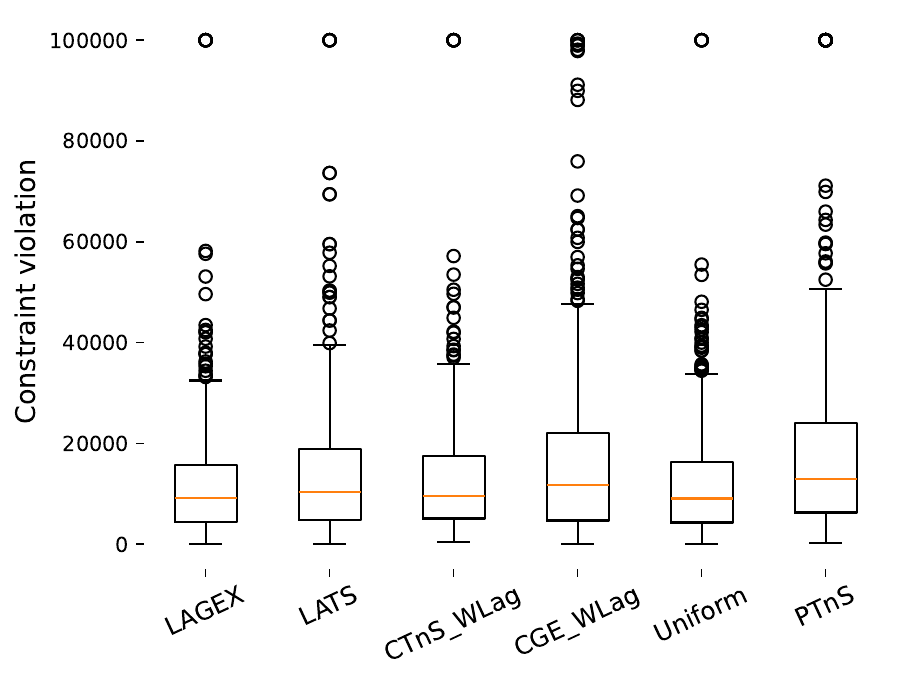}
    \caption{Constraint violation (median$\pm$std.) algorithms for \textbf{hard environment in setup 1}.}\label{fig:constraint_violation_hard}
\end{minipage}\hfill
\begin{minipage}{0.48\textwidth}
    \centering
    \includegraphics[width=0.7\textwidth]{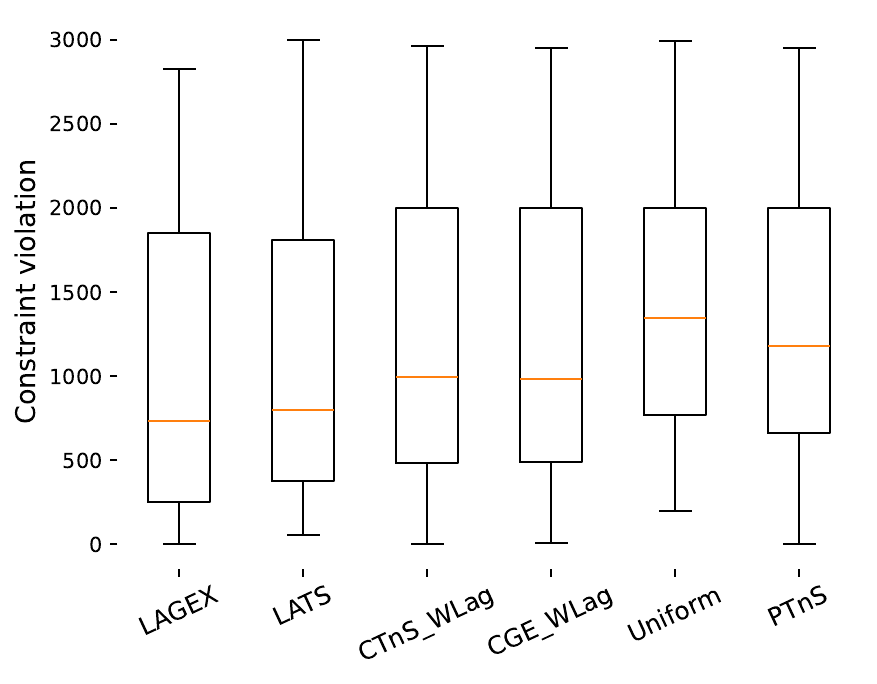}
    \caption{Constraint violation (median$\pm$std.) of algorithms for \textbf{easy environment in setup 1}.}\label{fig:Constraint_violation_easy}
\end{minipage}%\vspace*{-1.5em}
\end{figure}

\begin{figure}[t!]
	\centering
	\begin{minipage}{0.48\textwidth}
		\centering
		\includegraphics[width=0.7\textwidth]{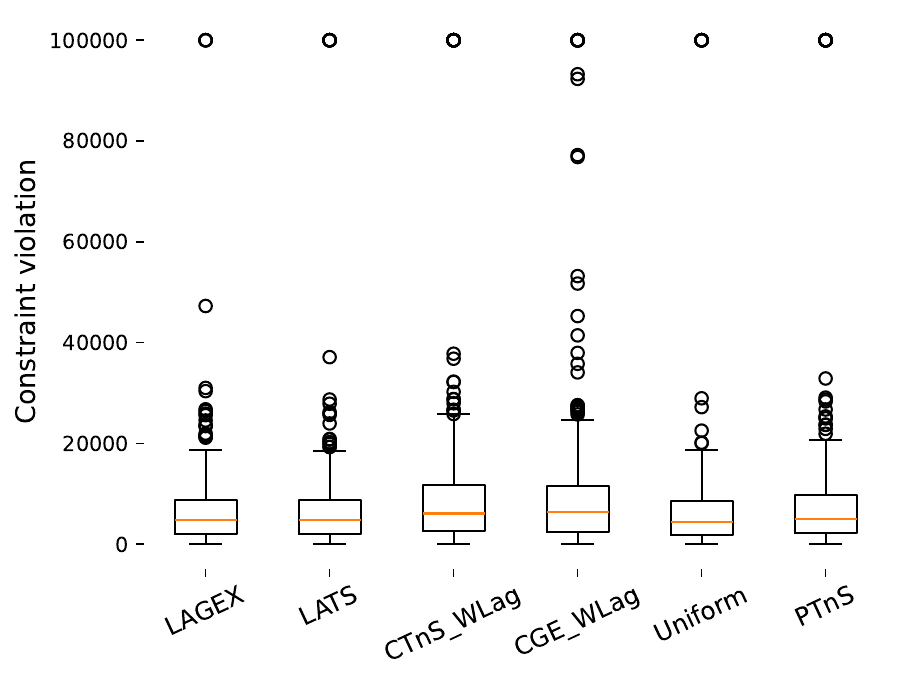}
		\caption{Constraint violation (median$\pm$std.) algorithms for \textbf{hard environment in setup 2}.}\label{fig:constraint_violation_emil_hard}
	\end{minipage}\hfill
	\begin{minipage}{0.48\textwidth}
		\centering
		\includegraphics[width=0.7\textwidth]{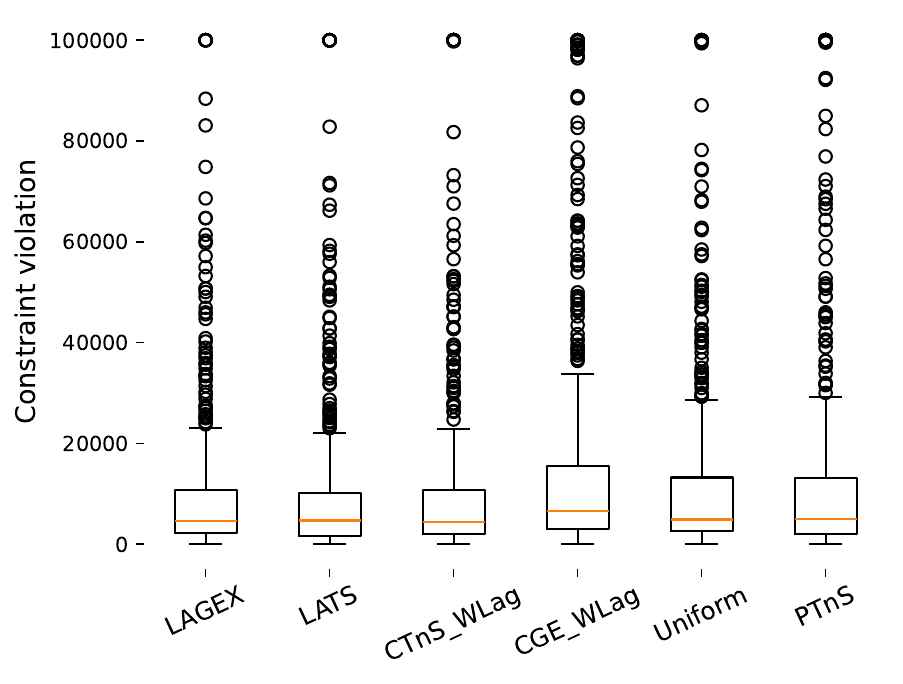}
		\caption{Constraint violation (median$\pm$std.) of algorithms for \textbf{easy environment in setup 2}.}\label{fig:Constraint_violation_emil_easy}
	\end{minipage}
	\begin{minipage}{0.48\textwidth}
	\centering
	\includegraphics[width=0.7\textwidth]{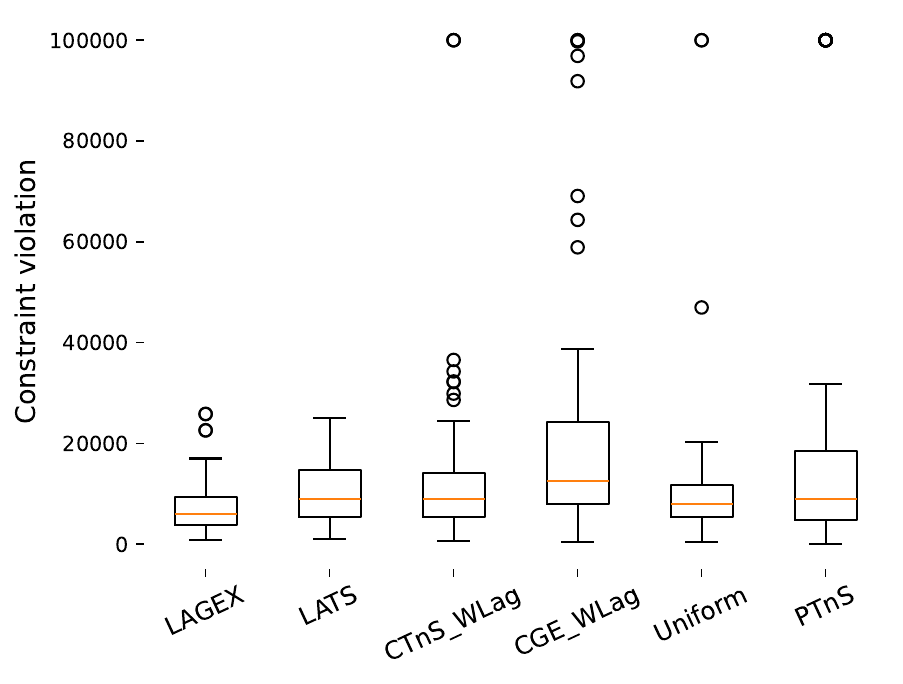}
	\caption{Constraint violation (median$\pm$std.) of algorithms for \textbf{IMDB environment}.}\label{fig:Constraint_violation_imdb}
\end{minipage}
\end{figure}

\textbf{Observation: LAGEX is the ``safest" among the existing baselines.} 

\textbf{Setup 1.} From Figure \ref{fig:constraint_violation_hard} unlike LAGEX,  we see LATS shows comparable amount of constraint violation with Uniform sampling, whereas in Figure \ref{fig:Constraint_violation_easy} we prominently observe LATS and LAGEX outperforming other algorithms in safety.

\textbf{Setup 2.} In the \textbf{hard env} as per \cite{carlsson2024pure} in Figure \ref{fig:constraint_violation_emil_hard} we observe both LAGEX and LATS show constraints violation as low as Uniform sampling outperforming other algorithms, but in the \textbf{easy env} in Figure \ref{fig:Constraint_violation_emil_easy} we see all the algorithms except CGE without Lagrangian relaxation shows more or less same amount of safety violation while choosing allocation.

\textbf{IMDB env.} During the experiment with IMDB-50K dataset wee see that LAGEX proves to be the safest among all other algorithms with a distinguishably low safety violations. Though again LATS shows comparable violation with Uniform sampling. 

Therefore, we see that across multiple environments LAGEX is the "safest" out of all the competing algorithms.  

\section{Experimental Details}\label{sec:Exp details}%\vspace*{-1em}
\textbf{Computation resources.} We run the algorithms on a 64-bit 13th Gen Intel® Core™ i7-1370P × 20 processor machine with 32GB ram with a disc capability of 1TB and graphics Mesa Intel® Graphics (RPL-P). The hardware model is HP EliteBook 840 14 inch G10 Notebook PC. 
%Code is available in this \href{https://anonymous.4open.science/r/Pure-exploration-with-unknown-linear-constraints-EB68/}{Link}.

Throughout this paper we have set $r=0.01$ and $\beta(t,\delta) = \log\frac{1+\log\log t}{\delta}$ for all the experiments.  For CGE we have used $g(t) = \log t$. Each plot has been generated over 500 random seeds. %Experiments in \textbf{easy environment} in Setup 1 are clipped at $3000$ for better visualisation.

\textbf{Baseline algorithms.}We compare LAGEX and LATS with the two algorithms under known constraints, i.e. CTnS and CGE~\citep{carlsson2024pure}. Also, to understand the utility of Lagrangian relaxation, we implement versions of CTnS and CGE with estimated constraints. In these variants, we solve the constrained optimisation problems without Lagrangian relaxation but with estimated constraints. We also compare with PTnS (Projected Track and Stop), a variant of TnS, where the algorithm computes the allocation by solving the unconstrained BAI optimisation problem and projects this allocation in the feasible set as necessary.

\newpage
\newpage
\section{Technical Results: New and Known}
In this section, we will devise some technical lemma using the help of standard text on online linear regression to ensure the convergence of unknown constraints. We specifically give the expression of the radius of confidence ellipsoid mentioned in the main text in Equation \ref{def:conf_ellipse}. We then prove an upper bound on the \textit{bad event} i.e when the constraint matrix is not concentrated around the true matrix.  We also acknowledge some known theoretical results from BAI and pure exploration literature that are used in this work.

\subsection{Concentration Lemma for Constraints}
Here, we want to get concentration on the deviation of the pessimistic estimate of the constraint matrix from the actual one  quantified by $\| (\tilde{A}-A)\allocation \|_{\infty}$, it becomes very crucial to prove upper bounds on sample complexity of our proposed algorithms. The following lemma ensures the concentration of the constraint matrix. 

\begin{lemma}\label{lemm:Concentration of the infinite norm}
    For the pessimistic estimate $\tilde{A}$ of $A$, the following holds
    \begin{enumerate}
        \item $|(\tilde{A}^i - A^i)\allocation|\leq\rho(t,\delta) $ where $\rho(t,\delta) \defn f(t,\delta) \|\allocation\|_{\Sigma_t^{-1}}$
        \item $\sum_{s=1}^t \|\allocation_s\|_{\Sigma_s^{-1}}^2 \leq 2\numconst\log\left( 1+\frac{1+t}{\numconst}\right)$
        \item $\sum_{s=1}^t \rho(s,\delta) \leq t\sqrt{2\numconst f^2(t,\delta)\log \left( 1+\frac{1+t}{\numconst}\right)}$
    \end{enumerate}
\end{lemma}

\begin{proof}
The first result gives control on the deviations $(\tilde{A} - A)\allocation$ for $A\in\mathcal{C}_t^A (\delta)$. Then $\forall i\in[\numconst]$
\begin{align*}
    |(\tilde{A}^i - A^i)\allocation| \leq |(\tilde{A}_t^i - A^i)\allocation| \leq 2\sup_{A\in\mathcal{C}_t^A (\delta)} \|\tilde{A}_t^i - A^i \|_{\Sigma_t}\|\allocation\|_{\Sigma_t^{-1}} \leq f(t,\delta) \|\allocation\|_{\Sigma_t^{-1}} \leq \rho(t,\delta) 
\end{align*}
here, we define $\rho(t,\delta) \defn f(t,\delta) \|\allocation\|_{\Sigma_t^{-1}}$. The penultimate inequality follows from the definition of the confidence set defined in Equation \ref{def:conf_ellipse}. Now we want to derive an explicit expression of this upper bound. It is natural to use the concentration of the gram matrix $\Sigma_t$ over time. We refer to \cite{NIPS2011_e1d5be1c} for the control over the behaviour of $\Sigma_t$ and we directly get the second as
    \begin{align*}
        \sum_{s=1}^t \|\allocation_s\|_{\Sigma_t^{-1}}^2 \leq 2\log\det\Sigma_{t+1}\leq 2\numconst\log\left( 1+\frac{1+t}{\numconst}\right)
    \end{align*}
Refer \cite{NIPS2011_e1d5be1c} for the context. Now we have to control the cumulative deviation because later on when we define the bad event based on this concentration we will need to know the cumulative behavior of $\rho(t,\delta)$. 

Then for an arbitrary sequence of actions $\{ \allocation_s\}_{s\in[T]}$ 
    \begin{align*}
        \sum_{s=1}^t \rho(s,\delta) \leq t\sqrt{\sum_{s=1}^t \rho^2 (s,\delta)}\leq \sqrt{2\numconst tf^2(t,\delta) \log\left( 1+\frac{1+t}{\numconst}\right)}
    \end{align*}
    where, $\sum_{s=1}^t \rho^2(s,\delta) \leq 2\numconst tf^2(t,\delta) \left( 1+\frac{1+t}{\numconst}\right)$ using result 2 of this lemma. This holds because as we have already stated $\{f(s,\delta)\}_{s\in[T]}$ is a non-decreasing sequence of function and $f(t,\delta)$ is the maximum possible value in the set i.e $\sum_{s=1}^t f^2(s,\delta) \leq tf^2(t,\delta)$
    
\end{proof}

Now we proceed to state an upper bound on the cumulative probability of the \textit{bad event} i.e the event $|(\tilde{A}_t - A)\allocation|_{\infty}>|(\tilde{A}_t^i - A^i)\allocation| > \rho(t,\delta)$.

\begin{lemma}\label{lemm:bad event prob}
    The cumulative probability of the bad event till time t<T, 
    \begin{align*}
        \sum_{s=1}^t \mathbb{P}\left( \|(\tilde{A}_t - A)\allocation\|_{\infty}>  \rho(t,\delta)\right) \leq \zeta(1)
    \end{align*}
where $\zeta(.)$ is the Euler-Riemann zeta function.
\end{lemma}

\begin{proof}
    We already have stated in the main paper $$f(t,\delta) = 1+\sqrt{\frac{1}{2}\log\frac{K}{\delta}+\frac{1}{4}\log\det\Sigma_t} \leq 1+\sqrt{\frac{1}{2}\log\frac{K}{\delta}+\frac{\numconst}{4}\log\left(1+\frac{t}{\numconst}\right)} \defn f'(t,\delta) $$ by Lemma \ref{lemm:Concentration of the infinite norm} and we define $\rho'(t,\delta) \defn f'(t,\delta)\|\allocation\|_{\Sigma_t^{-1}}$. It implies $\mathbb{P}\left( \exists t\in[T] \|(\tilde{A}_t - A)\allocation\|_{\infty}>  \rho'(t,\delta) \right) \leq \mathbb{P}\left( \exists t\in[T] \|(\tilde{A}_t - A)\allocation\|_{\infty}>  \rho(t,\delta) \right) \leq \delta$. Now if we replace $\log\frac{1}{\delta}$ by u, we can write $$\mathbb{P}\left( \exists t\in[T],\forall i\in[\numconst] |  \|\tilde{A}_t^i - A^i \|_{\Sigma_t} > 1+\sqrt{\frac{1}{2}\log\numarm+ \frac{u}{2}+\frac{\numconst}{4}\log\left(1+\frac{t}{\numconst}\right)} \right) \leq \exp(-u).$$ We can directly assign $\log t$ as the simplest and natural choice for $u$, since $\sum_{s=1}^{\infty}\frac{1}{t} = \zeta(1)$, $\zeta(.)$ being the Euler-Riemann zeta function. Though this integral is improper, it has a Cauchy principal value as  Euler–Mascheroni constant which means $\sum_{s=1}^{\infty}\frac{1}{t} \approx \gamma = 0.577 $. So we assign $u = \log t$ 
    \begin{align*}
        \sum_{s=1}^t \mathbb{P}\left( \|(\tilde{A}_t - A)\allocation\|_{\infty}>  \rho(t,\delta)\right) \leq \sum_{s=1}^t \frac{1}{t} \leq \sum_{s=1}^{\infty} \frac{1}{t} \leq \zeta(1) \approx 0.577
    \end{align*}
\end{proof}

\begin{lemma}\label{lemm:Dantzig lemma}
    Let $\bar{\bmu} \geq 0$ be a vector, and consider the set $Q_{\bar{\bmu}}=\{\bmu \geq 0 \mid q(\bmu) \geq q(\bar{\bmu})\}$. Let Slater condition hold. Then, the set $Q_{\bar{\bmu}}$ is bounded and, in particular, we have
    $\|\bmu\|_1 \leq \frac{1}{\gamma}(f(\bar{x})-q(\bar{\bmu})) , \forall\bmu \in Q_{\bar{\bmu}}$
    where, $\gamma=\min _{1 \leq j \leq m}\left\{-g_j(\bar{x})\right\}$ and $\bar{x}$ is a Slater vector. f(.) and q(.) respectively denotes the primal and the dual function of the optimization problem.
\end{lemma}
Using the aforementioned lemmas we give a bound for the part in the inverse of the characteristic time function that gets added for Lagrangian relaxation which eventually helps up landing on a unique formulation of sample complexity upper bounds of our proposed algorithm.
\begin{lemma}\label{lemm:constraint part upper bound} 
For any $\lag\in\cL$ and $\allocation\in\simp$, $-\lag^{\top} \tilde{A} \allocation \leq \mathcal{D}(\bmu,\allocation,\Fest)\psi$

where, $\psi = \frac{||(\tilde{A} - A)\allocation||_{\infty} + \max_{i\in[1,N]} (-A\allocation)}{\min_{i \in [1,N]}(-\tilde{A}\allocation)}$
\end{lemma}
\begin{proof}
For any $\lag\in\cL$ and $\allocation\in\simp$ we write
\begin{align*}
   (-\lag^{\top} \Tilde{A}\allocation) = \lag^{\top}(-\Tilde{A} +A -A)\allocation
   &\leq \|\lag\|_1 \|(A - \Tilde{A} )\allocation\|_{\infty} + \|\lag\|_1 \max_{i\in[1,N]}(-A^i \allocation)\\
   &\leq \|\lag\|_1 \left( (\|A-\tilde{A})\allocation \|_{\infty} + \max_{i\in[1,N] }(-A\allocation)\right)\\
   &\leq \mathcal{D}(\allocation,\bmu,\cF) \frac{\|(\tilde{A} - A)\allocation\|_{\infty} + \max_{i\in[1,N]} (-A\allocation)}{\min_{i \in [1,N]}(-\tilde{A}\allocation)}
\end{align*}
Plugging in the definition of $\psi$ mentioned in the statement of the lemma concludes the proof.
\end{proof}

\subsection{Useful Results from BAI and Pure Exploration Literature}

\begin{lemma}[Lemma 19 in \citep{pmlr-v49-garivier16a}]\label{lemm:cumulative bad event for mu}
There exists two constants B and C (depends on $\bmu$ and $\epsilon$), such that---
\begin{align*}
    \sum_{t=h(T)}^{\top} \mathbb{P}\left\{ \|\Hat{\bmu}_t - \bmu\|_{\infty} >\xi(\epsilon)\right\} \leq BT\exp(-CT^{\frac{1}{8}})
\end{align*}
\end{lemma}

\begin{lemma}[Lemma 7, \cite{pmlr-v49-garivier16a}]\label{lemm:Kaufmann tracking}
   For all $t\geq 1$ and $\forall a\in[\numarm]$, C-Tracking ensures $N_{a,t}\geq \sqrt{t+\numarm^2}-\numarm$ and 
   \begin{align*}
       \max_{a\in[\numarm]}| N_{a,t} - \sum_{s=1}^t\allocation_{a,s}|\leq \numarm(1+\sqrt{t})
   \end{align*}
\end{lemma}

\begin{lemma}[Theorem 14 in \citep{JMLR:v22:18-798}]\label{lemm:Kauffman martinagle stopping}
    Let $\delta>0, \nu$ be independent one-parameter exponential families with mean $\bmu$ and $S \subset[d]$. Then we have,
$$
\mathbb{P}_\nu\left[\exists t \in \mathbb{N}: \sum_{a \in S} \tilde{N}_{t, a} d_{K L}\left(\mu_{t, a}, \mu_a\right) \geq \sum_{a \in S} 3 \ln \left(1+\ln \left(\tilde{N}_{t, a}\right)\right)+|S| \mathcal{T}\left(\frac{\ln \left(\frac{1}{\delta}\right)}{|S|}\right)\right] \leq \delta\,.
$$
Here, $\mathcal{T}: \mathbb{R}^{+} \rightarrow \mathbb{R}^{+}$is such that $\mathcal{T}(x)=2 \tilde{h}_{3 / 2}\left(\frac{h^{-1}(1+x)+\ln \left(\frac{\pol^2}{3}\right)}{2}\right)$ with
\begin{align}
\forall u \geq 1, \quad h(u) & =u-\ln (u) \\
\forall z \in[1, e], \forall x \geq 0, \quad \tilde{h}_z(x) & = \begin{cases}\exp \left(\frac{1}{h^{-1}(x)}\right) h^{-1}(x) & \text { if } x \geq h^{-1}\left(\frac{1}{\ln (z)}\right) \\
z(x-\ln (\ln (z))) & \text { else }\end{cases} \,.
\end{align}
\end{lemma}

\begin{lemma}[Lemma 17 in \citep{Degenne2019PureEW}]\label{lemm:C-tracking Remy}
    Under the good event $G_T$, there exists a $T_\epsilon$ such that for $T$ where $h(T) \geqslant T_\epsilon $ C-tracking will satisfy
$$
\inf _{\boldsymbol{w} \in w^*(\boldsymbol{\bmu})}\left\|\frac{N_t}{t}-\boldsymbol{w}\right\|_{\infty} \leqslant 3 \epsilon, \forall t \geqslant 4 \frac{K^2}{\epsilon^2}+3 \frac{h(T)}{\epsilon}
$$
\end{lemma}

\begin{lemma}[Theorem 2 in \citep{degenne2019nonasymptotic}]\label{lemm:complexity of GE}
   The sample complexity of GE is
$
\mathbb{E}[\tau] \leqslant T_0(\delta)+\frac{e K}{a}
$
where $T_0(\delta)=\max \left\{t \in \mathbb{N}: t \leqslant T(\boldsymbol{\bmu}) c(t, \delta)+C_{\boldsymbol{\bmu}}\left(R_t^{\boldsymbol{\blambda}}+R_t^{\boldsymbol{w}}+O(\sqrt{t \log t})\right)\right\}
$
where $R_t^{\boldsymbol{\blambda}}$ is the regret of the instance player, $R_t^{\boldsymbol{w}}$ the regret of the allocation player and $C_{\boldsymbol{\bmu}}$ an instance-dependent constant. 
\end{lemma}

\subsection{Useful Results for Continuity of Convex Functions}

\begin{definition}[Definition of Upper Hemicontinuity]\label{def:Upper Hemicontinuity}
    We say that a set-valued function $C: \Theta \rightarrow \allocation$ is upper hemicontinuous at the point $\theta \in \Theta$ if for any open set $S \subset \allocation$ with $C(\theta) \in S$ there exists a neighborhood $U$ around $\theta$, such that $\forall x \in U, C(x)$ is a subset of $S$.
\end{definition}

\begin{theorem}[Berge's maximum theorem, \citep{berge1963topological}]\label{thm:Berge's theorem}
    Let $X$ and $\Theta$ be topological spaces. Let $f: X \times \Theta \rightarrow \mathbb{R}$ be a continuous function and let $C: \Theta \rightarrow \bar{X}$ be a compact-valued correspondence such that $C(\theta) \neq \varnothing \forall \theta \in \Theta$. If $C$ is continuous at $\theta$ then $f^*(\theta)=\sup _{x \in C(\theta)} f(x, \theta)$ is continuous and $C^*=\left\{x \in C(\theta): f(x, \theta)=f^*(\theta)\right\}$ is upper hemicontinuous.
\end{theorem}

\begin{theorem}[Heine-Borel theorem, due to Eduard Heine and Émile Borel]\label{Heine-Borel theorem}
    For a subset $S$ in $\reals^n$, the following two statements are equivalent
    \begin{enumerate}
        \item $S$ is closed and bounded.
        \item $S$ is compact, means every open cover of $S$ has a finite sub-cover.
    \end{enumerate}
\end{theorem}

\begin{theorem}\label{thm:Dantzig theorem on continuity}
    Let $C$ be a closed convex set with nonempty (topological) interior. Let $f$ and $\left\{f^r\right\}$ be affine functions from $E^n$ to $E^m$ with $f^r \rightarrow f$. Then

(II.1.2) $\uplim_{r \rightarrow \infty}\left(H\left(f^r\right)\cap C\right) \subset H(f) \cap C$

(II.1.3). $\lowlim_{r \rightarrow \infty}\left(H\left(f^r\right) \cap C\right)$ is a closed convex subset of $H(f) \cap C$,

(II.1.4). If $H(f) \cap C$ has nonempty interior and no component of $f$ is identically zero, then $\lim _{r \rightarrow \infty}\left(H\left(f^r\right) \cap C\right)=H(f) \cap C .$
\end{theorem}

\begin{lemma}[Lemma 13, \cite{magureanu2014lipschitz}]\label{lemm:hemi-cont}
   Consider $A \in\left(\mathbb{R}^{+}\right)^{k \times k}, c \in\left(\mathbb{R}^{+}\right)^k$, and $\mathcal{T} \subset\left(\mathbb{R}^{+}\right)^{k \times k} \times\left(\mathbb{R}^{+}\right)^k$. Define $t=(A, c)$. Consider the function $Q$ and the set-valued map $Q^{\star}$
\begin{align*}
    Q(t) & =\inf _{x \in \mathbb{R}^k}\{c x \mid A x \geq 1, x \geq 0\} \\
    Q^{\star}(t) & =\{x: c x \leq Q(t) \mid A x \geq 1, x \geq 0\} .
\end{align*}

Assume that: For all $t \in \mathcal{T}$, all rows and columns of $A$ are non-identically 0 and $\min _{t \in \mathcal{T}} \min _k c_k>0$. Then, 1. $Q$ is continuous on $\mathcal{T}$, 2. $Q^{\star}$ is upper-hemicontinuous on $\mathcal{T}$.
\end{lemma}
\clearpage
%\input{Rebuttal}
%%%%%%%%%%%%%%%%%%%%%%%%%%%%%%%%%%%%%%%%%%%%%%%%%%%%%%%%%%%%

\end{document}